%% file: main.tex
\pdfoutput=1
\documentclass[twoside]{article}
\usepackage{aistats2020}
\usepackage{amssymb}
\usepackage{amsmath}
\usepackage{amsthm}
\usepackage{bm}
\usepackage{algorithm}
\usepackage[noend]{algorithmic}
\usepackage{graphicx}
\usepackage{csvsimple}
\usepackage{booktabs}
\usepackage{subcaption}
\usepackage{array}
\usepackage{hyperref}

\graphicspath{ {./pictures/} }

\newtheorem{proposition}{Proposition}
\newtheorem{lemma}{Lemma}

\newtheorem{theorem}{Theorem}
\newtheorem{definition}{Definition}
\newtheorem{example}{Example}

    \newtheoremstyle{TheoremNum}
        {\topsep}{\topsep}              
        {\itshape}                      
        {}                              
        {\bfseries}                     
        {.}                             
        { }                             
        {\thmname{#1}\thmnote{ \bfseries #3}}
    \theoremstyle{TheoremNum}
    \newtheorem{reptheorem}{Theorem}

    \theoremstyle{TheoremNum}

\newcommand{\Var}{\mathrm{Var}}

\DeclareMathOperator*{\expectation}{\mathbb{E}}
\DeclareMathOperator{\sgn}{sgn}
\DeclareMathOperator*{\argmin}{arg\,min}

\newif\ifdraft
\draftfalse

\newcommand*{\underlyingSpace}{\mathbb{X}}
\newcommand*{\underlyingSpacePoint}{x}
\newcommand*{\underlyingSpaceDistribution}{\mathcal{D}_{\underlyingSpace}}
\newcommand*{\underlyingSpaceRandomVariable}{X}
\newcommand*{\underlyingSpaceSample}{T}
\newcommand*{\kTupleSpace}{\underlyingSpace^K}
\newcommand*{\kTupleToUnderlyingDistance}{\Delta}
\newcommand*{\kTuple}{\bar{\underlyingSpacePoint}}
\newcommand*{\kTupleRandomVariable}{\bar{X}}
\newcommand*{\kTupleSample}{S}
\newcommand*{\Ecc}{\mathrm{Ecc}}
\newcommand*{\EccHat}{\mathrm{\widehat{E}cc}}
\newcommand*{\medoid}{\nu}
\newcommand*{\trueEccSampleMedoid}{\Hat{\medoid}}
\newcommand*{\sampleMedoid}{\Hat{\trueEccSampleMedoid}}
\newcommand*{\errOne}{\text{err}_1}
\newcommand*{\errTwo}{\text{err}_2}
\newcommand*{\errThree}{\text{err}_3}

\newcommand*{\errMME}{\text{err}_{\text{MME}}}
\newcommand*{\errME}{\text{err}_{\text{ME}}}
\newcommand*{\candidateMedoid}{\underlyingSpacePoint^{\text{cur}}}
\newcommand*{\candidateMedoidSwap}{\underlyingSpacePoint_{il}}
\newcommand*{\candidateMedoidUpdated}{\underlyingSpacePoint^{\text{new}}}

\newcommand{\standardizedCDFOfDifference}{F_{\tilde{d}}}
\newcommand{\zscoreOfZero}{z_{0}}
\newcommand{\familyParameters}{\alpha, \beta, \gamma, k}
\newcommand{\familyParametersAugumented}{\familyParameters, n}
\newcommand{\zscoreOfZeroFirstTerm}{\frac{\mu_{1:2}^{1 - \frac{\beta}{2}} n^{\frac{1}{2}}}{\alpha^{\frac{1}{2}}}}
\newcommand{\zscoreOfZeroSecondTerm}{\frac{\delta}{\left((1 + \delta)^{\beta} + 1 + \frac{2k}{\alpha \mu_{1:2}^{\beta}} \right)^{\frac{1}{2}}}}
\newcommand{\zscoreOfZeroFirstTermInf}{\frac{\gamma^{1 - \frac{\beta}{2}} n^{\frac{1}{2}}}{\alpha^{\frac{1}{2}}}}
\newcommand{\zscoreOfZeroSecondTermInf}{\frac{\delta}{\left((1 + \delta)^{\beta} + C_1 \right)^{\frac{1}{2}}}}
\newcommand*{\relError}{\epsilon}
\newcommand{\denominatorFunction}{\mathrm{dr}}
\newcommand{\denominatorFunctionExpression}{\left((1 + \delta)^\beta + C_1\right)^{\frac{1}{2}}}
\newcommand*{\partialDerivative}[2]{\frac{\partial #1}{\partial #2}}
\newcommand{\zscoreOfZeroUBTwo}{\zscoreOfZero^{\mathrm{UB}}}
\newcommand*{\relErrorUBTwo}{\relError^{\text{UBInt}}}
\newcommand*{\relErrorUBThree}{\relError^{\mathrm{UB}}}
\newcommand*{\zscoreOfZeroUBTwoInverse}{\zscoreOfZero^{\mathrm{UB2}-1}}

\newcommand*{\range}{\mathcal{R}}
\newcommand*{\standardNormalCDF}{\Phi}
\newcommand*{\normalDistribution}{\mathcal{N}}
\newcommand*{\relErrorUBThreeNormal}{\relError^{\mathrm{UB3}}_{\normalDistribution}}
\newcommand*{\relErrorUBThreeNormalInverse}{\relError^{\mathrm{UB3}-1}_{\normalDistribution}}
\newcommand*{\relErrorNormal}{\relError_{\normalDistribution}}
\newcommand*{\GGen}{G_{\mathrm{gen}}}
\newcommand*{\GNormal}{G_{\normalDistribution}}
\newcommand*{\GBE}{G_{\mathrm{BE}}}
\newcommand*{\relErrorUBThreeBE}{\relError^{\mathrm{UB3}}_{\mathrm{BE}}}
\newcommand*{\relErrorUBThreeGen}{\relError^{\mathrm{UB3}}_{\mathrm{gen}}}
\newcommand*{\relErrorUBThreeBEInverse}{\relError^{\mathrm{UB3-1}}_{\mathrm{BE}}}
\newcommand*{\relErrorUBThreeGenInverse}{\relError^{\mathrm{UB3-1}}_{\mathrm{gen}}}
\newcommand*{\relErrorGen}{\relError_{\mathrm{gen}}}
\newcommand*{\deltaTh}{\delta^{\mathrm{th}}}
\newcommand*{\deltaThNormal}{\delta^{\mathrm{th}}_{\normalDistribution}}
\newcommand*{\deltaThGen}{\delta^{\mathrm{th}}_{\mathrm{gen}}}

\newcommand*{\kurtosis}{\kappa}
\newcommand*{\kurtosisUB}{\kappa^{\mathrm{UB}}}
\newcommand*{\kurtosisUBTwo}{\kappa^{\mathrm{UB2}}}
\newcommand*{\relErrorUBFour}{\relError^{\mathrm{UB4}}}
\newcommand*{\relErrorUBFourNormal}{\relError^{\mathrm{UB4}}_{\mathcal{N}}}
\newcommand*{\relErrorUBFourGen}{\relError^{\mathrm{UB4}}_{\mathrm{gen}}}
\newcommand*{\relErrorUBFiveGen}{\relError^{\mathrm{UB5}}_{\mathrm{gen}}}
\newcommand*{\muTh}{\mu^{\mathrm{th}}}
\newcommand*{\musOrderedM}{\mu_{1:m}, \ldots, \mu_{m:m}}
\newcommand*{\probabilityOfChoosingim}{\mathbb{P}_{i:m}}
\newcommand*{\TStarExpansion}{\frac{C_6^{\frac{1}{3}}}{2^{\frac{1}{3}} C_2 m^{\frac{2}{3}} n^{\frac{2}{3}}}}
\newcommand*{\TStarExpansionInverse}{\frac{2^{\frac{1}{3}} C_2 m^{\frac{2}{3}} n^{\frac{2}{3}}}{C_6^{\frac{1}{3}}}}
\newcommand*{\pMax}{p_{\max}}
\newcommand*{\mMin}{m_{\min}}
\newcommand*{\mMinOne}{m_{\min1}}
\newcommand*{\mMinTwo}{m_{\min2}}
\newcommand*{\mMinThree}{m_{\min3}}
\newcommand*{\mMinFour}{m_{\min4}}
\newcommand*{\CInTermsOfp}{\frac{2 \times 150^{\frac{3}{2}} C_6^{\frac{1}{2}}}{C_2^{\frac{3}{2}} p^{\frac{3}{2}}}}
\newcommand*{\CInTermsOfpInverse}{\frac{C_2^{\frac{3}{2}} p^{\frac{3}{2}}}{2 \times 150^{\frac{3}{2}} C_6^{\frac{1}{2}}}}

\begin{document}

\twocolumn[

\aistatstitle{Scalable K-Medoids via True Error Bound and Familywise Bandits}

\aistatsauthor{Aravindakshan Babu, Saurabh Agarwal, Sudarshan Babu, Hariharan Chandrasekaran}
]

\begin{abstract}
K-Medoids(KM) is a standard clustering method, used extensively on semi-metric data.
Error analyses of KM have traditionally used an in-sample notion of error,
which can be far from the true error and suffer from generalization gap.
We formalize the true K-Medoid error based on the underlying data distribution.
We decompose the true error into fundamental statistical problems of: minimum estimation (ME) and
minimum mean estimation (MME). 
We provide a convergence result for MME. 
We show $\errMME$ decreases no slower than
$\Theta(\frac{1}{n^{\frac{2}{3}}})$, where $n$ is a measure of sample size.
Inspired by this bound, we propose a computationally efficient, distributed KM algorithm namely MCPAM.  
MCPAM has expected runtime $\mathcal{O}(km)$,
where $k$ is the number of medoids and $m$
is number of samples.
MCPAM provides massive computational savings for a small tradeoff in accuracy. 
We verify the quality and scaling properties of MCPAM on various datasets.
And achieve the hitherto unachieved feat of calculating the KM of 1 billion points on semi-metric spaces.
\end{abstract}

\input{introduction.tex}

\input{accuracy_problem_setup.tex}


\input{solutions_to_theory_problems.tex}

\input{mcpam.tex}

\input{conclusions_and_future_work.tex}

\bibliographystyle{plain}
\bibliography{kmedoids_references}

\pagebreak

\appendix

\input{sample_medoid_error_analysis.tex}

\input{statistical_inference_of_minimum_of_sample_means.tex}

\input{mcpam_theory.tex}

\end{document}

%% file: introduction.tex
\section{Introduction}

K-medoids \cite{kaufman1990partitioning} is an extremely general clustering method.
Given a data sample, it finds K data points that are the centers of K clusters.
The K points together are called a K-medoid.
K-medoids is extensively used on semi-metric data, that doesn't necessarily respect the triangle inequality.
Eg internet RTT \cite{fraignaudRTT},
RNA-Seq analysis \cite{ntranos2016fast},
recommender systems \cite{leskovec2014mining}
where Euclidean metric doesn't apply.
See \cite{kaufman1990partitioning, bagaria2017medoids} for good introductions.
Other K-medoid algorithms include
\cite{ng2002clarans, kaufman2008clustering, song2017pamae, eppstein2001fast, okamoto2008ranking,newling2016sub}.

Traditionally, K-medoids has used an in-sample notion of error \cite{eppstein2001fast, bagaria2017medoids}.
For example in the $K = 1$ case, the most central point of the sample is taken as the reference/true point
to calculate error.
This is analogous 
\footnote{
The analogy can be formalized by recognizing the medoid
as a generalization of the mean 
}
to using the sample mean $\hat{\mu}$ 
as the reference to calculate error,
instead of the true mean $\mu$.
However $\hat{\mu}$ is a fluctuating quantity and $|\hat{\mu} - \mu|$ can be quite large.
To reiterate, any dataset is a random, limited sampling of an underlying distribution $\mathcal{D}$.
All sample quantities are fluctuating, noisy approximations to underlying distributional quantities,
and should be treated as such.
Again, from the perspective of mixture models,
using the true mean $\mu$ as the reference point is standard practice 
\cite{vempala2004spectral,regev2017learning,Louis1982EMInformationMatrix,OakesEMInformationMatrix,basford1997Standard}.
The use of in-sample error can lead to 
 generalization gap issues \cite{hansen1996Generalization} 
%
%
Our contributions include:
\begin{itemize}
\item Formalizing the true K-medoid error
\item Fundamental insights into the $K$-medoids problem, by showing that it
decomposes into two basic statistical problems: minimum estimation (ME) \& minimum mean estimation (MME).
\item A fundamental convergence result for MME. We show $\errMME$ decreases no slower than
$\Theta(\frac{1}{n^{\frac{2}{3}}})$. Where $n$ is a measure of sample size.
\item Inspired by above analysis, a new extremely scalable K-medoid algorithm (MCPAM).
MCPAM has average runtime $\mathcal{O}(km)$, where $k$ is number of medoids and $m$
is number of samples.
This makes it the first linear 1-medoid algorithm (expected runtime).
\end{itemize}
We provide detailed comparisons of these contributions to prior literature in sections
\ref{subsection_related_work_theory},\ref{subsection_theoretical_guarantees_related_work_mcpam}.

%% file: accuracy_problem_setup.tex
\section{Problem Formulation}
\label{section_problem_formulation}

Let $(\underlyingSpace, d, \underlyingSpaceDistribution)$ be a semi-metric space
equipped with a probability distribution $\underlyingSpaceDistribution$.
Let $\kTupleSpace = \underlyingSpace \times \cdots \times \underlyingSpace$ 
be the cartesian product of $k$ copies of $\underlyingSpace$.
An element $\kTuple \in \kTupleSpace$ is a K-tuple and its $l^{\text{th}}$ entry is denoted $\kTuple[l]$.
The distance $\kTupleToUnderlyingDistance$ 
from a point $x \in \underlyingSpace$ 
to a K-tuple $\kTuple \in \kTupleSpace$
is the minimum of the $k$ componentwise distances: 
$\kTupleToUnderlyingDistance(x, \kTuple) = \min_{l = 1}^{k} d(x, \kTuple[l])$.
Intuitively, $\kTuple$ represents $K$ cluster centers and 
$\kTupleToUnderlyingDistance$ is the distance to the nearest center.
This is the standard distance for K-medoids \cite{kaufman1990partitioning}.
The probabilistic setting has been explored in \cite{newling2016sub},
but only from the perspective of runtime calculation, not for error calculation.

\begin{definition}[Eccentricity and Medoid]
The average distance of a K-tuple $\kTuple$ to the points in $\underlyingSpace$ is the 
eccentricity $\Ecc$. A k-medoid $\medoid$ is a K-tuple that has minimum eccentricity.

\begin{alignat*}{2}
\Ecc(\kTuple) & := \expectation_{X \sim \underlyingSpaceDistribution}
    \kTupleToUnderlyingDistance(X, \kTuple) & \quad \quad 
    \medoid & := \argmin_{\kTuple \in \kTupleSpace} \Ecc(\kTuple)
\end{alignat*}
\end{definition}

Eccentricity is inverse centrality.
We now develop estimators for $\medoid$.
Let $\kTupleRandomVariable$ be a $\kTupleSpace$ valued random variable.
Let $\kTupleSample = \{ \kTupleRandomVariable_i \}_{i = 1}^m$ be $m$ random observations from $\kTupleSpace$.

\begin{definition}[True Sample Medoid]
The true sample medoid $\trueEccSampleMedoid$ is a minimizer of $\Ecc$,
over the K-tuples in $\kTupleSample$:
$
\trueEccSampleMedoid(\kTupleSample) := \argmin_{\kTupleRandomVariable_i \in \kTupleSample} \Ecc(\kTupleRandomVariable_i)
$
\end{definition}

It is widely recognized \cite{bagaria2017medoids} that K-medoids on $\mathbb{R}$ with k = 1 
and the $L_1$ metric is the median.
We will use this in a running example to illustrate various concepts.

\begin{example}[1-medoid on $\mathbb{R}$]
\label{example_1_medoid_on_R_part_a}
Let 
$\mathbb{X} = \mathbb{R}$,
$d$ be $L_1$ metric, 
$\underlyingSpaceDistribution = \mathcal{N}(\mu = 100, \sigma = 100)$ \& $K = 1$.
Now $\mathbb{X}^K = \mathbb{R}$ \&
$\Ecc(\kTuple) =$
\mbox{$\int_{-\infty}^{\infty} |x - \bar{x}|
    \frac{1}{\sqrt{2 * \pi * 100^2}} \exp{-\frac{(x - 100)^2}{2*100^2}} dx$}.
This gives $\nu = 100$ (the median).
Given a sample of $\mathbb{X}^K$, $\kTupleSample = \{90,\ 170,\ 60,\ 200,\ 190\}$,
$\Ecc(90) \simeq 80.187$.
This is the minimum on $\kTupleSample$, so $\trueEccSampleMedoid = 90$.
\end{example}

But $\Ecc$ is not computable in practice.
We need to approximate it.
Let $\underlyingSpaceSample = \{\underlyingSpaceRandomVariable_{j}\}_{j = 1}^{n}$
be a random sample of size $n$ from $\underlyingSpace$.
Let us have $m$ such random samples
$\underlyingSpaceSample_{i} = \{\underlyingSpaceRandomVariable_{ij}\}_{j = 1}^{n}$ $i \in [1,m]$.
Let all $\underlyingSpaceRandomVariable_{ij}$ follow distribution $\underlyingSpaceDistribution$.

\begin{definition}[Sample Eccentricity and Sample Medoid]
The average distance of a K-tuple $\kTuple$ to a random sample 
$\underlyingSpaceSample = \{\underlyingSpaceRandomVariable_{j}\}_{j = 1}^{n}$
of $\underlyingSpace$
is the sample eccentricity $\EccHat$.
The sample K-medoid $\sampleMedoid$ is a K-tuple from $\kTupleSample$ that 
has minimum sample eccentricity:

\begin{alignat*}{1}
\EccHat(\kTuple, \underlyingSpaceSample) & :=
    \frac{1}{n} \sum_{j = 1}^{n} \kTupleToUnderlyingDistance(\kTuple, \underlyingSpaceRandomVariable_j) \\
\sampleMedoid(\kTupleSample, \underlyingSpaceSample_1, \cdots, \underlyingSpaceSample_m) & := 
    \argmin_{\kTupleRandomVariable_i \in \kTupleSample} \EccHat(\kTupleRandomVariable_i, \underlyingSpaceSample_i)
\end{alignat*}
\end{definition}

$\trueEccSampleMedoid$ has one level of approximation to $\medoid$, 
namely the use of $\kTupleSample$ as a proxy for $\kTupleSpace$.
Whereas $\sampleMedoid$ has two levels of approximation to $\medoid$,
the additional one being the use of $T$ as a proxy for $\underlyingSpace$.

\begin{example}[1-medoid on $\mathbb{R}$ (cont'd)]
\label{example_1_medoid_on_R_part_b}
Consider the setting of example \ref{example_1_medoid_on_R_part_a}.
Let $T = \kTupleSample$.
Then $\EccHat(\kTuple) = \frac{1}{5} \sum_{x \in T}  |x - \kTuple|$ 
and $\sampleMedoid = 170$.
With two more data points $\{90,\ 170,\ 60,\ 200,\ 190,\ -10,\ 150\}$,
$\sampleMedoid = 150$.
$\trueEccSampleMedoid$ and $\medoid$ do not change.
\end{example}

The $\kTupleSample$, $m$, $\underlyingSpaceSample_i$ and $n$ are central to $\sampleMedoid$.
We will reuse them throughout the paper, so we reiterate them as a formal definition.
\begin{definition}[$\kTupleSpace$ Sample: $\kTupleSample$]
$\kTupleSample = \{\kTupleRandomVariable_i\}_{i = 1}^m$ is a sample of size $m$ from $\kTupleSpace$.
Each $\kTupleRandomVariable_i$ is a $\kTupleSpace$ valued random variable.
\end{definition}
\begin{definition}[$\underlyingSpace$ Sample: $\underlyingSpaceSample$]
For each $\kTupleRandomVariable_i \in S$, we have a sample $\underlyingSpaceSample_i$
of size $n$ from $\underlyingSpace$.
$\underlyingSpaceSample_i = \{\underlyingSpaceRandomVariable_{ij}\}_{j = 1}^n$ will be used
to estimate eccentricity of $\kTupleRandomVariable_i$.
$\underlyingSpaceRandomVariable_{ij}$ is a $\underlyingSpaceDistribution$ distributed 
random variable.
\end{definition}

We now express a number of existing K-medoid algorithms as sample medoids
$\sampleMedoid(\kTupleSample, \underlyingSpaceSample_1, \ldots, \underlyingSpaceSample_m)$
by appropriate choice of $\kTupleSample, T_i$.
Most existing algorithms derive the $\kTupleSample$ and $T_i$ samples
from a common iid sample $R$ of $\underlyingSpace$.
For instance in PAM \cite{kaufman1990partitioning},
$R = \{X_1, \ldots, X_n\}$ is $n$ iid samples from $\underlyingSpace$.
The $T_i$ are all equal to $R$, i.e. $T_i = \{X_{ij}\}_{j = 1}^n = \{X_j\}_{j = 1}^n = R$.
$\kTupleSample$ is a subset of $R^K = R \times \cdot \times R$.
In more detail, $\kTupleSample$ is constructed by picking a $\kTupleRandomVariable_{\text{current}}$
from $R^K$ at random and then:

\begin{enumerate}
\item Calculating $\EccHat$ for all single swap
\footnote{
A single swap neighbour of a $\bar{X}$ is got by
swapping exactly one entry of $\bar{X}$ with another $X_i \in R$
}
neighbours of $\kTupleRandomVariable_{\text{current}}$
\item Setting $\kTupleRandomVariable_{\text{current}}$ to neighbour with lowest $\EccHat$
\item Repeating from step 1 until no further decrease in $\EccHat$
\end{enumerate}

$\kTupleSample$ is all the $\kTupleRandomVariable$ for which $\EccHat$ is calculated
(and minimized over).
$\kTupleSample$ is a function of $R^K$. Since $R^K$ is a random sample of $\mathbb{X}^K$,
it follows that $\kTupleSample$ is too. 
Finally note $m = |\kTupleSample| \leq |R^K| = n^K$.
We also describe IPAM, an iid version of PAM.
Here, $\kTupleSample$ is derived from the first $n$ samples of $R$ in the same way.
However for each $\EccHat$ calculation in step 1, a fresh batch of $n$ samples are taken.

CLARA \cite{kaufman2008clustering} is essentially PAM, but $n$ is chosen to be quite small.
Other algorithms such as CLARANS \cite{ng2002clarans} and
RAND \cite{eppstein2001fast}, 
are also expressed as $\sampleMedoid$
in table \ref{table_expression_as_sampleMedoid}.
Algorithms TOPRANK \cite{okamoto2008ranking}, trimed \cite{newling2016sub} and
meddit \cite{bagaria2017medoids} are closely related to our framework.
They solve ME problem via an exhaustive exploration of $\kTupleSample$.
They solve the MME problem via different estimators, for example trimed uses the
triangle inequality to limit $n$ for certain $\EccHat(\kTupleRandomVariable_i)$ evaluation, etc.
We do not express iid versions of algorithms besides PAM.
But the modification is the same:
for each $\EccHat$ evaluation, use a fresh sample from $\underlyingSpace$.

\begin{table}
\caption{\label{table_expression_as_sampleMedoid}
Various K-medoid algorithms as $\sampleMedoid$
}
\tiny
	\begin{tabular}{>{\raggedright}p{0.2\linewidth}>{\raggedright}p{0.2\linewidth}>{\raggedright}p{0.25\linewidth}>{\raggedright\arraybackslash}p{0.15\linewidth}}
\toprule
Estimator & $R$ (iid samples)            & $\kTupleSample$                   & $T_i$ \\
\midrule
EXHAUSTIVE & $X_1, \ldots, X_n$          & $R^K$                 & $R$  \\
\hline
PAM       & $X_1, \ldots, X_n$          & Walk through $R^K$  & $R$  \\
\hline
IPAM      & $X_1, \ldots, X_{n + m n}$ &  Walk through $R_1^K$ 
                                         ($R_1 = \{X_i\}_{i = 1}^n$) & $\{X_{n i + j}\}_{j = 1}^n$,
                                                                        $i \in [1,m]$  \\
\hline
CLARANS   & $X_1,\ldots,X_n$    & Walk through $R^K$    & R \\
\hline
RAND      & $X_1,\dots,X_m$     & $R$     & random subset of $R$ 
                                            $|T_i| < m$ \\
\bottomrule
\end{tabular}
\end{table}

\subsection{Issues with Existing K-medoid Error; Definition of True k-Medoid Error}

The literature has traditionally focused on bounding the error relative to the
EXHAUSTIVE $\sampleMedoid$ (see \cite{eppstein2001fast, song2017pamae}),
in some cases even to zero (see \cite{okamoto2008ranking, newling2016sub, bagaria2017medoids}).
This is important and valuable, but uses a sample quantity as reference.
Using $\sampleMedoid$ as reference is akin to using sample mean (resp. median)
as the reference for mean (resp. median) estimation.
$\sampleMedoid$ is a fluctuating sample quantity 
and can be quite far from the true center of the distribution.
For the K = 1 case, EXHAUSTIVE $\sampleMedoid$ is the sample median and
the fluctuation of $\sampleMedoid$ is seen in examples
\ref{example_1_medoid_on_R_part_a} \& \ref{example_1_medoid_on_R_part_b}.
The median example is quite pertinent, since
the random nature of sample median
is widely recognized and various confidence intervals and estimates of error to true median
have been developed for the sample median \cite{nair1940Median,chu1955SampleMedian}.
In the extreme case,
$\sampleMedoid$ may not be a consistent estimator of $\medoid$ and may never converge.
Taking an example from mean estimation, we note that
the sample mean of the Cauchy distribution never converges.

From a ML perspective, in-sample errors suffer from generalization gap.
This has been studied in the context of clustering in \cite{hansen1996Generalization}.
For the same reason it is standard practice in mixture modeling
\cite{vempala2004spectral,regev2017learning,Louis1982EMInformationMatrix,OakesEMInformationMatrix,basford1997Standard}.
to use
the true centers as the references.
Errors using $\medoid$ as reference will address all these gaps.
We now define error with respect to $\medoid$.

\begin{definition}[K-Medoid Errors]
\label{definition_KM_errors}
We have two partial notions of errors:
\begin{alignat*}{2}
\errOne & :=
    \frac{
    \expectation
    |\Ecc(\sampleMedoid) - \Ecc(\trueEccSampleMedoid)|
    }{
    \expectation(\Ecc(\trueEccSampleMedoid))
    } & \quad
\errTwo & :=
    \frac{
    \expectation
    |\Ecc(\trueEccSampleMedoid) - \Ecc(\medoid)|
    }{
    \expectation(\Ecc(\medoid))
    }
\end{alignat*}
And one true error
\begin{alignat*}{1}
\errThree & :=
    \frac{
    \expectation
    |\Ecc(\sampleMedoid) - \Ecc(\medoid)|
    }{
    \expectation(\Ecc(\medoid))
    }
\end{alignat*}
\end{definition}


We study $\errOne$ and $\errTwo$, as they provide a crucial decomposition of $\errThree$
into canonical problems.


%% file: solutions_to_theory_problems.tex
\section{Results on K-Medoids Errors}
\label{section_kmedoids_theory_results}

The chief results of this section:
\begin{itemize}
\item Decomposition of $\errThree$ into two canonical problems
\item A new fundamental convergence result for one of these,
    namely the minimum mean estimation (MME) problem.
\end{itemize}

\subsection{$\errThree$ Decomposition}
We start by defining two fundamental statistical problems \ref{definition_MME} \ref{definition_ME}.
Then we decompose $\errThree$ into these problems.

\begin{definition}[Minimum Mean Estimation (MME)]
\label{definition_MME}
Let $m$ distributions $\{\mathcal{E}_i\}_{i = 1}^m$ be given.
Let $n$ samples from each distribution be given, $E = \{E_{ij}\}, i \in [1,m], j \in [1,n]$.
We want to identify the distribution with minimum mean.

Let $\mu_i$ be the mean of the $i$th distribution.
Let $i_{\min}$ be the index of the minimum mean distribution.
Let $\hat{i}_{\min}(E)$ be an estimate.
Then define the MME error as $\errMME := |\mu_{\hat{i}_{\min}} - \mu_{i_{\min}}| / |\mu_{i_{\min}}|$
\end{definition}

\begin{definition}[Minimum Estimation]
\label{definition_ME}
Given $m$ samples from a distribution, estimate the minimum of the distribution $x_{\min}$ 
using the min of the sample $\hat{x}_{\min}$.
Then define the ME error as $\errME := |\hat{x}_{\min} - x_{\min}| / |x_{\min}|$
\end{definition}

\begin{reptheorem}[\ref{theorem_decomposition_of_errThree_into_canonical_problems}]
\begin{alignat*}{1}
\errThree := \errMME + \errMME \errME + \errME
\end{alignat*}
\end{reptheorem}

The structuring of the true error $\errThree$ 
(proof at appendix \ref{section_medoid_error_analysis},
theorem \ref{theorem_decomposition_of_errThree_into_canonical_problems})
into these canonical problems holds with great generality.
This is one of our chief contributions, as it reveals the internal workings of K-medoids.
Essentially K-medoids has been decomposed into a ME problem followed by a MME problem.

The ME problem is well studied in the Extreme Value Theory literature
\cite{balkema1990EVT,smith1982EVT,leadbetterEVT}.
A central theorem in EVT is the Fisher-Tippet-Gnedenko (FTG) theorem providing asymptotic distributions
for the sample minimum. This reveals a deep connection between the K-medoids problem
and the dynamics of the FTG theorem.
The various K-medoid strategies are now re-interpreted as variance reduction strategies
for sample min.
The path is now open to leverage the extensive EVT literature for better min estimation strategies.

The MME problem is a novel continuous variant of the discrete best arm identification problem 
found in \cite{gabillon2011BAI,bagaria2017medoids}.
Formulation of the MME problem as a stand-alone problem will again help craft novel 
K-medoid strategies by solving MME in isolation.
Finally, examining the tradeoff between ME and MME components of the error
will be crucial in designing optimal K-medoid algorithms.
An initial version of this has been successfully done in section \ref{section_MCPAM}
to develop the MCPAM algorithm.

\subsection{Error Bounds}

Before we develop bounds on errors, we develop a model for distance distributions.
Given a R.V $\underlyingSpaceRandomVariable$ with distribution $\underlyingSpaceDistribution$.
We have a family of random variables
$\kTupleToUnderlyingDistance(\underlyingSpaceRandomVariable, \kTuple)$,
indexed by $\kTuple \in \kTupleSpace$.
This is a family of distance distributions.
We have developed a very general model, called the power-variance family 
to encode such families.
\begin{definition}
\label{definition_power_variance_distribution_family}
Consider a family $\mathcal{F}$ of distributions $\mathcal{E}$ parametrized by the mean $\mu$:
$
\mathcal{F} = \{\mathcal{E}(\mu) | \mu \in \mathbb{R}^{>0}\}
$
Such that the variance is a power function of the mean:
$
\sigma^2(\mu) := \alpha \mu^{\beta} + k
$
The parameters $\alpha, \beta, \gamma, k$ are constants for the entire family $\mathcal{F}$
and satisfy:
$
\beta \geq 0,\ \alpha > 0,\ \mu \geq \gamma > 0,\ \alpha \gamma^{\beta} \geq -k
$
\end{definition}

\begin{definition}
\label{definition_power_variance_compatible_triple}
Consider a triple $(\underlyingSpace, d, \underlyingSpaceDistribution)$.
Let $\underlyingSpaceRandomVariable$ be a R.V with distribution $\underlyingSpaceDistribution$.
We then have the family of random variables
$\mathcal{F} = 
    \{\kTupleToUnderlyingDistance(\underlyingSpaceRandomVariable, \kTuple) | \kTuple \in \kTupleSpace\}$.
If $\mathcal{F}$ is a power-variance family
as per definition \ref{definition_power_variance_distribution_family},
then the triple $(\underlyingSpace, d, \underlyingSpaceDistribution)$ is termed power-variance
compatible.
\end{definition}

The inequality $\mu \geq \gamma > 0$ models the strict positivity of
distances.
The inequality \mbox{$\alpha \gamma^{\beta} \geq -k$} is needed to ensure positivity of $\sigma^2$.
This model is quite general and permits us to easily encode
distance distributions.
As a concrete example of definitions 
\ref{definition_power_variance_distribution_family} \&
\ref{definition_power_variance_compatible_triple},
consider example \ref{example_eccentricity_of_gaussian}.

To calculate $\expectation \errMME$ requires calculating probabilities of the form:
\mbox{
$
\mathbb{P}(\hat{\mu}_{i} < \hat{\mu}_{1} \cdots \hat{\mu}_{(i - 1)},\hat{\mu}_{(i + 1)} \cdots \hat{\mu}_{m})
$}.
Where $\hat{\mu}_i$ is the sample mean of $\mathcal{E}(\mu_i) \in \mathcal{F}$.
Even for $m = 3$ and Gaussian $\mathcal{E}$
this calculation requires the integration of a complicated function of a polynomial of
$\Phi$ (standard normal cdf) and is intractable. One of our chief theorems is the upper bound of 
expected MME error 
(theorem \ref{theorem_upper_bound_on_n_for_given_tolerance} from appendix 
\ref{subsection_m_mean_case_main_results}).
In simplified form:

\begin{reptheorem}[\ref{theorem_upper_bound_on_n_for_given_tolerance}]
Let $\mathcal{F}$ be a power-variance distribution family, having
continuous cdfs and uniformly bounded kurtosis\footnote{
kurtosis (or Pearson kurtosis) is the fourth standardized moment $\expectation \left[\frac{X  - \mu}{\sigma}\right]^4$.
} $\leq \kurtosisUB$.
Let $\errMME$ be as in definition \ref{definition_MME}
and let the samples from the distributions be independent.
Let $0 < p \leq \pMax, m > \mMin, \beta  < \frac{4}{3},
n = \left\lceil
    C m^{1/2} / p^{3/2}
    \right\rceil$ .
Where $\pMax$ \& $\mMin$ are constants defined in defns
\ref{definition_pMax} \& \ref{definition_mMin}.
And $C$ is a constant for a given $\mathcal{F}$.
Then:
$
\errMME < \frac{p}{100} 
$
\end{reptheorem}

In other words, under general conditions, we can control the MME error to arbitrarily small tolerance $p$ with
a $n = \Theta(m^{1/2}/p^{3/2})$.
A more general version is theorem \ref{theorem_minimum_mean_estimation_error_convergence_rate}
in appendix \ref{subsection_m_mean_case_main_results}.
Under general conditions it gives the bound:
$
\errMME < \Theta(m^{1/3}/n^{2/3})
$
These are fundamental results for MME and one of our chief contributions. They are the analog of the
standard square root rate of convergence of Monte Carlo algorithms. 
Unlike the asympototic CLT based convergence our result is an {\em exact} 
upper bound when its conditions are satisfied.

Next, we use this bound to control $\errOne$ for iid estimators
such as IPAM etc.
The proof is in theorem \ref{theorem_bound_on_absdiff_sampleMedoid_trueEccSampleMedoid}
in appendix \ref{section_medoid_error_analysis}).
A simplified statement:
\begin{reptheorem}[\ref{theorem_bound_on_absdiff_sampleMedoid_trueEccSampleMedoid}]
Consider a power-variance compatible triple
$(\underlyingSpace, d, \underlyingSpaceDistribution)$.
Consider $\sampleMedoid$ with $S \perp T_i$ and iid $\{T_i\}_{i = 1}^m$.
Let the num samples $m,n$ and $p > 0$ be such that the
conditions of theorem \ref{theorem_upper_bound_on_n_for_given_tolerance} are met.
Then:
\begin{alignat*}{1}
\expectation \errOne & \leq \frac{p}{100}
\end{alignat*}
\end{reptheorem}

It is straightforward to show monotonic decrease of $\errME$ for raising $m$ 
via Jensens or FTG theorem.
Combined with the above result, this controls $\errThree$.

\subsection{Example and Experimental Verification}

We give a concrete example to illustrate the above.
And follow it up with experimental verification of the error bound.

\begin{example}[Eccentricity of Gaussian]
\label{example_eccentricity_of_gaussian}
Let
$(\underlyingSpace, d, \underlyingSpaceDistribution) = (\mathbb{R},  {||}.{||}_2^2, \normalDistribution(0,1))$
and $K = 1$.
Then $\underlyingSpace^K = \mathbb{R}$.
Let $\kTupleSample = \{\kTupleRandomVariable_i\}_{i = 1}^{m}$
and $\underlyingSpaceSample_i = \{\underlyingSpaceRandomVariable_{ij}\}_{j = 1}^{n}$ be as per IPAM:

\begin{alignat*}{1}
\kTupleRandomVariable_{i} & \sim \normalDistribution(0, 1) \quad i = [1,m] \\
\underlyingSpaceRandomVariable_{ij} & \sim \normalDistribution(0, 1) \quad i = [1,m], j = [1,n]
\end{alignat*}

Then

\begin{alignat*}{1}
\EccHat(\kTupleRandomVariable_i, \{\underlyingSpaceRandomVariable_{ij}\}_{j = 1}^m) & =
    \frac{1}{n} \sum_{j = 1}^{n} || \kTupleRandomVariable_i - \underlyingSpaceRandomVariable_{ij} {||}_2^2
\end{alignat*}

Given the $\kTupleRandomVariable_i$, the distances 
$\kTupleToUnderlyingDistance(\underlyingSpaceRandomVariable_{ij}, \kTupleRandomVariable_i)$
are distributed as

\begin{alignat*}{1}
|| \kTupleRandomVariable_i - \underlyingSpaceRandomVariable_{ij} {||}_2^2 \sim
    \chi^2(1, \underlyingSpaceRandomVariable_i^2)
\end{alignat*}

Where $\chi^2(1, \lambda)$  denotes a non-central Chi-squared random variable with 1 degree of 
freedom and having mean, variance:
$\mu = 1 + \lambda$, $\sigma^2 = 2(1 + 2 \lambda)$.
In our case
We can rewrite the variance as $\sigma^2 = 4 \mu - 2$.
This is a power-variance distribution with $\beta = 1, \alpha = 4, k = -2, \gamma = 1$.
Clearly all requirements on the parameters are satisfied, including $\alpha \gamma^\beta \geq -k$.

We need to uniformly upper bound kurtosis for our upper bounds to hold.
The kurtosis $\kappa$ of non-central Chi-squared R.V is given by:

\begin{alignat*}{1}
\kurtosis(\lambda) = 3 + 12 \frac{(1 + 4 \lambda)}{(1 + 2 \lambda)^2}
\end{alignat*}

This is strictly decreasing for all $\lambda > 0$.
And has a maximum at $\lambda = 0$. This gives the upper bound:

\begin{alignat*}{1}
\kurtosis(\lambda) \leq \kurtosis(0) = 15
\end{alignat*}
\end{example}


This distribution model is conceptually similar to the well known Tweedie family of distributions,
but more suited to model families of distance distributions.
In figure \ref{figure_experimental_verification},
we experimentally verify the error bound for the setting of example \ref{example_eccentricity_of_gaussian}.

\begin{figure}
\centering
\includegraphics[width=8cm]{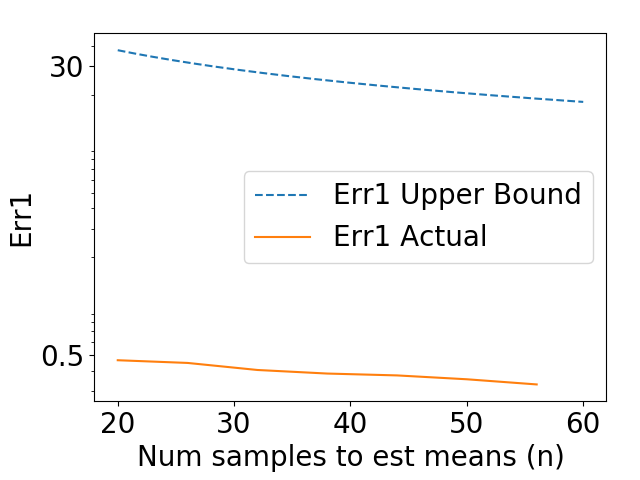}
\caption{
\label{figure_experimental_verification}
Experimental verification of $\errOne$ bound.
We sample $n$ times from each non central chi square distributions with 1 degree of freedom.
We use these $n$ samples to estimate the distribution with minimum mean and
then calculate $\errOne$. 
For each $n$, the experiment is repeated 200 times and the average $\errOne$ is calculated.
We confirm that the actual $\errOne$ is below the upper bound given by our analysis.
}
\end{figure}

\subsection{Related Work}
\label{subsection_related_work_theory}

MME like problems have been well studied under 
the rich theory of multi armed bandits.
However our setting has significant differences from earlier settings.
For instance, in stochastic linear bandits \cite{abbasi2011improved,dani2008stochastic}, 
the decision space is $\mathbb{R}^d$ and rewards are linear.
We don't make any assumptions about the decision space.

A more direct comparison can be made to \cite{agrawal2012analysis}.
The decision space $D$ is arbitrary. But the rewards over $D$ are restricted to be a discrete set 
with a gap parameter $\Delta > 0$.
Analyzing K-Medoids true error requires consideration of a continuous rewards set 
with $\Delta \in [0,\infty)$.
In exchange,
our problem has more structure in terms of the relation between the rewards and noise (power variance),
which they do not assume.
Due to the differing nature of problems the error bounds are also different 
$\Theta(\log(n)/n \times m/\Delta^4_{\max})$
vs $\Theta(m^{1/3}/n^{2/3})$. 
A direct comparison is not possible, but we have taken their total regret, normalized by total amount
of data ($m n$) as a rough analog and assumed all gaps $\Delta_{i \geq 2} = \Delta_{\max}$ (which decreases the bound).
There are also links to the best arm identification problems 
found in \cite{gabillon2011BAI,bagaria2017medoids}.
Again the discrete gap parameter $\Delta > 0$ in both settings means that the upper bounds found in those
papers go to infinity in our continuous setting $\Delta \in [0, \infty)$.
We refer to our MME setting as familywise bandits.
Since the worst case bound is over arms that are from a continuous family of distributions.

To our knowledge, we are the first to 
(i) define the true error $\errThree$
(ii) decompose $\errThree$ into canonical problems
(iii) prove a convergence result for familywise bandits / $\errOne$.
Various K-medoid error analyses include
\cite{eppstein2001fast,okamoto2008ranking,newling2016sub,bagaria2017medoids} (1-medoid), 
\cite{sergei2019clusteringbook_1} (PAM) and \cite{song2017pamae} (PAMAE).
However, error is calculated with respect to a sample medoid: EXHAUSTIVE $\sampleMedoid$.

%% file: mcpam.tex
\section{Faster, Large Scale $K$-Medoids}
\label{section_MCPAM}

Lemma \ref{lemma_upper_bound_relError_gaussian_case} 
(appendix \ref{subsection_generalization_of_2_mean_upper_bounds_to_m_means})
suggests that for Gaussian and sub-Gaussian eccentricity distributions,
$\errOne$ is exceedingly fast decreasing in $n$ for a given tolerance $T$.
Given that EVT error convergence rates are usually slow \cite{leadbetterEVT,jianwenEVT},
this suggests that more computational resources are used to control $\errME$
and fewer to control $\errMME$.
This suggests an optimal K-medoids regime of $m >> n$.
This is inline with the 1-medoid algorithms RAND, TOPRANK, trimed and meddit
And differs from the K-medoid algorithms PAM, CLARA, PAMAE ($m = \Theta(n)$)
\footnote{
To understand the statement $m = \Theta(n)$ for PAM, 
consider the earlier description of PAM.
Corresponding to a single iteration of PAM steps 1 - 3,
there are $k n$ single swaps.
It is well known that there are a small number of PAM iterations,
hence $m = \Theta(kn) = \Theta(n)$.
Similarly for other algorithms
}
\& CLARANS ($m << n$).
We propose a novel K-medoid algorithm MCPAM in the $m >> n$ regime.

MCPAM (Monte Carlo PAM) estimates the $n$ required for a given error tolerance
using a Monte Carlo approach (listing \ref{algorithm_MCPAM}).
The core idea is to estimate $\Ecc$ of candidate medoids
via sequential Monte Carlo sampling. We start with an initial medoid and $n = 10^3$ 
and increment by factors of $10$ until a medoid swap having lower $\Ecc$ with high probability, is found.
We chose $z_{\alpha}$ to give $1 - \alpha$ coverage and calculate symmetric confidence intervals: 
\begin{alignat*}{1}
\Var(\EccHat(\kTuple, \{y_j\}))
    & = \frac{1}{n} \sum_{j = 1}^{n} (d(y_j, \kTuple))^2 - \EccHat(\kTuple, \{y_j\}) \\
\EccHat_{\text{lo}}(\kTuple) 
    & = \EccHat(\kTuple) - z_{\alpha} (\Var(\EccHat(\kTuple)))^{\frac{1}{2}}
\end{alignat*}

\begin{algorithm*}
\caption{MCPAM}
\label{algorithm_MCPAM}
\begin{algorithmic}[1]
\REQUIRE Initial candidate k-medoid $\candidateMedoid$;
    samples of $\underlyingSpace$: $\{x_i\}$, $i = [1,m]$;
    consts $\tau \geq 0, n_{\max} > 0$
\STATE Calc $\EccHat(\candidateMedoid, \{x_i\})$,
    $\EccHat_{\text{lo}}(\candidateMedoid, \{x_i\})$, $\EccHat_{\text{hi}}(\candidateMedoid, \{x_i\})$
\STATE $\candidateMedoidUpdated \leftarrow \candidateMedoid$
\REPEAT
    \STATE $\candidateMedoid \leftarrow \candidateMedoidUpdated$ ;  $n \leftarrow 1000$
    \WHILE{$n < n_{\max}$}
        \STATE sample $\{y_j\}$ $j = [1,n]$ from $\{x_i\}$
        \STATE $\text{minhi} \leftarrow \argmin_{i,l \in [1,m], \times [1,k]} \EccHat_{\text{hi}}(\candidateMedoidSwap; \{y_j\})$ 
        \STATE $\text{minlo} \leftarrow \argmin_{i,l \in [1,m], \times [1,k]} \EccHat_{\text{lo}}(\candidateMedoidSwap; \{y_j\})$ 
        \IF{$\EccHat_{\text{hi}}(\candidateMedoid; \{x_i\}) < \EccHat_{\text{lo}}(\kTuple_{\text{minlo}}) + \tau$}
            \STATE break \quad \COMMENT{no significant improvement available}
        \ELSIF{$\EccHat_{\text{lo}}(\candidateMedoid; \{x_i\})) \geq \EccHat_{\text{hi}}(\kTuple_{\text{minhi}}) + \tau$}
            \STATE $\candidateMedoidUpdated \leftarrow  \kTuple_{\text{minhi}}$
                \quad \COMMENT{significant improvement available}
            \STATE Calc $\EccHat(\candidateMedoidUpdated, \{x_i\})$,
                $\EccHat_{\text{lo}}(\candidateMedoidUpdated, \{x_i\})$,
                $\EccHat_{\text{hi}}(\candidateMedoidUpdated, \{x_i\})$
            \STATE break
        \ELSE
            \STATE \COMMENT{significant improvement not found, but may exist}
            \STATE $n = 10 * n$ \quad \COMMENT{Try with larger sample size}
        \ENDIF 
    \ENDWHILE
\UNTIL{$\EccHat(\candidateMedoid; \{x_i\}) - \EccHat(\candidateMedoidUpdated; \{x_i\}) \leq 0$}
\RETURN $\candidateMedoid$, $\EccHat(\candidateMedoid, \{x_i\})$, 
    $\EccHat_{\text{lo}}(\candidateMedoid, \{x_i\})$, $\EccHat_{\text{hi}}(\candidateMedoid, \{x_i\})$
\end{algorithmic}
\end{algorithm*}

Loop 5 of MCPAM is called  MCPAM\_INNER.\\
{\textbf{practical optimizations:}} $\tau, n_{\max}$ are practical controls for time vs accuracy.
In the $m >> n$ regime, the confidence intervals of $\EccHat(\candidateMedoid, \{x_i\})$
are quite small compared to the error bars of $\EccHat(\candidateMedoidSwap, \{y_i\})$.
As a practical optimization we use $\EccHat(\candidateMedoid, \{x_i\})$ instead of 
$\EccHat_{\text{lo}}(\candidateMedoid, \{x_i\})$,
$\EccHat_{\text{hi}}(\candidateMedoid, \{x_i\})$.
Practically, we find a strong correlation
between $\Ecc_{\text{lo}}(\kTuple_{\text{minlo}})$ \& $\Ecc_{\text{hi}}(\kTuple_{\text{minhi}})$.
Hence we use $\kTuple_{\text{minhi}}$ for both lines 9 and 11, and combine by
exiting without swapping if 
$\EccHat(\kTuple_{\text{minhi}}, \{y_j\}) \pm z_{\alpha} \sigma_{\EccHat(\kTuple_{\text{minhi}}, \{y_j\})}$
is within $\tau$ of 
$\EccHat(\candidateMedoid, \{x_i\})$. Else we either swap or continue the loop depending on whether
$\EccHat(\candidateMedoid, \{x_i\})$ is higher or lower.

{\textbf{Single Medoid Variant:}} 
When $k = 1$, the set of swaps is unchanged as $\candidateMedoid$ changes (line 4).
This allows us to unify the outer and inner loops (lines 3 \& 5).
Hence we propose a simplified 1-medoid variant of MCPAM.
We repeat at line 18 unconditionally, modify line 14 to be a continue of loop 5.
and modify line 10 to terminate the program.
Then the only way to exit loop 5 and the program is via condition 9.

{\textbf{Distributed MCPAM:}} The computationally heavy steps are 
7,8 $\mathcal{O}(kmn)$ \& 1,13 $\mathcal{O}(km)$.
These are parallel operations on $\{x_i\}$ and each $x_i$ needs the full set $\{y_j\}$ for its calculations.
Assume the $n$ chosen by MCPAM will be $n << m$ 
(this is confirmed in theorem \ref{theorem_mcpam_guarantees_k_geq_1}).
Then, given $c$ worker nodes and 1 master node,
$\{x_i\}$ is distributed in $c$ chunks to workers.
Since $n << m$, we simply sample on the master and broadcast
a copy of $\{y_j\}$ to each worker.
Our costs are given in theorem \ref{theorem_mcpam_guarantees_k_geq_1}.

\subsection{Theoretical Guarantees and Related Work}
\label{subsection_theoretical_guarantees_related_work_mcpam}

To simplify computational costs we assume $m,n > k$ in this subsection.
We analyze MCPAM\_INNER, with $\tau = 0, n_{\max} = \infty$
(we will lose some accuracy in exchange for speed if $\tau > 0, n_{\max} < \infty$).
In the MCPAM setting,
we assume that the minimum gap $\Delta$ between the eccentricities of all k-points
constructable on $\{x_i\}$,
is strictly positive and indpendent of $m$.
This is similar to the assumption in \cite{bagaria2017medoids}.
Simplified versions of our convergence results are:
\begin{reptheorem}[\ref{theorem_mcpam_guarantees_k_geq_1}][MCPAM $K \geq 1$]
When $\Delta > 0$ and $\Delta = \Theta(1)$.
MCPAM\_INNER (loop 5) has expected runtime $\mathcal{O}(km)$.
Upon exit from MCPAM\_INNER, with high probability we have either:
(i) a decrease in $\Ecc$ from $\Ecc(\candidateMedoid)$ or
(ii) the swap set constructed around $\candidateMedoid$ has no smaller $\Ecc$.
MCPAM provides a true confidence interval on the eccentricity for the final medoid estimate.
For the distributed version, computational cost is $\mathcal{O}(\frac{km}{c})$ 
and communication cost is $\Theta(1)$.
\end{reptheorem}
\begin{reptheorem}[\ref{theorem_mcpam_guarantees_k_1}][MCPAM $K = 1$]
When $\Delta > 0$ and $\Delta = \Theta(1)$.
MCPAM has expected runtime $\mathcal{O}(m)$.
It finds the sample medoid with high probability and provides a true confidence interval
on the eccentricity of the same.
\end{reptheorem}
Proofs are in appendix \ref{section_MCPAM_theory}: theorems \ref{theorem_mcpam_guarantees_k_1},
\ref{theorem_mcpam_guarantees_k_geq_1}.
The MCPAM outer loop almost always runs only a few times (like other K-medoid methods), so
$\mathcal{O}(km)$ is the practical runtime.
Note, the confidence intervals are not for estimates of $\Ecc(\medoid)$,
but are true confidence intervals for $\Ecc(\candidateMedoid)$.
To our knowledge, MCPAM is the first fully general (i.e. semi-metric) K-medoids algorithm to have:
(i) expected linear runtime in 1-medoid case
(ii) computational cost of $\mathcal{O}(km)$ (detailed comparisons in ln 224 - 237)
(iii) constant communication cost, providing massive scalability in distributed setting
(iv) been scaled to 1 billion points.
All this while closely matching PAM’s quality (fig \ref{table_quality_comparison}).

We compare to the 1-medoid algorithms first.
RAND \cite{eppstein2001fast} computes $\epsilon$-approximate sample 1-medoid w.h.p in 
$\mathcal{O}(\frac{m \log(m)}{\epsilon})$ time for all finite datasets.
However, $\epsilon$ is measured relative to the network diameter $\Delta$, which can be fragile.
The 1-medoid algorithms TOPRANK, trimed, meddit all have various distributional assumptions on the data.
TOPRANK \cite{okamoto2008ranking} finds the sample 1-medoid w.h.p 
in $\mathcal{O}(m^{\frac{5}{3}} \log^{\frac{4}{3}} m)$. 
trimed \cite{newling2016sub} finds the sample 1-medoid in 
$\mathcal{O}(m^{\frac{3}{2}} 2^{\Theta(d)})$ for dimension $d$
while requiring the distances to satisfy the triangle inequality.
This and the exponential dependence on $d$ significantly limit the applicability of trimed.
meddit \cite{bagaria2017medoids} finds the sample 1-medoid w.h.p
in $\mathcal{O}(m \log(m))$ time, when $\Delta = \Theta(1) > 0$. These are all worst case times.

The runtimes per iteration for PAM \cite{kaufman1990partitioning}, 
CLARANS \cite{ng2002clarans} are $\mathcal{O}(km^2)$ 
(although swaps are subsampled in CLARANS, they are kept proportional to $km$).
CLARA \cite{kaufman2008clustering} is $\mathcal{O}(kn^2)$,
PAMAE \cite{song2017pamae} is $\mathcal{O}(m + r k n^2)$ (distributed $\mathcal{O}(m + k n^2)$),
where $n$ is sample size typically set to a multiple of $k$
and $r$ is number of reruns.
PAMAE has an accuracy guarantee, bounding the error in the estimate relative to the sample medoid.
However, PAMAE is rather restrictive requiring the data come from a normed vector space.

\subsection{Results and Comparisions}

We perform scaling and quality comparisons on various datasets from literature \cite{Ssets,mallah2013leaves,frey1991letter, yang2016foursquare} 
and 2 newly created datasets (appendix \ref{subsubsection_mcpam_datasets}).
We compare to PAM, as that is the 'gold standard' in quality for K-medoids.
Since PAM uses strictly more data than MCPAM (the $T_i$ are much smaller for MCPAM),
 PAM will have better quality.
However, per our motivation, we expect MCPAM to have a small drop in quality with
a massive speedup in runtime.
This is verified in table \ref{table_quality_comparison} and \ref{picture_single_node_scaling}.
MCPAM was run with K++ initialization \cite{Arthur07K++}, 10 times on each dataset.
PAM, MCPAM were run with the true number of clusters.
We compare to DBSCAN as it is widely used for non-Euclidean distributed clustering
\cite{esterKDD,gotzHPDBSCAN,heMRDBSCAN}.
To estimate epsilon parameter for DBSCAN, we (i) visually identify the knee of knn plot,
(ii) do extensive grid search around that (iii) pick epsilon giving maximum ARI.
The full details of our experimental setup are given in appendix \ref{subsubsection_mcpam_experimental_setup}.
We also compare MCPAM to PAM via clustering cost, again MCPAM is quite close to PAM in quality:
figure \ref{figure_cc_comparison} in appendix \ref{subsubsection_experimental_results_contd}.

The variance of eccentricity distributions tends to be correlated with
the mean.
Since $n$ is just required to control the variance of eccentricity distributions
relative to their means,
it is quite small for a variety of natural and artifical datasets, almost always $n = 10^3$ in practice.
This enables MCPAM to scale to very large dataset sizes.
Figure \ref{picture_single_node_scaling} shows MCPAM scaling linearly as 
$m$ increases keeping underlying distribution same.
Figure \ref{figure_distributed_results} 
shows distributed MCPAM results.
We are, to the best of our knowledge, the first team to scale a semi-metric K-medoids algorithm to
1 billion points (with 6 attributes)
(see BillionOne dataset appendix \ref{subsubsection_mcpam_datasets}).
PAMAE \cite{song2017pamae} does run on a dataset of around 4 billion,
but the data is restricted to Euclidean space,
other distributed semi-metric K-medoids algorithms
\cite{jiang2014parallel, zhang2004parallel, he2011parallel, yue2016parallel, martino2017efficient}
are not run at this scale.
To compare to other non-Euclidean algorithms, HPDBSCAN \cite{gotzHPDBSCAN} demonstrates runs
upto 82 million points with 4 attributes.

\begin{figure*}
\caption{
\label{table_quality_comparison}
Quality comparison between PAM and MCPAM.
We use ARI as it is based on underlying truth and is a loose proxy for $\errThree$.
Higher ARI is better \cite{rand1971objective}.
MCPAM ARI is quite close to PAM 8\% degradation on average, compare to
26\% for DBSCAN.
}
\input{pictures/quality_with_avg_with_sd.tex}
\end{figure*}

\begin{figure*}[!htb]
\centering
\includegraphics[width=13cm]{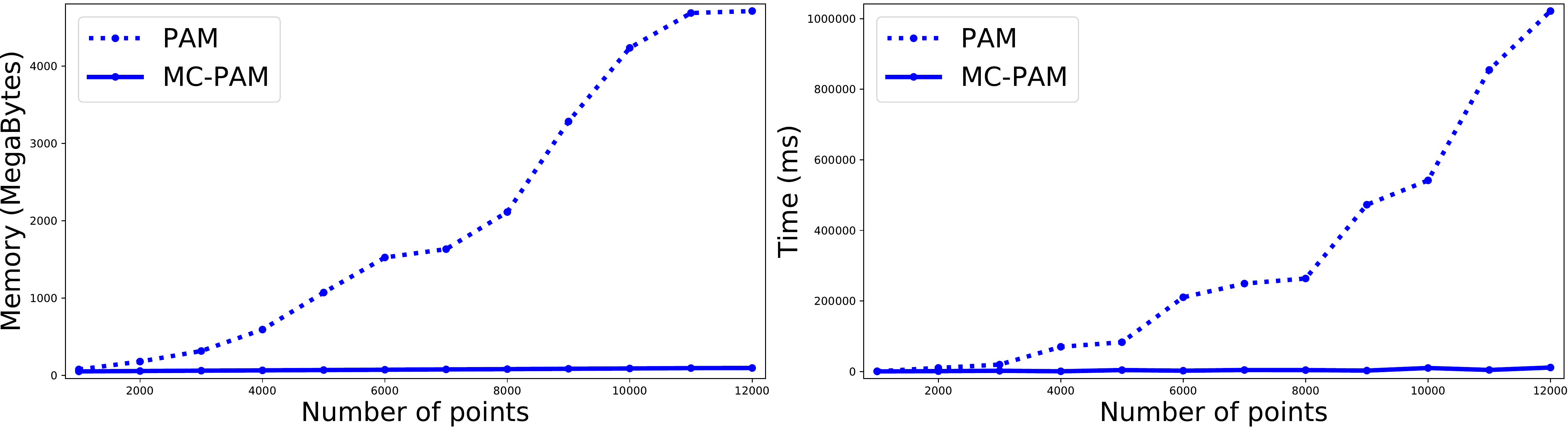}
\caption{
\label{picture_single_node_scaling}
Memory and time scaling of PAM and MCPAM on a single CPU,
as $m$ increases. The points were subsampled from the letters data set.
}
\end{figure*}

\begin{figure*}[!htb]
\centering
\includegraphics[width=13.7cm]{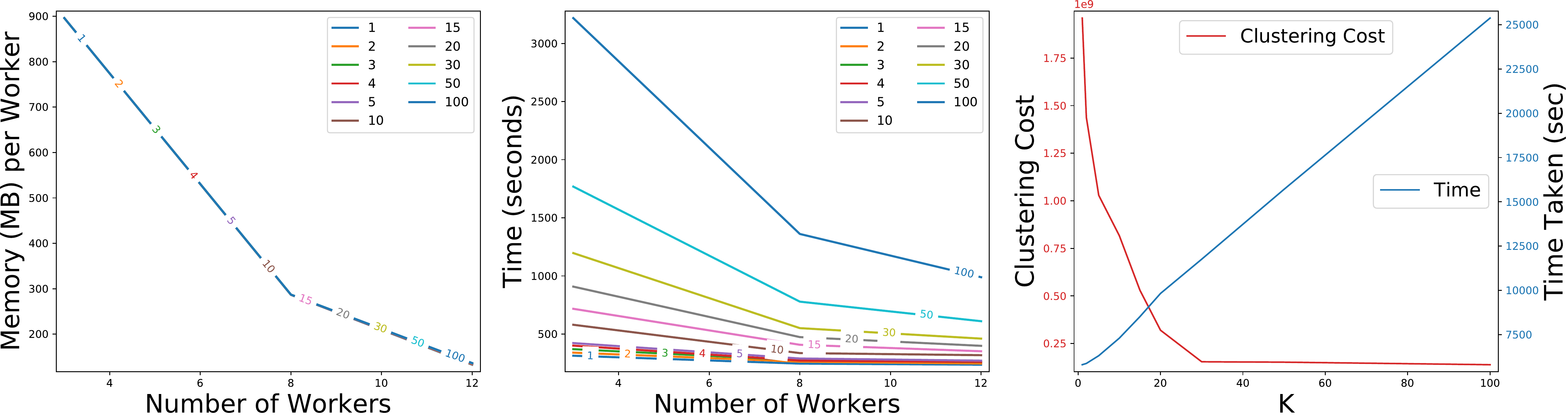}
\caption{
\label{figure_distributed_results}
The first \& second plots show
memory usage \& runtime of distributed MCPAM as num CPUs increase.
Each line corresponds to a particular $k$, showing the progression across $k$ also.
We see near-linear scale-up.
This is on 32 million point Foursquare data (appendix \ref{subsubsection_mcpam_datasets}).
y-axis quantities are medians across multiple MCPAM loop 3 iterations.
The $3^{\text{rd}}$ plot shows clustering cost \& runtime on BillionOne data 
\ref{subsubsection_mcpam_datasets} as $k$ increases, numCPUs = 12
(full scaling data in appendix \ref{subsubsection_experimental_results_contd},
table \ref{billion_scaling}).
}
\end{figure*}

%% file: pictures/quality_with_avg_with_sd.tex
\centering
\scalebox{0.7}{
\begin{tabular}{@{}|l|l|l|l|l|l|l|l|l|@{}}
\toprule
\textbf{Dataset} 
    & \textbf{K}
    & \textbf{PAM-ARI}
    & \multicolumn{1}{c|}{\textbf{\begin{tabular}[c]{@{}c@{}}MCPAM-ARI\\ (lo)\end{tabular}}} 
    & \multicolumn{1}{c|}{\textbf{\begin{tabular}[c]{@{}c@{}}MCPAM-ARI\\ (mean)\end{tabular}}} 
    & \multicolumn{1}{c|}{\textbf{\begin{tabular}[c]{@{}c@{}}MCPAM-ARI\\ (hi)\end{tabular}}} 
    & \textbf{DBSCAN-ARI}
    & \multicolumn{1}{c|}{\textbf{\begin{tabular}[c]{@{}c@{}}MCPAM \\ rel. err\end{tabular}}} 
    & \multicolumn{1}{c|}{\textbf{\begin{tabular}[c]{@{}c@{}}DBSCAN \\ rel. err\end{tabular}}} \\ \midrule
S1 & 15 & 0.985 & 0.98 & 0.983 & 0.986 & 0.947 & 0.268 & 3.91 \\ \midrule 
S2 & 15 & 0.934 & 0.9 & 0.92 & 0.939 & 0.664 & 1.5 & 28.8 \\ \midrule 
S3 & 15 & 0.72 & 0.639 & 0.684 & 0.73 & 0.448 & 4.99 & 37.8 \\ \midrule 
S4 & 15 & 0.635 & 0.529 & 0.576 & 0.623 & 0.394 & 9.22 & 37.9 \\ \midrule 
leaves & 100 & 0.278 & 0.194 & 0.211 & 0.228 & 0.0161 & 24.2 & 94.2 \\ \midrule 
letter1 & 26 & 0.264 & 0.249 & 0.268 & 0.287 & 0.408 & -1.76 & -54.8 \\ \midrule 
letter2 & 26 & 0.164 & 0.144 & 0.166 & 0.189 & 0.204 & -1.46 & -24.3 \\ \midrule 
letter3 & 26 & 0.196 & 0.158 & 0.193 & 0.228 & 0.254 & 1.61 & -29.4 \\ \midrule 
letter4 & 26 & 0.162 & 0.13 & 0.154 & 0.177 & 0.135 & 5.48 & 16.8 \\ \midrule 
M1 & 32 & 0.756 & 0.546 & 0.641 & 0.735 & 0.519 & 15.3 & 31.4 \\ \midrule 
M2 & 32 & 0.483 & 0.36 & 0.393 & 0.426 & 0.191 & 18.6 & 60.5 \\ \midrule 
M3 & 32 & 0.353 & 0.273 & 0.291 & 0.309 & 0.137 & 17.6 & 61.2 \\ \midrule 
M4 & 32 & 0.246 & 0.212 & 0.227 & 0.241 & 0.0696 & 7.81 & 71.7 \\ \midrule 
\multicolumn{7}{|l|}{Avg-all Datasets} & 7.954 & 25.819 \\ \bottomrule
\end{tabular}}

%% file: conclusions_and_future_work.tex
\section{Conclusions}

We have formalized the true K-medoid error $\errThree$, and revealed a core dynamic inside this error
by decomposing it into the ME and MME problems. We have a new 'familywise bandits' convergence
result to bound MME.
Inspired by the tradeoffs in these bounds, we have proposed the first $K( > 1)$ medoids algorithm
in the $n << m$ regime: MCPAM. MCPAM has good scaling and quality properties, and is the first
K-medoids algorithm to provide true confidence interval on $\Ecc$ of estimated medoid.
This work opens some interesting future directions:
(i) Identifying conditions under which $\sampleMedoid$ diverges.
Based on our study of the core $\zscoreOfZero$ function in 
appendix \ref{subsection_convenient_expression_for_2D_relative_error}, we suspect
there are diverging distributions when the power-variance parameter $\beta > 2$.
(ii) Our analysis combined with EVT provides a possible, natural 
approach to analyzing the number of K-medoid outer loop iterations.

%% file: sample_medoid_error_analysis.tex
\section{Medoid Error Analysis: $\text{err}_{1,2,3}$}
\label{section_medoid_error_analysis}

In this section we analyze $\errOne, \errTwo, \errThree$ 
and provide an error decomposition and bounds.

\begin{lemma}[Decomposition of $\errThree$ into $\errOne$, $\errTwo$]
\label{lemma_decomposition_of_errThree_into_canonical_problems}
\begin{alignat*}{1}
\errThree := \errOne + \errOne \errTwo + \errTwo
\end{alignat*}
\end{lemma}

\begin{proof}
For a given $S$ and $\{\underlyingSpaceSample_i\}_{i = 1}^m$:
\begin{alignat*}{1}
\errThree 
    & = \frac{\expectation \left|\Ecc(\sampleMedoid) - \Ecc(\medoid) \right|}{\expectation(\Ecc(\medoid))} \\
    & = \frac{\expectation \left|\Ecc(\sampleMedoid) - \Ecc(\trueEccSampleMedoid) + \Ecc(\trueEccSampleMedoid) - \Ecc(\medoid) \right|}{\expectation(\Ecc(\medoid))} \\
    & = \frac{\expectation \left|\Ecc(\sampleMedoid) - \Ecc(\trueEccSampleMedoid) \right| + \left| \Ecc(\trueEccSampleMedoid) - \Ecc(\medoid) \right|}{\expectation(\Ecc(\medoid))} \\
    & = \errOne \frac{\expectation(\Ecc(\trueEccSampleMedoid))}{\expectation(\Ecc(\medoid))} + \errTwo
\end{alignat*}
But $\expectation(\Ecc(\trueEccSampleMedoid)) = (1 + \errTwo) \expectation(\Ecc(\medoid))$.
So:
\begin{alignat*}{1}
\errThree 
    & = \errOne + \errOne \errTwo + \errTwo
\end{alignat*}

\end{proof}

\begin{theorem}[Decomposition of $\errThree$ into Canonical Problems]
\label{theorem_decomposition_of_errThree_into_canonical_problems}
\begin{alignat*}{1}
\errThree := \errMME + \errMME \errME + \errME
\end{alignat*}
\end{theorem}

\begin{proof}
Consider a given $S$ and $\{\underlyingSpaceSample_i\}_{i = 1}^m$.
For each $\kTupleRandomVariable_i \in S$ \& $\underlyingSpaceRandomVariable_{ij} \in \underlyingSpaceSample_i$,
$\kTupleToUnderlyingDistance(\underlyingSpaceRandomVariable_{ij}, \kTupleRandomVariable_i)$ 
is distributed with mean $\Ecc(\kTupleRandomVariable_i)$.
From the definitions it is clear that $\sampleMedoid$ is estimating 
the distribution with minimum mean from these distributions.
Further, $\trueEccSampleMedoid$ is the distribution with minimum mean.
Hence if we consider the minimum mean estimation problem with
the randomly chosen means $\Ecc(\kTupleRandomVariable_i)$ $\kTupleRandomVariable_i \in S$, we have:

\begin{alignat*}{1}
\errOne & = \errMME
\end{alignat*}

$S$ is a random sample of $\underlyingSpace^K$.
Hence $\{\Ecc(\kTupleRandomVariable) | \kTupleRandomVariable \in S\}$ is a random sample 
of $\{\Ecc(\kTuple) | \kTuple \in \underlyingSpace^K\}$.
The min $\trueEccSampleMedoid$ of the first set is estimating the min $\medoid$ of the second set.
Hence:

\begin{alignat*}{1}
\errTwo & = \errME
\end{alignat*}

The result is now immediate from lemma \ref{lemma_decomposition_of_errThree_into_canonical_problems}.
\end{proof}

\begin{theorem}[Bound on $\errOne$]
\label{theorem_bound_on_absdiff_sampleMedoid_trueEccSampleMedoid}
Consider a power-variance compatible triple 
$(\underlyingSpace, d, \underlyingSpaceDistribution)$.
Given samples $\kTupleSample, \underlyingSpaceSample_1, \ldots, \underlyingSpaceSample_m$,
let the following independence assumptions be satisfied:

\begin{itemize}
\item $\kTupleSample \perp \underlyingSpaceSample_i \quad \forall i \in [1,m]$
\item $\underlyingSpaceRandomVariable_{i j} \perp \underlyingSpaceRandomVariable_{i^{\prime}j^{\prime}}$
    for $i j \neq i^{\prime}j^{\prime}$
\end{itemize}

Let the num samples $m,n$ and $p > 0$ be such that the 
conditions of theorem \ref{theorem_upper_bound_on_n_for_given_tolerance} are met.
Then taking expectation over $\kTupleSample, \underlyingSpaceSample_1, \ldots, \underlyingSpaceSample_m$,
we have:
\begin{alignat*}{1}
\expectation \errOne & \leq \frac{p}{100}
\end{alignat*}
\end{theorem}

\begin{proof}
Expectations are over the joint distribution of $S, \{\underlyingSpaceSample_i\}_{i = 1}^m$ unless noted otherwise.
By independence assumption $S \perp \{\underlyingSpaceSample_i\}_{i = 1}^m$ we have that the joint law of 
$S$ and $\{\underlyingSpaceSample_i\}_{i = 1}^m$ is a product measure. By non-negativity of $\Ecc$ we can 
apply Fubini-Tonelli:
\begin{alignat}{1}
\label{equation_expectation_over_samples_intermediate_one}
\expectation \Ecc(\sampleMedoid) =
    \expectation_{S} \expectation_{\{\underlyingSpaceSample_i\}_{i = 1}^m} \Ecc(\sampleMedoid)
\end{alignat}

Consider the inner integral in more detail:

\begin{alignat*}{1}
\expectation_{\{\underlyingSpaceSample_i\}_{i = 1}^m} \Ecc(\sampleMedoid) =
    \expectation_{\{\underlyingSpaceSample_i\}_{i = 1}^m} \Ecc(\sampleMedoid(S, \underlyingSpaceSample_1, \ldots, \underlyingSpaceSample_m)) | S
\end{alignat*}

For a given $\kTupleRandomVariable_i \in S$,
$\kTupleToUnderlyingDistance(\underlyingSpaceRandomVariable_{ij}, \kTupleRandomVariable_i)$ 
is a function of R.V $\underlyingSpaceRandomVariable_{ij} \in \underlyingSpaceSample_i$
and is distributed with mean $\Ecc(\kTupleRandomVariable_i)$.
From the definitions it is clear that $\sampleMedoid$ is estimating 
the distribution with minimum mean from these distributions.
Further, $\trueEccSampleMedoid$ is the distribution with minimum mean.

By assumption these distributions belong to a power-variance family.
Further the assumption $X_{ij} \perp X_{i^{\prime}j^{\prime}}$ gives the
required iid structure. Thus we have satisfied the conditions 
of theorem $\ref{theorem_upper_bound_on_n_for_given_tolerance}$
(upper bound on $n$ for given error percentage $p$).
So:

\begin{alignat}{1}
\label{equation_expectation_over_samples_intermediate_two}
\expectation_{\{\underlyingSpaceSample_i\}_{i = 1}^m} \Ecc(\sampleMedoid) | S \leq
    (1 + \frac{p}{100}) \Ecc(\trueEccSampleMedoid) | S
\end{alignat}

Now:

\begin{alignat*}{1}
\expectation |\Ecc(\sampleMedoid) - \Ecc(\trueEccSampleMedoid)| & =
    \expectation \Ecc(\sampleMedoid) - \expectation \Ecc(\trueEccSampleMedoid) \\
    & \leq \expectation_{S} (1 + \frac{p}{100}) \Ecc(\trueEccSampleMedoid) - \expectation_{S} \Ecc(\trueEccSampleMedoid) \\
    & = \frac{p}{100} \expectation \Ecc(\trueEccSampleMedoid)
\end{alignat*}

Where we have used equations 
\ref{equation_expectation_over_samples_intermediate_one} and
\ref{equation_expectation_over_samples_intermediate_two}.
The result is now immediate.

\end{proof}

%% file: statistical_inference_of_minimum_of_sample_means.tex
\section{MME Standard Estimator: Introduction}
\label{appendix_MME_standard_estimator_introduction}

We define a standard estimator for the MME (Minimum Mean Estimation) problem defined in \ref{definition_MME}.
We are interested in bounding its error. We do this in 3 parts.
First we formulate the estimator and its expected error (appendix \ref{appendix_MME_standard_estimator_introduction}).
Next, we upper bound the error in the 2-D case, where we restrict to two distributions 
(appendix \ref{section_two_mean_case}).
Finally, we generalize the bound to the M-D case, where we have $m$ distributions
(appendix \ref{section_MD_mean_case}).

\input{statistical_inference_of_minimum_of_sample_means_setup.tex}

\input{2d_case.tex}

\input{md_rel_error_worst_case_bound.tex}

%% file: statistical_inference_of_minimum_of_sample_means_setup.tex
\subsection{MME Standard Estimator}

Consider the setting of the MME problem (definition \ref{definition_MME}).
We use the sample means $\hat{\mu}_i = \frac{1}{n} \sum_{j = 1}^{n} E_{ij}$ to estimate
$i_{\min}$.

\begin{definition}[MME Standard Estimator]
\label{definition_minimum_mean_estimator}
Let $\hat{i}_{\min}$ be the index of the minimum sample mean:
\begin{eqnarray*}
\hat{i}_{\min} := \argmin_i \hat{\mu}_i
\end{eqnarray*}
\end{definition}

\begin{definition}[Expected Error of $\hat{i}_{\min}$]
\begin{alignat*}{1}
\relError(\mu_1,\ldots,\mu_m) :=
     \expectation_{E_{ij}} \frac{\mu_{\hat{i}_{\min}}
        - \mu_{i_{\min}}}{\mu_{i_{\min}}}
\end{alignat*}
\end{definition}

Note $\relError =  \expectation \errMME$.
While dealing with the standard MME estimator 
(appendices \ref{appendix_MME_standard_estimator_introduction}, 
\ref{section_two_mean_case}, \ref{section_MD_mean_case}),
we use a few notational conveniences:
\begin{itemize}
\item $\hat{i}_{\min}$ will refer specifically to 
the estimator in definition \ref{definition_minimum_mean_estimator}.
\item $\expectation$ is implicitly over ${E_{ij}}$.
\item Akin to order statistics we use the notation $\mu_{i:m}$ to denote the $i$th smallest $\mu_i$.
There is the obvious mapping from the $i:m$ index to the $i$ index,
and $\mu_{1:m} = \mu_{i_{\min}}$.
Hereafter we use the $i:m$ indexing, for the most part.
\end{itemize}

We now derive a simple formula for the relative error.
\begin{definition}[Relative Exceedances]
\label{definition_relative_exceedances}
The relative exceedance of the $i:m^{\mathrm{th}}$ mean:
\begin{eqnarray*}
\delta_{i:m} & := & \frac{\mu_{i:m} - \mu_{1:m}}{\mu_{1:m}} 
\end{eqnarray*}
\end{definition}

\begin{definition}[Probability of Choosing $i:m^{\mathrm{th}}$ Mean]
\label{definition_probability_of_choosing_i:mth_mean}
The probability of the $i:m^{\mathrm{th}}$ mean estimate undershooting all other mean estimates:
\footnote{For the case where the semi-metric is almost surely restricted to a countable set of values,
ties are broken at random with equal probability for each possibility.
}:
\begin{alignat*}{1}
\mathbb{P}_{i:m}(\mu_{1:m}, \ldots, \mu_{m:m}) :=
    \mathbb{P}(\hat{\mu}_{i:m} < \hat{\mu}_{1:m} & \cdots \hat{\mu}_{(i - 1):m}, \\
        & \hat{\mu}_{(i + 1):m} \cdots \hat{\mu}_{m:m})
\end{alignat*}
\end{definition}

We now have the following obvious proposition.
\begin{proposition}[Formula for Expected Relative Error]
\label{proposition_formula_for_expected_relative_error}
We can reorder the $\mu_i$ without changing $\epsilon$,
and in particular:

\begin{eqnarray*}
\epsilon(\mu_1,\ldots,\mu_m) & = & \epsilon(\mu_{1:m},\ldots,\mu_{m:m})
\end{eqnarray*}

Also we have:
\begin{eqnarray*}
\epsilon(\mu_{1:m},\ldots,\mu_{m:m}) & = & \sum_{i = 2}^{m} \delta_{i:m} \mathbb{P}_{i:m}(\mu_{1:m}, \ldots, \mu_{i:m})
\end{eqnarray*}
\end{proposition}

In the next appendix we upper bound the relative error when $m = 2$.

%% file: 2d_case.tex
\input{2d_rel_err_expression.tex}

\input{2d_rel_err_worst_case_bound_abs.tex}

\input{2d_rel_err_worst_case_bound_gaussian.tex}

\input{2d_rel_err_worst_case_bound_general.tex}

%% file: 2d_rel_err_expression.tex
\section{The Two Mean Case}
\label{section_two_mean_case}

We want to upper bound $\relError$ for $m = 2$. We do this in four subsections:
\begin{itemize}
\item First, we find a more convenient expression for relative error in the 2D case
    (subsection \ref{subsection_convenient_expression_for_2D_relative_error})
\item Then, we find an abstract upper bound for $\relError$
    (subsection \ref{subsection_upper_bounding_relError_two_means_abstract_bound}).
    The abstract bound requires the existence of an upper bounding '$G$' function,
    satisfying certain properties.
\item Subsequently, we show the existence of such a $G$ function for
    the Gaussian case (subsection \ref{subsection_upper_bounding_relError_two_means_gaussian_case})
\item Finally, we show the existence of a $G$ function for arbitrary distributions
    (subsection \ref{subsection_upper_bounding_relError_two_means_general_case}), by extrapolating 
    out of the Gaussian case via the Berry-Esseen theorem.
\end{itemize}

\subsection{Convenient Expression for 2D Relative Error}
\label{subsection_convenient_expression_for_2D_relative_error}

In this subsection we derive a more convenient expression for relative error in the 2D case.
Let the distributions $\mathcal{E}_i$ come from a power-variance distribution family
$\mathcal{F}$
(see definition \ref{definition_power_variance_distribution_family}).
The $\mathcal{E}$ are parametized by $\familyParameters$, and hence so is $\relError$:
\begin{alignat*}{1}
    \relError(\mu_1,\ldots,\mu_m;\familyParametersAugumented)
\end{alignat*}
For notational convenience we suppress the dependence on $\familyParametersAugumented$,
and write:
\begin{alignat*}{1}
    \relError(\mu_1,\ldots,\mu_m)
\end{alignat*}
Further, we assume the samples $E_{ij}$ (see definition \ref{definition_power_variance_distribution_family})
are independent
i.e. $E_{i,j} \perp E_{i^{\prime},j^{\prime}}$ when $(i,j) \neq (i^{\prime},j^{\prime})$.

\begin{definition}[2D Relative Error (Reparametrized)]
\label{definition_relative_error_2d_case}
Given the set \mbox{$\{\mathcal{E}(\mu_{i:m}) | i = 1,2 \} \subset \mathcal{F}$}.
Let $\delta_{2:2}$ be as in definition \ref{definition_relative_exceedances}.
For convenience set $\delta  := \delta_{2:2}$.
So $\mu_{2:2} = (1 + \delta) \mu_{1:2}$ and
we have the obvious reparametrization of $\relError$:
\begin{eqnarray*}
\epsilon(\mu_{1:2}, \delta) := \epsilon(\mu_{1:2}, \mu_{2:2})
\end{eqnarray*}
\end{definition}

\begin{definition}[Difference of Samples]
\label{definition_difference_of_samples}
Define $D_j$ as the difference of the j-th pair of samples from the distributions $\mathcal{E}_{i:m}$:
\begin{eqnarray*}
D_j & = & E_{2:2,j} - E_{1:2,j}
\end{eqnarray*}
\end{definition}


\begin{definition}[Difference of Sample Means]
Define $\hat{d}$ as the difference of sample means:
\begin{eqnarray*}
\hat{d} & := & \hat{\mu}_{2:2} - \hat{\mu}_{1:2}
\end{eqnarray*}
\end{definition}

\begin{definition}
\label{definition_zscoreOfZero}
Define the z-score of zero ($\zscoreOfZero$) for the random variable $\hat{d}$:
\begin{eqnarray}
\label{equation_zscore_of_zero}
\zscoreOfZero & := & - \frac{\mathbb{E}(\hat{d})}{\sqrt{\Var{(\hat{d})}}}
\end{eqnarray}

Then using $\mu_{2:2} = (1 + \delta) \mu_{1:2}$, we get:
\begin{eqnarray*}
\mathbb{E}(\hat{d}) & = & \delta \mu_{1:2} \\
\Var(\hat{d}) & = & \frac{1}{n} \alpha (\mu_{1:2}^{\beta} + \mu_{2:2}^{\beta}) + 2k \\
    & = & \frac{\alpha}{n} \mu_{1:2}^{\beta} \left((1 + \delta)^{\beta} + 1 + \frac{2k}{\alpha \mu_{1:2}^{\beta}}\right)
\end{eqnarray*}
and:
\begin{alignat}{1}
\label{equation_expression_for_zscoreOfZero}
\zscoreOfZero(\mu_{1:2}, \delta) & = \\
    & - \frac{\mu_{1:2}^{1 - \frac{\beta}{2}} \sqrt{n} }{\alpha^{\frac{1}{2}}}
    \frac{\delta}{\left((1 + \delta)^{\beta} + 1 + \frac{2k}{\alpha \mu_{1:2}^{\beta}} \right)^{\frac{1}{2}}}
\end{alignat}
\end{definition}


\begin{proposition}[2D Relative Error Formula]
\label{proposition_2D_relative_error_formula}
For the 2D relative error from definition \ref{definition_relative_error_2d_case},
we have:
\begin{eqnarray*}
\epsilon(\mu_{1:2}, \delta)
    & = & \delta \standardizedCDFOfDifference(\zscoreOfZero(\mu_{1:2}, \delta))
\end{eqnarray*}

Where $\tilde{d}$ is the standardized (mean zero and unit variance) version
of $\hat{d}$ and $\standardizedCDFOfDifference$ is its cumulative distribution function.
\end{proposition}

\begin{proof}
The expected relative error is:
\begin{eqnarray*}
\relError(\mu_{1:2}, \delta)
    & = & \relError(\mu_{1:2}, \mu_{2:2})
    \quad \quad (\mathrm{by\ definition\ }\ref{definition_relative_error_2d_case}) \\
    & = & \delta_{2:2} \mathbb{P}_{2:2}(\mu_{1:2}, \mu_{2:2})
    \quad \quad (\mathrm{by\ proposition\ }\ref{proposition_formula_for_expected_relative_error})
\end{eqnarray*}
But:
\begin{eqnarray*}
\mathbb{P}_{2:2}(\mu_{1:2}, \mu_{2:2}) =  \mathbb{P}(\hat{\mu}_{2:2} < \hat{\mu}_{1:2}) =  \mathbb{P}(\hat{d} < 0)
\end{eqnarray*}
$\hat{d}$ is standardized (mean zero and variance one) by the linear transformation:
\begin{eqnarray*}
l(x) & := & \frac{x - \mathbb{E}(\hat{d})}{\sqrt{\Var{(\hat{d})}}}
\end{eqnarray*}
So
\begin{eqnarray*}
\mathbb{P}(\hat{d} \leq 0) = \mathbb{P}(l(\hat{d}) \leq l(0))
    = \standardizedCDFOfDifference(l(0)) = \standardizedCDFOfDifference(\zscoreOfZero)
\end{eqnarray*}

\end{proof}

We recap the quantities used henceforth:
\begin{itemize}
\item $\delta = \frac{\mu_{2:2} - \mu_{1:2}}{\mu_{1:2}}$ 
\item $\zscoreOfZero$ is the z-score of zero for R.V $\hat{d}$
\item $\standardizedCDFOfDifference$ is the c.d.f of R.V $\tilde{d}$.
\end{itemize}
These quantities make $\relError$ tractable.

%% file: 2d_rel_err_worst_case_bound_abs.tex
\subsection{Upper Bounding $\relError$ for Two Means: Abstract Bound}
\label{subsection_upper_bounding_relError_two_means_abstract_bound}

In this section we derive an abstract upper bound on $\relError(\mu_{1:2}, \delta)$.
We do this in three stages:

\begin{itemize}
\item First we upper bound $\zscoreOfZero$ for a given $\delta$ over all $\mu_{1:2}$
\item Second we upper bound $\relError$ for a given $\delta$ over all $\mu_{1:2}$
\item Third we upper bound $\relError$ over $\mu_{1:2}$ and $\delta$
\end{itemize}


\subsubsection{Upper Bound on $\zscoreOfZero$ for a given $\delta$ over all $\mu_{1:2}$}

Define:

\begin{eqnarray*}
\zscoreOfZero(\delta)
    & := &
    \sup_{\mu_{1:2} \in [\gamma, \infty)}
    - \zscoreOfZeroFirstTerm
    \zscoreOfZeroSecondTerm \\
\end{eqnarray*}

We will upper bound this function in the regime $\beta \in [0,2]$.
Consider:
\begin{alignat*}{1}
\zscoreOfZero(\delta)
    & =
    - \left(
    \inf_{\mu_{1:2} \in [\gamma, \infty)}
    \zscoreOfZeroFirstTerm
    \zscoreOfZeroSecondTerm
    \right) \\
    & \leq
    -
    \left(
    \inf_{\mu_{1:2} \in [\gamma, \infty)}
    \zscoreOfZeroFirstTerm
    \right) \\
    & \quad \quad
    \left(
    \inf_{\mu_{1:2} \in [\gamma, \infty)}
    \zscoreOfZeroSecondTerm
    \right) \\
    & =
    -
    \left(
    \zscoreOfZeroFirstTermInf
    \right) 
    \left(
    \zscoreOfZeroSecondTermInf
    \right)
\end{alignat*}


Where we have defined the constant:
\begin{eqnarray*}
C_1 & = &
    \left\{
    \begin{array}{ll}
    1  + \frac{2 k}{\alpha \gamma^\beta} & k > 0 \\
    1    & k \leq 0
    \end{array}
    \right.
\end{eqnarray*}
and in the last step we have used:
\begin{alignat*}{1}
\beta \leq 2
    \Rightarrow
    \gamma^{1 - \frac{\beta}{2}} \leq \mu_{1:2}^{1 - \frac{\beta}{2}}
    \quad \quad \forall \mu_{1:2} \in [\gamma, \infty)
\end{alignat*}

Next, we upper bound the denominator term:
\begin{alignat*}{1}
\denominatorFunction(\delta) := \left((1 + \delta)^\beta + C_1 \right)^{\frac{1}{2}}
\end{alignat*}

Define

\begin{eqnarray*}
\delta_c := \max(1, C_1^{\frac{1}{\beta}} - 1)
\end{eqnarray*}


When $\forall \delta \geq \delta_c$ and $\beta > 0$, use $C_1 \geq (1 + \delta)^{\beta}$
and $\delta \geq 1$ to get:

\begin{eqnarray}
\label{equation:denominatorFunctionUBRightDomain}
\denominatorFunction(\delta) \leq 2^{\frac{1 + \beta}{2}} \delta^{\frac{\beta}{2}}
\end{eqnarray}

This also bounds $\denominatorFunction$ for $\delta \leq \delta_c$, 
since $\denominatorFunction$ is strictly increasing when $\beta > 0$.


We get when $\beta > 0$:
\begin{eqnarray*}
\frac{1}{\denominatorFunction(\delta)} & \geq &
    \left\{
    \begin{array}{ll}
    \frac{1}{2^{\frac{1 + \beta}{2}} \delta_c^{\frac{\beta}{2}}} & 0 \leq \delta \leq \delta_c \\
    \frac{1}{2^{\frac{1 + \beta}{2}} \delta^{\frac{\beta}{2}}} & \delta \geq \delta_c
    \end{array}
    \right.
\end{eqnarray*}


This gives the upper bound on $\zscoreOfZero(\delta)$ for $\beta \in (0,2]$:

\begin{eqnarray}
\label{equation_zscoreOfZeroUBTwo_beta_greater_zero}
\zscoreOfZero(\delta)
    & \leq &
    \left\{
    \begin{array}{ll}

    - \left( \zscoreOfZeroFirstTermInf \right)
    \left( \frac{\delta}{2^{\frac{1 + \beta}{2}} \delta_c^{\frac{\beta}{2}}} \right)
    & 0 \leq \delta \leq \delta_c \\

    - \left( \zscoreOfZeroFirstTermInf \right)
    \left( \frac{\delta^{1 - \frac{\beta}{2}}}{2^{\frac{1 + \beta}{2}}} \right)
    & \delta \geq \delta_c
    \end{array}
    \right.
\end{eqnarray}

When $\beta = 0$:

\begin{eqnarray*}
\zscoreOfZero(\delta)
    & \leq &
    - \left( \zscoreOfZeroFirstTermInf \right)
    \left( \frac{\delta}{\left( 1 + C_1 \right)^{\frac{1}{2}}} \right) \forall \delta \geq 0
\end{eqnarray*}

In both cases the  upper bound may be interpreted as a piecewise function,
initially linear and then a polynomial.
In the $\beta = 0$ case, the linear part occupies the whole of $\mathbb{R}^{\geq 0}$ since $\delta_c = \infty$
We now unify these two cases.

\begin{definition}[$\zscoreOfZeroUBTwo$]
\label{definition_zscoreOfZeroUBTwo}
\begin{eqnarray*}
C_2
    & := & 
    \left\{
    \begin{array}{ll}
    \frac{\gamma^{1 - \frac{\beta}{2}}}{\alpha^{\frac{1}{2}}} \frac{1}{\left(1 + C_1\right)^{\frac{1}{2}}} &
    \beta = 0 \\
    \frac{\gamma^{1 - \frac{\beta}{2}}}{\alpha^{\frac{1}{2}}}
        \frac{1}{2^{\frac{1 + \beta}{2}} \delta_c^{\frac{\beta}{2}}} &
    \beta > 0 \\
    \end{array}
    \right.
\end{eqnarray*}

\begin{eqnarray*}
C_3
& := &
    \left\{
    \begin{array}{ll}
    \frac{\gamma^{1 - \frac{\beta}{2}}}{\alpha^{\frac{1}{2}}}
        \frac{1}{2^{\frac{1 + \beta}{2}}} &
    \beta > 0,\ \ \delta \geq \delta_c \\
    \end{array}
    \right.
\end{eqnarray*}

\begin{eqnarray*}
\zscoreOfZeroUBTwo(\delta)
    & := &
    \left\{
    \begin{array}{ll}

    - C_2 n^{\frac{1}{2}} \delta
    & 0 \leq \delta \leq \delta_c \\

    - C_3 n^{\frac{1}{2}} \delta^{1 - \frac{\beta}{2}}
    & \delta \geq \delta_c
    \end{array}
    \right.
\end{eqnarray*}
\end{definition}

We now have:
\begin{alignat*}{1}
\zscoreOfZero(\delta) \leq \zscoreOfZeroUBTwo(\delta) & \quad \forall \delta \geq 0, \forall \beta \in [0,2]
\end{alignat*}

For a given value of $\beta$, $\zscoreOfZeroUBTwo$ is continuous on the closed subsets
$0 \leq \delta\leq \delta_c$ and $\delta_c \leq \delta < \infty$
(easier to see this by considering the $\beta > 0$ case from equation 
\ref{equation_zscoreOfZeroUBTwo_beta_greater_zero} separately).
Hence by pasting lemma, $\zscoreOfZeroUBTwo$ is continuous in $\delta$.

Next we prove the monotonicity of $\zscoreOfZeroUBTwo$.
The $\beta = 0$ case is trivial.
Consider the case $\beta > 0$.
It is easy to see $\zscoreOfZeroUBTwo$ is strictly decreasing on $0 \leq \delta \leq \delta_c$ (linear part).
The polynomial $C \delta^{1 - \frac{\beta}{2}}$ is strictly decreasing $\forall \delta > 0$, when $\beta < 2$.
Hence $\zscoreOfZeroUBTwo$ is strictly decreasing on $\delta_c \leq \delta$.
%
%
We collect all the above in a proposition.

\begin{proposition}[Simple Upper Bound on $\zscoreOfZero(\delta)$]
\label{proposition_zscoreOfZeroUBTwo}
The function $\zscoreOfZeroUBTwo(\delta)$  (definition \ref{definition_zscoreOfZeroUBTwo})
is continuous, strictly decreasing and invertible
on $[0, -\infty)$. Furthermore:

\begin{eqnarray*}
\zscoreOfZero(\delta) & \leq & \zscoreOfZeroUBTwo(\delta) \ \ 
\forall \delta \geq 0
\end{eqnarray*}
\end{proposition}

\subsubsection{Upper Bound on $\relError$ for a given $\delta$ over all $\mu_{1:2}$}

Define:
\begin{eqnarray*}
    \epsilon(\delta)
    & = &
    \sup_{
    \left.
    \begin{array}{l}
    \mu_{1:2} \in [ \gamma, \infty) \\
    \end{array}
    \right.
    }
    \epsilon(\mu_{1:2}, \delta)
\end{eqnarray*}


Then we have:

\begin{proposition}[Upper Bound on $\relError(\delta)$]
\label{proposition_upper_bound_on_relError_delta}
If $\standardizedCDFOfDifference$ is continuous
and $\beta \in [0,2]$, then:
\begin{alignat*}{1}
\relError(\delta) \leq \delta \standardizedCDFOfDifference(\zscoreOfZeroUBTwo(\delta)) \quad \forall \delta > 0
\end{alignat*}
\end{proposition}

\begin{proof}
Now by proposition \ref{proposition_formula_for_expected_relative_error} (formula for relative error)

\begin{eqnarray*}
\epsilon(\delta) & = & 
    \sup_{
    \left.
    \begin{array}{l}
    \mu_{1:2} \in [ \gamma, \infty) \\
    \end{array}
    \right.
    }
    \delta \standardizedCDFOfDifference(\zscoreOfZero(\mu_{1:2}, \delta)) \\
    & = & \delta \standardizedCDFOfDifference \left(\sup_{\mu_{1:2} \in [\gamma, \infty)} \zscoreOfZero(\mu_{1:2}, \delta)\right) \\
    & = & \delta \standardizedCDFOfDifference \left(\zscoreOfZero(\delta)\right) \\
    & \leq & \delta \standardizedCDFOfDifference \left(\zscoreOfZeroUBTwo(\delta)\right)
\end{eqnarray*}
We have used proposition \ref{proposition_zscoreOfZeroUBTwo} and non-decreasing property of a cdf for the last step.
\end{proof}


\subsubsection{Maximizing Over $\delta$}

We seek an upper bound on $\epsilon(\delta)$ of the following form (essentially a tail bound).
For some $T, \deltaTh > 0$:

\begin{eqnarray*}
\epsilon(\delta) < T \quad \quad \forall \delta > \deltaTh
\end{eqnarray*}

\begin{definition}[$\relErrorUBThree$]
\label{definition_relErrorUBThree}
Define

\begin{eqnarray*}
\relErrorUBThree(\delta) := G(\zscoreOfZeroUBTwo(\delta)) \quad \quad \forall \delta > 0
\end{eqnarray*}

Where $G$ is a function satisfying:

\begin{eqnarray*}
G(x) & > & \zscoreOfZeroUBTwoInverse(x) \standardizedCDFOfDifference(x) \quad \quad \forall x \in (0, -\infty)
\end{eqnarray*}
\end{definition}

Then we have the following bound.

\begin{lemma}[Abstract Upper Bound on $\relError(\delta)$ for Two Means]
\label{lemma_abstract_upper_bound_on_relError_two_means}
Consider a power-variance distribution family $\mathcal{F}$,
with $\beta \in [0,2]$.
Let the corresponding $\standardizedCDFOfDifference$ be continuous.
Then, if there exists $G(x)$ such that:

\begin{eqnarray}
\label{equation_definition_of_G}
G(x) > \zscoreOfZeroUBTwoInverse(x) \standardizedCDFOfDifference(x) \quad \quad \forall x \in (0, _\infty)
\end{eqnarray}

Then $\relErrorUBThree$ (definition \ref{definition_relErrorUBThree}) satisfies:

\begin{eqnarray*}
\relError(\delta) < \relErrorUBThree(\delta) \quad \quad \forall \delta > 0
\end{eqnarray*}

\end{lemma}

\begin{proof}
We denote the upper bound from proposition \ref{proposition_upper_bound_on_relError_delta}
by $\relErrorUBTwo$.

\begin{eqnarray*}
\relErrorUBTwo(\delta)
     & := & 
     \delta \standardizedCDFOfDifference(\zscoreOfZeroUBTwo(\delta))
     \quad \quad \forall \delta > 0
\end{eqnarray*}

By proposition \ref{proposition_zscoreOfZeroUBTwo} (properties of $\zscoreOfZeroUBTwo$),
the inverse of $\zscoreOfZeroUBTwo$ exists and we can rewrite $\relErrorUBTwo$ as a function 
of $\zscoreOfZeroUBTwo$. We have $\forall \delta > 0$:

\begin{eqnarray}
\label{equation_relErrorUBTwo_as_function_of_zscoreOfZeroUBTwo}
\relErrorUBTwo(\delta) = 
\relErrorUBTwo(\zscoreOfZeroUBTwo(\delta)) = 
   \zscoreOfZeroUBTwoInverse(\zscoreOfZeroUBTwo(\delta))\standardizedCDFOfDifference(\zscoreOfZeroUBTwo(\delta))
\end{eqnarray}

Where we have abused notation using $\relErrorUBTwo(\zscoreOfZeroUBTwo(\delta))$ to denote 
the dependence of $\relErrorUBTwo$ on $\delta$ solely via $\zscoreOfZeroUBTwo$.

Then we have a new upper bound

\begin{alignat*}{2}
\relErrorUBThree(\zscoreOfZeroUBTwo(\delta))
    & = G(\zscoreOfZeroUBTwo(\delta))
    & & \forall \delta > 0 \\
    & > \zscoreOfZeroUBTwoInverse(\zscoreOfZeroUBTwo(\delta)) \standardizedCDFOfDifference(\zscoreOfZeroUBTwo(\delta)) \\
    & = \relErrorUBTwo(\delta)
    &  & (\mathrm{by\ equation\ }\ref{equation_relErrorUBTwo_as_function_of_zscoreOfZeroUBTwo} ) \\
    & \geq \relError(\delta)
\end{alignat*}

\end{proof}

This is the chief result of this subsection.
This is one of our core results, because it is applicable with {\em great generality}
and gives us an {\em easy way} to bound
errors for any distribution family $\mathcal{F}$, provided we can find a
suitable $G$.


%% file: 2d_rel_err_worst_case_bound_gaussian.tex
\subsection{Upper Bounding $\relError$ for Two Means: Gaussian Case}
\label{subsection_upper_bounding_relError_two_means_gaussian_case}

In this subsection we assume that the distribution family ($\mathcal{E}$) is Gaussian ($\normalDistribution$).
As usual the family is further specified by the parameters $\familyParametersAugumented$.
We start by finding an upper bound for $\standardizedCDFOfDifference = \standardNormalCDF$.
Here $\standardNormalCDF$ is the cdf of a standard normal.
We use an ubiquitous tail bound (proof omitted).

\begin{proposition}[A Gaussian Tail Bound]
\label{proposition_gaussian_tail_bound}
This is reproduced from \cite{fellerProbability}.
Let $a > 0$, then:
\begin{eqnarray*}
\Phi(-a) < \frac{1}{a} \phi(-a)
\end{eqnarray*}
\end{proposition}

Then we set

\begin{eqnarray*}
G_{\normalDistribution}(x)
    =
    \zscoreOfZeroUBTwoInverse(x) \frac{1}{-x} \frac{\exp\left(-\frac{x^2}{2}\right)}{\sqrt{2 \pi}}
    \quad \quad \forall x \in (0, -\infty)
\end{eqnarray*}

Observe \footnote{
When $\beta \in [0, 2)$, $\zscoreOfZeroUBTwo$ is unbounded and so
$\zscoreOfZeroUBTwo([0, \infty)) = [0, -\infty)$.
Further $\zscoreOfZeroUBTwo(0) = 0$ always.}
that $\zscoreOfZeroUBTwoInverse(x) > 0$ $\forall x \in (0, -\infty)$ .
So by using the tail bound in proposition \ref{proposition_gaussian_tail_bound}:

\begin{eqnarray*}
G(x) 
    & = & \zscoreOfZeroUBTwoInverse(x) \frac{1}{-x} \frac{\exp\left(-\frac{x^2}{2}\right)}{\sqrt{2 \pi}} \\
    & > & \zscoreOfZeroUBTwoInverse(x) \standardizedCDFOfDifference(x)
    \quad \quad \forall x \in (0, -\infty)
\end{eqnarray*}

Thus satisfying equation \ref{equation_definition_of_G} and so $\forall \delta > 0$:

\begin{eqnarray*}
\relErrorUBThreeNormal(\delta) 
    =
    \zscoreOfZeroUBTwoInverse(\zscoreOfZeroUBTwo(\delta)) \frac{1}{-\zscoreOfZeroUBTwo(\delta)} 
    \frac{\exp \left(- \frac{\left(\zscoreOfZeroUBTwo(\delta)\right)^2}{2} \right)}{\sqrt{2 \pi}} \\
    =
    -\frac{\delta}{\zscoreOfZeroUBTwo(\delta)} 
    \frac{\exp \left(- \frac{\left(\zscoreOfZeroUBTwo(\delta)\right)^2}{2} \right)}{\sqrt{2 \pi}}
\end{eqnarray*}

Substituting for $\zscoreOfZeroUBTwo$:

\begin{eqnarray*}
\relErrorUBThreeNormal(\delta)
    =
    \left\{
    \begin{array}{ll}
    \frac{1}{\sqrt{2 \pi}} \frac{1}{C_2 n^{\frac{1}{2}}} \exp\left[-\frac{1}{2} C_2^2 n \delta^2 \right] &
        0 < \delta \leq \delta_c \\
    \frac{1}{\sqrt{2 \pi}} \frac{\delta^{\frac{\beta}{2}}}{C_3 n^{\frac{1}{2}}} \exp\left[-\frac{1}{2} C_3^2 n \delta^{2 - \beta} \right] &
        \delta \geq \delta_c
    \end{array}
    \right.
\end{eqnarray*}

Given a function $f$, for the log (or any strictly increasing function) we have:

\begin{eqnarray*}
\sgn\left(\partialDerivative{f}{x}\right) = \sgn\left(\partialDerivative{\log(f)}{x}\right)
\end{eqnarray*}

Consider the function:

\begin{eqnarray*}
f(x) = \frac{1}{\sqrt{2 \pi}}
     \frac{x^{\frac{\beta}{2}}}{C_3 n^{\frac{1}{2}}} \exp\left[-\frac{1}{2} C_3^2 n x^{2 - \beta} \right]
     \quad \quad \forall x > 0
\end{eqnarray*}
Then

\begin{eqnarray*}
\partialDerivative{\log(f)}{x}
    =
    \frac{\beta}{2} \frac{1}{x} - \frac{C_3^2}{2} n (2 - \beta) x^{(1 - \beta)}
\end{eqnarray*}

Hence $\partialDerivative{f}{x} < 0$ iff:

\begin{eqnarray*}
\frac{\beta}{2} \frac{1}{x} - \frac{C_3^2}{2} n (2 - \beta) x^{(1 - \beta)} & < & 0 \\
x & > & \left( \frac{\beta}{C_3^2 n (2 - \beta)} \right)^{\frac{1}{2 - \beta}}
\end{eqnarray*}
Where on the last line we used that $g(y) = y^{\frac{1}{2 - \beta}}$ is strictly increasing
\mbox{$\forall \beta \in [0,2)$} and $y \geq 0$ 
We want $f$ to be increasing for all $x \geq \delta_c$. So set $x = \delta_c$

\begin{eqnarray*}
\delta_c & > & \left( \frac{\beta}{C_3^2 n (2 - \beta)} \right)^{\frac{1}{2 - \beta}} \\
\Rightarrow n & > & \frac{\beta}{C_3^2 (2 - \beta) \delta_c^{2 - \beta}}
\end{eqnarray*}

Now 
$\relErrorUBThreeNormal(\delta)$ is strictly decreasing when the above equation is satisfied.
Hereafter we assume this.
$\range(f)$ denotes the range of function $f$.
Then with $\delta > 0$:
\begin{eqnarray*}
\range(\relErrorUBThreeNormal)
    & = & (\relErrorUBThreeNormal(0), \lim_{\delta \rightarrow \infty} \relErrorUBThreeNormal(\delta)) \\
    & = & (\frac{1}{\sqrt{2 \pi} C_2 n^{\frac{1}{2}}}, 0)
    \quad \quad (\mathrm{interval\ is\ reversed})
\end{eqnarray*}

Note: In the $\beta > 0$ case:

\begin{eqnarray*}
\lim_{\delta \rightarrow \infty} \relErrorUBThreeNormal(\delta))
    & = & \lim_{\delta \rightarrow \infty} D_1 \delta^{\frac{\beta}{2}}
        \exp\left[-D_2 \delta^{2 - \beta}\right] \\
    &   & \quad \quad (D_1, D_2 > 0 \mathrm{\ are\ constants})\\
    & = & \lim_{x \rightarrow \infty} D_1 x^{\frac{\beta}{4 - 2 \beta}} \exp\left[ - D_2 x \right]
    \quad \quad (x = \delta^{2 - \beta}) \\
    & = & 0 \\
\end{eqnarray*}

For the last step we have used finite numberr of applications of L'Hopital.
Hence $T \in \range(\relErrorUBThreeNormal)$ is equivalent to:
\begin{eqnarray*}
T < \frac{1}{\sqrt{2 \pi} C_2 n^{\frac{1}{2}}}
\end{eqnarray*}

Given a target $T > 0$ to bound relative error, we search for a corresponding $\deltaThNormal$.
If $T \geq \frac{1}{\sqrt{2 \pi} C_2 n^{\frac{1}{2}}}$ ($T$ larger than max-range),
then $\relErrorUBThreeNormal(\delta) < T \quad \forall \delta > 0$.
Hence $\deltaThNormal = 0$ in this case.

Next assume $T < \frac{1}{\sqrt{2 \pi} C_2 n^{\frac{1}{2}}}$.
Set:

\begin{eqnarray*}
f_1(\delta) = \frac{1}{\sqrt{2 \pi} C_2 n^{\frac{1}{2}}}
    \exp\left[-\frac{1}{2} C_2^2 n \delta^2 \right]
    \quad \quad \forall \delta > 0
\end{eqnarray*}

Now $T \in \range(f_1) = \range(\relErrorUBThreeNormal)$ and

\begin{eqnarray*}
f_1^{-1}(T) & = & \left(-\frac{2}{C_2^2 n} \log(\sqrt{2 \pi n} C_2 T)\right)^{\frac{1}{2}}
\end{eqnarray*}

When $f_1^{-1}(T) \leq \delta_c$
i.e. when $n  \geq \frac{1}{2 \pi C_2^2 T^2}\exp\left[-C_2^2  n \delta_c^2\right]$,
we have
\begin{eqnarray*}
\relErrorUBThreeNormalInverse(T) & = & f_1^{-1}(T) = \left(-\frac{2}{C_2^2 n} \log(\sqrt{2 \pi n} C_2 T)\right)^{\frac{1}{2}}
\end{eqnarray*}

Hence $\deltaThNormal = \left(-\frac{2}{C_2^2 n} \log(\sqrt{2 \pi n} C_2 T)\right)^{\frac{1}{2}}$ in this case.

By lemma \ref{lemma_abstract_upper_bound_on_relError_two_means}
(abstract upper bound on $\relError(\delta)$), we get

\begin{eqnarray*}
\relError(\delta) < T \quad \quad \forall \delta > \deltaThNormal
\end{eqnarray*}

We collect all these  in a lemma.

\begin{lemma}[Upper Bound on $\relError$ for Two Means: Gaussian Case]
\label{lemma_upper_bound_on_relError_two_means_gaussian_case}
When $\mathcal{E}$ is the Gaussian family and given values for $\familyParametersAugumented$.
Define:

\begin{eqnarray}
\label{equation_definition_of_relErrorUBThreeNormal}
\relErrorUBThreeNormal(\delta)
    :=
    \left\{
    \begin{array}{ll}
    \frac{1}{\sqrt{2 \pi}} \frac{1}{C_2 n^{\frac{1}{2}}} \exp\left[-\frac{1}{2} C_2^2 n \delta^2 \right] &
        0 < \delta \leq \delta_c \\
    \frac{1}{\sqrt{2 \pi}} \frac{\delta^{\frac{\beta}{2}}}{C_3 n^{\frac{1}{2}}} \exp\left[-\frac{1}{2} C_3^2 n \delta^{2 - \beta} \right] &
        \delta \geq \delta_c
    \end{array}
    \right.
\end{eqnarray}

This is an upper bound:

\begin{eqnarray*}
\relErrorUBThreeNormal(\delta) > \relErrorNormal(\delta) \quad \quad \forall \delta > 0
\end{eqnarray*}

Additionally, given a $T > 0$ if $n$ satisfies:
\begin{eqnarray}
\label{equation_relErrorUBThreeNormal_monotonicity_requirement}
n & > & \frac{\beta}{C_3^2 (2 - \beta) \delta_c^{2 - \beta}} \\
\nonumber
  & & \quad \quad (\relErrorUBThreeNormal\ \mathrm{monotonicity\ requirement}) \\
\label{equation_relErrorUBThreeNormal_piecewise_inverse_requirement}
n & \geq & \frac{1}{2 \pi C_2^2 T^2}\exp\left[-C_2^2  n \delta_c^2\right] \\
\nonumber
  & & \quad (\relErrorUBThreeNormal\ \mathrm{piecewise\ inverse\ requirement})
\end{eqnarray}

Then $\relErrorUBThreeNormal$ is strictly decreasing.
And we can define a $\deltaThNormal$

\begin{eqnarray}
\label{equation_definition_of_deltaThNormal}
\deltaTh_{\normalDistribution} :=
    \left\{
    \begin{array}{ll}
    \left(-\frac{2}{C_2^2 n} \log(\sqrt{2 \pi n} C_2 T)\right)^{\frac{1}{2}}
    & T < \frac{1}{\sqrt{2 \pi} C_2 n^{\frac{1}{2}}} \\
    0 & \mathrm{else}
    \end{array}
    \right.
\end{eqnarray}

Such that:

\begin{eqnarray*}
\relErrorNormal(\delta) < T \quad \quad \forall \delta > \deltaTh_{\normalDistribution}
\end{eqnarray*}

\end{lemma}

Typically we will use this result as follows. Given a target $T$.
We will first find the minimum $n$  that satisfies equations
\ref{equation_relErrorUBThreeNormal_monotonicity_requirement},
\ref{equation_relErrorUBThreeNormal_piecewise_inverse_requirement}.
Then we can choose $n \geq n_{\min}$ to make $\deltaTh_{\normalDistribution}$ as small as we want.

%% file: 2d_rel_err_worst_case_bound_general.tex
\subsection{Upper Bounding $\relError$ for Two Means: General Case}
\label{subsection_upper_bounding_relError_two_means_general_case}

In this subsection we work with an arbitrary distribution family ($\mathcal{E}$).
As usual the family is specified by the parameters $\familyParametersAugumented$.
And we want to find a worst case upper bound on $\relError$.
We split this into two stages.
In subsection 
\ref{subsubsection_familywise_tail_bound_on_standardizedCDFOfDifference}
we find a familywise
tail error bound between $\standardizedCDFOfDifference$ and $\Phi$.
In subsection \ref{subsubsection_bounding_relErrorGen} we use this
to bound $\relErrorGen$.

\input{berry_esseen.tex}

\subsubsection{Bounding $\relErrorGen$}
\label{subsubsection_bounding_relErrorGen}

Herafter we assume $\mathcal{F}$ is such that the kurtosis are uniformly bounded:

\begin{eqnarray*}
\kappa(\mu) \leq \kurtosisUB  \quad \quad \forall \mu \in [\gamma, \infty)
\end{eqnarray*}

Then by the above proposition $\forall x < 0$:

\begin{eqnarray*}
\standardizedCDFOfDifference(x)
    & < & \Phi(x) + \frac{C_6}{|x|^3 n^{\frac{1}{2}}} \\
    & \leq &
    \frac{1}{-x} \frac{\exp\left(- \frac{x^2}{2}\right)}{\sqrt{2 \pi}} + \frac{C_6}{|x|^3 n^{\frac{1}{2}}} \\
    & & \quad (\mathrm{by\ proposition\ }\ref{proposition_gaussian_tail_bound}\ \mathrm{Gaussian\ tail\ bound})
\end{eqnarray*}

Now $\zscoreOfZeroUBTwoInverse(x) > 0 \ \ \forall x \in (0, -\infty)$. So:

\begin{alignat*}{1}
\zscoreOfZeroUBTwoInverse(x) & \standardizedCDFOfDifference(x) <  \\
    & \zscoreOfZeroUBTwoInverse(x)
    \left[
        \frac{1}{\sqrt{2 \pi}} \frac{\exp\left(- \frac{x^2}{2}\right)}{-x}
        + \frac{C_6}{|x|^3 n^{\frac{1}{2}}}
    \right]
\end{alignat*}

And we have the required $G$ satisfying equation \ref{equation_definition_of_G}:

\begin{eqnarray*}
\GGen(x)
    & = &
    \zscoreOfZeroUBTwoInverse(x)
    \left[
        \frac{1}{\sqrt{2 \pi}} \frac{\exp\left(- \frac{x^2}{2}\right)}{-x}
        + \frac{C_6}{|x|^3 n^{\frac{1}{2}}}
    \right] \\
    & & \quad \quad \forall x < 0
\end{eqnarray*}

We split this into two terms:

\begin{eqnarray*}
\GGen(x) = \GNormal(x) + \GBE(x)
\end{eqnarray*}

Where:

\begin{eqnarray*}
\GBE(x) = 
    \zscoreOfZeroUBTwoInverse(x) \frac{C_6}{|x|^3 n^{\frac{1}{2}}}
    \quad \quad \forall x < 0
\end{eqnarray*}

Correspondingly define the relative error components:

\begin{eqnarray*}
\relErrorUBThreeBE(\delta)
    & := & \GBE(\zscoreOfZeroUBTwo(\delta)) \\
    & = & \zscoreOfZeroUBTwoInverse(\zscoreOfZeroUBTwo(\delta))
         \frac{C_6}{|\zscoreOfZeroUBTwo(\delta)|^3 n^{\frac{1}{2}}}
         \quad \quad \forall \delta > 0 \\
    & = & \delta
         \frac{C_6}{|\zscoreOfZeroUBTwo(\delta)|^3 n^{\frac{1}{2}}}
\end{eqnarray*}

So:

\begin{eqnarray*}
\relErrorUBThreeBE(\delta)
    & = &
    \left\{
    \begin{array}{ll}
    \frac{C_6}{C_2^3 n^2} \delta^{-2} & 0 < \delta \leq \delta_c \\
    \frac{C_6}{C_3^3 n^2} \delta^{\frac{3 \beta - 4}{2}} & \delta \geq \delta_c
    \end{array}
    \right.
\end{eqnarray*}

Easy to see that $\relErrorUBThreeBE$ is strictly decreasing when $\beta < \frac{4}{3} = 1.33\ldots$

Now we define:

\begin{eqnarray*}
\relErrorUBThreeGen(\delta)
    & := & \GGen(\zscoreOfZeroUBTwo(\delta)) \\
    & = & \GNormal(\zscoreOfZeroUBTwo(\delta)) + \GBE(\zscoreOfZeroUBTwo(\delta)) \\
    & = & \relErrorUBThreeNormal(\delta) + \relErrorUBThreeBE(\delta)
\end{eqnarray*}

$\relErrorUBThreeGen(\delta)$ is strictly decreasing when equation
\ref{equation_relErrorUBThreeNormal_monotonicity_requirement} is satisfied and $\beta > \frac{4}{3}$.
Henceforth we will assume these conditions are satisfied.

We now seek a formula to invert $\relErrorUBThreeBE$. Let

\begin{eqnarray*}
f_1(\delta) = \frac{C_6}{C_2^3 n^2} \delta^{-2} \quad \quad \forall \delta > 0
\end{eqnarray*}

Then $\range(f_1) = (\infty, 0) = \range(\relErrorUBThreeBE)$.
So, given $T > 0$ we have the inverse:

\begin{eqnarray*}
f_1^{-1}(T) = \frac{C_6^{\frac{1}{2}}}{C_2^{\frac{3}{2}} n T^{\frac{1}{2}}}
\end{eqnarray*}

We have:

\begin{eqnarray*}
\relErrorUBThreeBEInverse(T) = f_1^{-1}(T) \quad \quad \mathrm{when\ } f_1^{-1}(T) \leq  \delta_c
\end{eqnarray*}

The condition $f_1^{-1}(T) \leq  \delta_c$ is equivalent to:

\begin{eqnarray*}
n \geq \frac{C_6^{\frac{1}{2}}}{C_2^{\frac{3}{2}} T^{\frac{1}{2}} \delta_c}
\end{eqnarray*}

Now we derive a formula for the inverse of $\relErrorUBThreeGen$.
Directly inverting $\relErrorUBThreeGen$ (in the $\delta \leq \delta_c$ region) will give
a transcendental equation. 
We get around this with a simple approach.
We invert both the component functions ($\relErrorUBThreeNormal$ and $\relErrorUBThreeBE$)
for a target of $\frac{T}{2}$, and take the max of the resulting $\delta$.
Since these are strictly decreasing functions, their sum
will be $< T$ at that $\delta$. We now work out the details.

Given a desired target $T > 0$,
let us assume conditions are satisfied ensuring monotonicity
and first piecewise invertibility\footnote{By first piecewise invertibility we mean
that the function inverse is the inverse of the first piecewise segment}
of $\relErrorUBThreeNormal$ and $\relErrorUBThreeBE$.
Note that the conditions for first piecewise invertibility of 
 $\relErrorUBThreeNormal$ and $\relErrorUBThreeBE$ are to be applied for a target
of $\frac{T}{2}$ and not $T$.
Then consider the case:

\begin{eqnarray*}
\frac{T}{2}
    \in \range(\relErrorUBThreeNormal) \cap \range(\relErrorUBThreeBE)
    = \left(\frac{1}{\sqrt{2 \pi} C_2 n^{\frac{1}{2}}}, 0 \right)
\end{eqnarray*}

Set:

\begin{eqnarray*}
\deltaThGen
    & = &
    \max\left(\relErrorUBThreeNormalInverse(\frac{T}{2}), \relErrorUBThreeBEInverse(\frac{T}{2})\right) \\
    & = & \max\left(
    \left(-\frac{2}{C_2^2 n} \log(\sqrt{\frac{\pi n}{2}} C_2 T)\right)^{\frac{1}{2}},
    \frac{C_6^{\frac{1}{2}} \sqrt{2} }{C_2^{\frac{3}{2}} n T^{\frac{1}{2}}}
 \right)
\end{eqnarray*}

Next consider the case: $\frac{T}{2} \geq \frac{1}{\sqrt{2 \pi} C_2 n^{\frac{1}{2}}}$
By lemma \ref{lemma_upper_bound_on_relError_two_means_gaussian_case} (Gaussian upper bound on rel error two means):

\begin{eqnarray*}
\relErrorUBThreeNormal(\delta) < \frac{T}{2} \quad \quad \forall \delta > 0
\end{eqnarray*}

So set:

\begin{eqnarray*}
\deltaThGen = \relErrorUBThreeBEInverse(\frac{T}{2})
\end{eqnarray*}

Then in both cases:

\begin{eqnarray*}
\relErrorUBThreeGen(\deltaThGen)
    & = & \relErrorUBThreeNormal(\deltaThGen) + \relErrorUBThreeBE(\deltaThGen) \\
    & \leq & \frac{T}{2} + \frac{T}{2} \\
    & = & T
\end{eqnarray*}

Hence define $\deltaThGen(T) :=$

\begin{eqnarray*}
    \left\{
    \begin{array}{ll}
    \max(
    \left(-\frac{2}{C_2^2 n} \log(\sqrt{\frac{\pi n}{2}} C_2 T)\right)^{\frac{1}{2}}, &
    \frac{C_6^{\frac{1}{2}} \sqrt{2} }{C_2^{\frac{3}{2}} n T^{\frac{1}{2}}}
    ) \\ 
    & T < \sqrt{\frac{2}{\pi}} \frac{1}{C_2 n^{\frac{1}{2}}} \\
    \\
    \frac{C_6^{\frac{1}{2}} \sqrt{2} }{C_2^{\frac{3}{2}} n T^{\frac{1}{2}}} & 
    T \geq \sqrt{\frac{2}{\pi}} \frac{1}{C_2 n^{\frac{1}{2}}}
    \end{array}
    \right.
\end{eqnarray*}

And get:

\begin{eqnarray*}
\relErrorUBThreeGenInverse(T) \leq \deltaThGen
\end{eqnarray*}

Then by lemma \ref{lemma_abstract_upper_bound_on_relError_two_means} (abstract upper bound 
on relative error):

\begin{eqnarray*}
\relErrorGen(\delta) < T \quad \quad \forall \delta > \deltaThGen
\end{eqnarray*}

We collect all the above in a lemma.

\begin{lemma}[Upper Bound on $\relError$ for Two Means: General Case]
\label{lemma_upper_bound_on_relError_two_means_general_case}
Let $\mathcal{E}$ be some distribution family having parameters $\familyParametersAugumented$.
For all distributions $\mathcal{E}_{i:m}$ in this family, let the cdf be continuous
and the kurtosis be uniformly bounded:
\begin{eqnarray*}
\kappa(\mu_{i:m}) \leq \kurtosisUB  \quad \quad \forall \mu_{i:m} \in [\gamma, \infty)
\end{eqnarray*}
Let $\relErrorUBThreeBE, \relErrorUBThreeGen$ be defined as:

\begin{eqnarray*}
\relErrorUBThreeBE(\delta)
    & := &
    \left\{
    \begin{array}{ll}
    \frac{C_6}{C_2^3 n^2} \delta^{-2} & 0 < \delta \leq \delta_c \\
    \frac{C_6}{C_3^3 n^2} \delta^{\frac{3 \beta - 4}{2}} & \delta \geq \delta_c
    \end{array}
    \right. \\
\relErrorUBThreeGen(\delta) & := & \relErrorUBThreeNormal(\delta) + \relErrorUBThreeBE(\delta) \quad \quad \forall \delta > 0
\end{eqnarray*}

This is an upper bound:

\begin{eqnarray*}
\relErrorUBThreeGen(\delta) > \relErrorGen(\delta) \quad \quad \forall \delta > 0
\end{eqnarray*}

Given a $T$, let $n, T$ satisfy the following conditions:

\begin{eqnarray*}
T & > & 0 \\
n & > & \frac{\beta}{C_3^2 (2 - \beta) \delta_c^{2 - \beta}} \\
    & & \quad \quad (\mathrm{ensures\ monotonicity\ of\ }\relErrorUBThreeNormal) \\
n & \geq & \frac{2}{\pi C_2^2 T^2} \exp\left[- C_2^2 n \delta_c^2\right] \\
    & & \quad \quad (\mathrm{required\ for\ first\ piecewise\ invertibility\ of\ } \\
    & & \quad \quad \relErrorUBThreeNormal \mathrm{\ at\ } \frac{T}{2}) \\
\beta & < & \frac{4}{3} = 1.333\ldots \\
    & & \quad \quad (\mathrm{ensures\ monotonicity\ of\ }\relErrorUBThreeBE) \\
n & \geq & \frac{C_6^{\frac{1}{2}} 2^{\frac{1}{2}}}{C_2^{\frac{3}{2}}T^{\frac{1}{2}}\delta_c} \\
    & & \quad \quad (\mathrm{required\ for\ first\ piecewise\ invertibility\ of\ } \\
    & & \quad \quad \relErrorUBThreeBE \mathrm{\ at\ } \frac{T}{2})
\end{eqnarray*}

and define $\deltaThGen(T) :=$

\begin{alignat}{1}
\label{equation_deltaThGen_definition}
    \left\{
    \begin{array}{ll}
    \max(
    \left(-\frac{2}{C_2^2 n} \log(\sqrt{\frac{\pi n}{2}} C_2 T)\right)^{\frac{1}{2}}, &
    \frac{C_6^{\frac{1}{2}} \sqrt{2} }{C_2^{\frac{3}{2}} n T^{\frac{1}{2}}}
    ) \\
    & T < \sqrt{\frac{2}{\pi}} \frac{1}{C_2 n^{\frac{1}{2}}}
    \\
    \frac{C_6^{\frac{1}{2}} \sqrt{2} }{C_2^{\frac{3}{2}} n T^{\frac{1}{2}}} & 
    T \geq \sqrt{\frac{2}{\pi}} \frac{1}{C_2 n^{\frac{1}{2}}}
    \end{array}
    \right.
\end{alignat}

Then:
\begin{eqnarray*}
\relErrorGen(\delta) < T \quad \quad \forall \delta > \deltaThGen
\end{eqnarray*}
\end{lemma}

This $\deltaThGen$ formula tells us that the $\delta$ threshold required to ensure 
$\relError(\delta)$ less than a specified tolerance is a relatively rapidly decreasing function of the
threshold $T$ (either behaving like $\sqrt{-C \log(C T)}$ or $\frac{C}{T^{\frac{1}{2}}}$).
For a fixed threshold $T$, the $\deltaThGen$ is a rapidly decreasing function of $n$ as
well (either behaving  like $\sqrt{-C \frac{\log(C n)}{n}}$ or $\frac{C}{n}$).
Combined, these indicate small requirements on $n$ to hit a target relative error.

Our usage of the error bound will be to  fix an acceptable error threshold $T$
and to fix an acceptable threshold on $\deltaThGen = \relErrorUBThreeGenInverse(T)$.
Then we find $n$ as small as possible that guarantees both simultaneously.
The bound on $\deltaThGen$ is required to generalize from the 2-dimensional case
(min of two sample means) to the m-dimensional case (min of m sample means).

%% file: berry_esseen.tex
\subsubsection{Familywise Tail Bound on $\standardizedCDFOfDifference$}
\label{subsubsection_familywise_tail_bound_on_standardizedCDFOfDifference}

By the CLT, the Gaussian cdf $\Phi$ is an attractor for the cdf of $\hat{d}_S$.
Hence we could attempt to generalize the bound in lemma 
\ref{lemma_upper_bound_on_relError_two_means_gaussian_case}
(Gaussian upper bound on $\relError$)
to non-Gaussian distributions by setting $\standardizedCDFOfDifference = \Phi$, when `$n$ is large enough'.
However, there are significant issues with a naive application of CLT in our context.

\begin{description}

\item[Varying Distributions] Firstly, we are applying the CLT over a family of distributions. 
That is, the random variable $\hat{d}$ is the sample mean of one out of a possible family
of distributions and not one fixed distribution.
Further the family is infinite (indexed by real valued parameters $\mu_{1:2}, \mu_{2:2}$).
It is easy to construct families such that the CLT {\bf requires infinite sample size $n$ to 
converge} to a given tolerance for all members of the family.

\item[Central Convergence vs Tail]
The CLT is known to converge fast near the mean.
And primarily that is how the CLT is used (to construct confidence intervals around the mean);
However we are interested in the tail of the cdf as well.
{\bf Convergence may require vastly more samples than usual}.
\end{description}

To deal with the above challenges we develop a familywise version of the
non-uniform Berry-Esseen theorem for the difference of sample averages.
Our starting point is the following Berry-Esseen theorem:

\begin{theorem}[Non-Uniform Berry-Esseen \cite{PaditzBE}]
\label{theorem:non_uniform_berry_esseen}
Let $X_i$ be independent random variables such that:
\begin{eqnarray*}
\mathbb{E} X_i & = & 0 \\
\sum_{i = 1}^{n} \mathbb{E} X_i^2 & = & 1 \\
\mathbb{E}|X_i|^3 & < & \infty \quad \forall i
\end{eqnarray*}

Then set $W_n := \sum_{i = 1}^{n} X_i$ so $\Var(W_n) = 1$.
And let $F_{W_n}$ denote the cdf of $W_n$.
Then:

\begin{eqnarray*}
|F_{W_n}(x) - \Phi(x)| \leq \frac{C_4}{1 + |x|^3} \sum_{i = 1}^n \mathbb{E} |X_i|^3
\end{eqnarray*}

Where $C_4 < 32$.
\end{theorem}

Then we have a simple corollary.
\begin{proposition}[Non-Uniform Berry Esseen for Sample Average]
\label{proposition:non_uniform_berry_esseen_for_sample_average}
Let $Y_i$ be $n$ i.i.d random variables such that:
\begin{eqnarray*}
\mathbb{E} Y_i & = & \mu \\
\Var(Y_i) & = & \sigma^2 \\
\mathbb{E}|Y_i - \mu|^3 & < & \infty
\end{eqnarray*}

Let $A_n := \frac{1}{n} \sum_{i} Y_i$ be the sample average.
Let $W_n := \frac{A_n - \mu}{\sigma / \sqrt{n}}$ be the standardized version of $A_n$.
And let $F_{W_n}$ denote the cdf of $W_n$.
Then:
\begin{eqnarray*}
|F_{W_n}(x) - \Phi(x)| \leq \frac{C_4}{|x|^3} \frac{\mathbb{E}|Y_i - \mu|^3}{n^{\frac{1}{2}} \sigma^3}
\end{eqnarray*}
\end{proposition}

\begin{proof}
With the above setup define:
\begin{eqnarray*}
X_i := \frac{Y_i - \mu}{\sqrt{n} \sigma}
\end{eqnarray*}
So
\begin{eqnarray*}
\mathbb{E} |X_i|^3 = \mathbb{E} |\frac{Y_i - \mu}{\sqrt{n} \sigma}|^3
    = \frac{1}{n^{\frac{3}{2} \sigma^3}} \mathbb{E}|Y_i - \mu|^3
\end{eqnarray*}
Thus $X_i$ are i.i.d and
\begin{eqnarray*}
\mathbb{E} X_i & = & 0 \\
\sum_{i = 1}^{n} \mathbb{E} X_i^2 & = & 1 \\
\mathbb{E}|X_i|^3 & < & \infty \quad \forall i \\
\sum_{i = 1}^{n} \mathbb{E} |X_i|^3 & = & \frac{1}{n^{\frac{1}{2} \sigma^3}} \mathbb{E}|Y_i - \mu|^3
\end{eqnarray*}
Further 
\begin{eqnarray*}
W_n = \frac{\left(\frac{1}{n} \sum_i Y_i \right) - \mu}{\sigma/\sqrt{n}} = \sum_i X_i
\end{eqnarray*}
The conditions of theorem \ref{theorem:non_uniform_berry_esseen} are satisfied and 
we have:
\begin{eqnarray*}
|F_{W_n}(x) - \Phi(x)| \leq \frac{C_4}{|x|^3} \frac{\mathbb{E}|Y_i - \mu|^3}{n^{\frac{1}{2}} \sigma^3}
\end{eqnarray*}

\end{proof}

For distance distributions it can be hard to calculate $\mathbb{E}|Y_i - \mu|^3$.
We use the common trick of bounding by a higher moment via Jensens.
But first a technical statement.

\begin{proposition}
\label{proposition:technical_statement_on_expectation}
Let $Y_0$ be a random variable and $g_1$, $g_2$ be functions. Define:
\begin{eqnarray*}
Y_1 & := & g_1(Y_0) \\
Y_2 & := & g_2(Y_1)
\end{eqnarray*}
Then:
\begin{eqnarray*}
\mathbb{E}_{Y_1} g_2(Y_1) & = & \mathbb{E}_{Y_0} g_2(g_1(Y_0))
\end{eqnarray*}
\end{proposition}

\begin{proof}
Clearly, if $g$ is a function and $Y$ is a random variable, we have:
\begin{eqnarray}
\label{equation:expectation_statement_1}
\mathbb{E}_{Y} g(Y) & = & \mathbb{E}_{g(Y)} g(Y)
\end{eqnarray}
If $h$ is a function and define r.v $Z := g(Y)$, we have:
\begin{eqnarray}
\label{equation:expectation_statement_2}
\mathbb{E}_{Z}h(Z) = \mathbb{E}_{g(Y)} h(g(Y))
\end{eqnarray}
By using the above two statements:
\begin{eqnarray*}
\mathbb{E}_{Y_1} (g_2(Y_1)) & = & \mathbb{E}_{g_1(Y_0)} (g_2(g_1(Y_0))) \\
    & = & \mathbb{E}_{g_2(g_1(Y_0))} (g_2(g_1(Y_0)))
\end{eqnarray*}
And:
\begin{eqnarray*}
\mathbb{E}_{Y_0} (g_2(g_1(Y_0))) & = & \mathbb{E}_{g_2(g_1(Y_0))} (g_2(g_1(Y_0))) \\
\end{eqnarray*}
\end{proof}

\begin{proposition}[Upper Bound of Third Absolute Moment]
\label{proposition:upper_bound_of_third_absolute_moment}
\begin{eqnarray*}
\mathbb{E}_{Y}|Y|^3 \leq \left( \mathbb{E}_Y(Y^4) \right)^{\frac{3}{4}}
\end{eqnarray*}
\end{proposition}

\begin{proof}
Recall Jensen's.
When $g$ is concave and finite:
\begin{eqnarray*}
\mathbb{E}_{X}[g(X)] \leq g(\mathbb{E}_X[X])
\end{eqnarray*}
Let $g(x) = x^{\frac{3}{4}}$. Then $g(x)$ is concave and we have:
\begin{eqnarray*}
\mathbb{E}_X X^{\frac{3}{4}} \leq \left( \mathbb{E}_X(X) \right)^{\frac{3}{4}}
\end{eqnarray*}
Set $X := Y^4$
\begin{eqnarray*}
\mathrm{LHS} & = & \mathbb{E}_{Y^4} \left( Y^4\right)^{\frac{3}{4}} \\
    & = & \mathbb{E}_{Y}(Y^4)^{\frac{3}{4}} \quad \quad (\mathrm{proposition}\ \ref{proposition:technical_statement_on_expectation}) \\
    & = & \mathbb{E}_{Y}(|Y|^4)^{\frac{3}{4}} \\
    & = & \mathbb{E}_Y |Y|^3
\end{eqnarray*}
\begin{eqnarray*}
\mathrm{RHS} & = & \left( \mathbb{E}_{Y^4} (Y^4) \right)^{\frac{3}{4}} \\
    & = & \left( \mathbb{E}_{Y} Y^4 \right)^{\frac{3}{4}} \quad \quad (\mathrm{proposition}\ \ref{proposition:technical_statement_on_expectation})
\end{eqnarray*}

\end{proof}

\begin{proposition}[Upper Bound of Third Absolute Moment of a Standardized Difference]
\label{proposition_upper_bound_of_third_absolute_moment_of_a_standardized_difference}
Let $\kappa_{Y} := \mathbb{E} \left( \frac{Y - \mu_Y}{\sigma_Y} \right)^4$ denote the kurtosis of a
random variable $Y$.
Let $Y_1, Y_2$ be independent random variables with $\sigma_{Y_i}, \kappa_{Y_i} < \infty$.
and set $Y := Y_2 - Y_1$. Then:
\begin{eqnarray*}
\kappa_{Y} 
    & = &
    \frac{\sigma_{Y_2}^4 \kappa_{Y_2} + 6 \sigma_{Y_2}^2 \sigma_{Y_1}^2 + \sigma_{Y_1}^4 \kappa_{Y_1}}
    {(\sigma_{Y_2}^2  + \sigma_{Y_1}^2)^2} \\
\mathbb{E} \frac{|Y - \mu_Y|^3}{\sigma_Y^3} & \leq & 
    \left(
    \frac{[\sigma_{Y_2}^4 \kappa_{Y_2} + 6 \sigma_{Y_2}^2 \sigma_{Y_1}^2 + \sigma_{Y_1}^4 \kappa_{Y_1}]}
         {(\sigma_{Y_2}^2 + \sigma_{Y_1}^2)^2}
    \right)^{\frac{3}{4}}
\end{eqnarray*}
\end{proposition}

\begin{proof}
Define the standardized version of $Y$ as $X := \frac{Y - \mu_Y}{\sigma_Y}$. Hence:
\begin{eqnarray*}
\mathbb{E}|X|^3 & = & \mathbb{E} \frac{|Y - \mu_Y|^3}{\sigma_Y^3} \\
\mathbb{E} X^4  & = & \mathbb{E} \left( \frac{Y - \mu_Y}{\sigma_Y} \right)^4 = \kappa_Y
\end{eqnarray*}
Then by proposition \ref{proposition:upper_bound_of_third_absolute_moment} we get:
\begin{eqnarray}
\label{equation:upper_bound_on_standardized_third_absolute_moment}
\mathbb{E} \frac{|Y - \mu_Y|^3}{\sigma_Y^3} & \leq & (\kappa_Y)^{\frac{3}{4}}
\end{eqnarray}
Next we derive a formula for $\kappa_Y$. 
The standard formula for kurtosis of sum of 2 random variables is:
$X_1, X_2$:
\begin{eqnarray*}
\kappa_{X_1 + X_2} & = &
 \frac{1}{\sigma_{X_1 + X_2}^4} \Big [
    \sigma_{X_1}^4 \kappa_{X_1} \\
    & & + 4 \sigma_{X_1}^3 \sigma_{X_2} \kappa(X_1, X_1, X_1, X_2) \\
    & & + 6 \sigma_{X_1}^2 \sigma_{X_2}^2 \kappa(X_1, X_1, X_2, X_2) \\
    & & + 4 \sigma_{X_1} \sigma_{X_2}^3 \kappa(X_1, X_2, X_2, X_2) + \sigma_{X_2}^4 \kappa_{X_2}
   \Big ]
\end{eqnarray*}
Where $\kappa(,,,)$ denotes the cokurtosis function. For independent random variables:
\begin{eqnarray*}
\kappa(X_1, X_1, X_1, X_2) & = & 0 \\
\kappa(X_1, X_2, X_2, X_2) & = & 0 \\
\kappa(X_1, X_1, X_2, X_2) & = & 1
\end{eqnarray*}
So
\begin{eqnarray*}
\kappa_{X_1 + X_2} & = & \frac{1}{\sigma_{X_1 + X_2}^4}
     \left[
     \sigma_{X_1}^4 \kappa_{X_1} + 6 \sigma_{X_1}^2 \sigma_{X_2}^2 + \sigma_{X_2}^4 \kappa_{X_2}
     \right]
\end{eqnarray*}
Now set $X_1 = Y_2$ and $X_2 = -Y_1$ to get:
\begin{eqnarray*}
\kappa_{Y_2 - Y_1} & = & \frac{1}{\sigma_{Y_2 - Y_1}^4}
     \left[
     \sigma_{Y_2}^4 \kappa_{Y_2} + 6 \sigma_{Y_2}^2 \sigma_{-Y_1}^2 + \sigma_{-Y_1}^4 \kappa_{-Y_1}
     \right]
\end{eqnarray*}
But $\sigma_{-Y_1} = \sigma_{Y_1}$, $\kappa_{-Y_1} = \kappa_{Y_1}$ and
$\sigma_{Y_2 - Y_1} = \sqrt{\sigma_{Y_2}^2 + \sigma_{Y_1}^2}$ so:
\begin{eqnarray*}
\kappa_{Y_2 - Y_1}  & = & \frac{1}{(\sigma_{Y_2}^2  + \sigma_{Y_1}^2)^2}
    \left[ \sigma_{Y_2}^4 \kappa_{Y_2} + 6 \sigma_{Y_2}^2 \sigma_{Y_1}^2 + \sigma_{Y_1}^4 \kappa_{Y_1} \right]
\end{eqnarray*}

Combining with equation \ref{equation:upper_bound_on_standardized_third_absolute_moment} we get:
\begin{eqnarray*}
\mathbb{E}\frac{|Y - \mu_{Y}|^3}{\sigma_Y^3} \leq \quad \quad \quad & & \\
    \left(
    \frac{1}{(\sigma_{Y_2}^2 + \sigma_{Y_1}^2)^2}
    \left[ \sigma_{Y_2}^4 \kappa_{Y_2} + 6 \sigma_{Y_2}^2 \sigma_{Y_1}^2 + \sigma_{Y_1}^4 \kappa_{Y_1} \right]
    \right)^{\frac{3}{4}}
\end{eqnarray*}
\end{proof}

\begin{proposition}[Familywise Upper Bound on Third Absolute Moment of a Standardized Difference]
\label{proposition_familywise_upper_bound_on_third_absolute_moment_of_standardized_difference}
Given a pair $\mu_{1:2}, \mu_{2:2} \in [0, \infty)$,
Let D be as defined in \ref{definition_difference_of_samples}.
Let the distribution family $\mathcal{F}$ be such that the kurtosis of any member 
of the family is upper-bounded (possibly tightly) by a constant $\kurtosisUB$.
Then:

\begin{eqnarray*}
\kappa_D & \leq & \max(3, \kurtosisUB) \\
\mathbb{E} \frac{|D - \mu_D|^3}{\sigma_D^3} & \leq & 
    \left(
    \max(3, \kurtosisUB)
    \right)^{\frac{3}{4}}
\end{eqnarray*}
\end{proposition}

\begin{proof}

Define a new familywise upper bound on the kurtosis

\begin{eqnarray*}
\kurtosisUBTwo := \max(3, \kurtosisUB) \geq \kurtosisUB
\end{eqnarray*}

Given an arbitrary pair $\mu_{1:2}, \mu_{2:2}$,

Consider an arbitrary $\mu_{1:2}, \mu_{2:2}$ pair.
Consider $D$ as in definition $\ref{definition_difference_of_samples}$.
$D = E_{2:2} - E_{1:2}$ and $E_{2:2} \perp E_{1:2}$.
Let $\sigma_{i:2}^2 = \Var{E_{i:2}}$ and $\kappa_{i:2} = \kappa_{E_{i:2}}$.
Since $\sigma_{i:2} < \infty$ (finite function of mean) \& $\kappa_{i:2} \leq \kurtosisUBTwo < \infty$,
we can apply proposition \ref{proposition_upper_bound_of_third_absolute_moment_of_a_standardized_difference} to get:
\begin{eqnarray*}
\mathbb{E} \frac{|D - \mu_{D}|^3}{\sigma_{D}^3} & \leq &
    \left(
    \frac{\sigma_{2:2}^4 \kappa_{2:2} + 6 \sigma_{2:2}^2 \sigma_{1:2}^2 + \sigma_{1:2}^4 \kappa_{1:2}}
    {(\sigma_{2:2}^2 + \sigma_{1:2}^2)^2}
    \right)^{\frac{3}{4}}
\end{eqnarray*}

\begin{eqnarray*}
\kappa_D(\mu_{1:2}, \mu_{2:2}) 
    & = &
    \left(
    \frac{\sigma_{2:2}^4 \kappa_{2:2} + 6 \sigma_{2:2}^2 \sigma_{1:2}^2 + \sigma_{1:2}^4 \kappa_{1:2}}
    {(\sigma_{2:2}^2 + \sigma_{1:2}^2)^2}
    \right)
\end{eqnarray*}

Clearly this expression is strictly increasing in $\kappa_{i:2}$, keeping everything else fixed.
So we have:
\begin{eqnarray*}
\kappa_D(\mu_{1:2}, \mu_{2:2}) 
    & \leq &
    \left(
    \frac{\sigma_{2:2}^4 \kurtosisUBTwo + 6 \sigma_{2:2}^2 \sigma_{1:2}^2 + \sigma_{1:2}^4 \kurtosisUBTwo}
    {(\sigma_{2:2}^2 + \sigma_{1:2}^2)^2}
    \right) \\
    & = &
    \left(
    \frac{\sigma_{2:2}^4 \kurtosisUBTwo + 6 \sigma_{2:2}^2 \sigma_{1:2}^2 + \sigma_{1:2}^4 \kurtosisUBTwo}
    {(\sigma_{2:2}^2 + \sigma_{1:2}^2)^2}
    \right) \\
    & = & 
    \kurtosisUBTwo + 
    \left(
    \frac{
    \sigma_{2:2}^2 \sigma_{1:2}^2 (6 - 2 \kurtosisUBTwo)
    }
    {(\sigma_{2:2}^2 + \sigma_{1:2}^2)^2}
    \right) \\
    & \leq & \kurtosisUBTwo
\end{eqnarray*}

\end{proof}

Now we bound the difference between the functions $\standardizedCDFOfDifference$ and $\Phi$.

\begin{proposition}[Non Uniform Berry Esseen for $\standardizedCDFOfDifference$]
\label{proposition:non_uniform_berry_esseen_for_standardizedCDFOfDifference}:
Consider $\standardizedCDFOfDifference$ as defined in proposition \ref{proposition_2D_relative_error_formula}.
Given a distribution $\mathcal{E}(\mu) \in \mathcal{F}$, let $\kappa(\mu)$ denote it's kurtosis as a function of $\mu$.
Further, let:

\begin{eqnarray*}
\kappa(\mu) \leq \kurtosisUB  \quad \quad \forall \mu \in [\gamma, \infty)
\end{eqnarray*}

And define:

\begin{eqnarray*}
C_5 & := & (\max(3, \kurtosisUB))^{\frac{3}{4}} \\
C_6 & := & C_4 C_5
\end{eqnarray*}

Then for any $\mu_1, \mu_2 \in [\gamma, \infty)$:

\begin{eqnarray*}
|\standardizedCDFOfDifference(x) - \Phi(x)| \leq
    \frac{C_6}{|x|^3 n^{\frac{1}{2}}}
\end{eqnarray*}
\end{proposition}

\begin{proof}
Given a pair $\mu_1, \mu_2 \in [\gamma, \infty)$,
by proposition \ref{proposition_familywise_upper_bound_on_third_absolute_moment_of_standardized_difference}:
\begin{eqnarray*}
\mathbb{E} \frac{|D - \mu_D|^3}{\sigma_D^3} & \leq & C_5
\end{eqnarray*}
And since $\sigma_D < \infty$
\begin{eqnarray*}
\mathbb{E} |D - \mu_{D}|^3 \leq \sigma_D^3 C_5 < \infty
\end{eqnarray*}

Additionally $D_i$ are iid. We can now apply proposition 
\ref{proposition:non_uniform_berry_esseen_for_sample_average}
by setting $Y_i :=  D_i$. We get:
\begin{eqnarray*}
A_n          & = & \hat{d} \\
W_n          & = & \tilde{d} \\
|\standardizedCDFOfDifference(x) - \Phi(x)| & \leq & 
    \frac{C_4}{|x|^3 n^{\frac{1}{2}}}
    \mathbb{E} \frac{|D - \mu_{D}|^3}{\sigma_{D}^3}
\end{eqnarray*}
and finally:
\begin{eqnarray*}
|\standardizedCDFOfDifference(x) - \Phi(x)|
    \leq \frac{C_4 C_5}{|x|^3 n^{\frac{1}{2}}}
    = \frac{C_6}{|x|^3 n^{\frac{1}{2}}}
\end{eqnarray*}
\end{proof}

%% file: md_rel_error_worst_case_bound.tex
\section{M-Mean Case}
\label{section_MD_mean_case}

In this section we derive a series of results upper bounding $\relError$,
when we have more than $2$ means.
In the first subsection \ref{subsection_generalization_of_2_mean_upper_bounds_to_m_means}
(Generalization of $2$ Mean Upper Bounds to $m$ means)
 we prove an important and central result 
(lemma \ref{lemma_relError_upper_bound_reduction_from_m_means_to_2_means})
 that generalizes a $2$ mean bound on $\relError$ to a $m \geq 2$ mean bound,
for \emph{all} distributions with full generality.
In the second subsection \ref{subsection_m_mean_case_main_results} (M-mean Case: Main Results),
we state and prove the fundamental result for the minimum mean estimation problem.
This is theorem \ref{theorem_minimum_mean_estimation_error_convergence_rate} 

\input{md_rel_error_generalizing_two_mean_upper_bounds.tex}

\input{md_rel_error_main_results.tex}

%% file: md_rel_error_generalizing_two_mean_upper_bounds.tex
\subsection{Generalization of $2$ Mean Upper Bounds to $m$ means}
\label{subsection_generalization_of_2_mean_upper_bounds_to_m_means}

The central result is lemma \ref{lemma_relError_upper_bound_reduction_from_m_means_to_2_means} 
which provides a generalization mechanism for {\em all} distributions provided the $2$ mean
case is bounded.
Subsequent results generalize the various two mean results of section \ref{section_two_mean_case}.

We start by reducing the $m$ mean $\mathbb{P}_{i:m}$ to its two mean counterpart.

\begin{proposition}[$\probabilityOfChoosingim$ Upper Bound: Reduction from $m$-Means to $2$-Means]
\label{proposition_probabilityOfChoosingim_upper_bound_reduction_from_m_means_to_2_means}:
\begin{eqnarray*}
\probabilityOfChoosingim(\musOrderedM) & \leq & \mathbb{P}(\mu_{i:m} < \mu_{1:m})
\end{eqnarray*}
\end{proposition}

\begin{proof}
Using the definition of $\probabilityOfChoosingim$ from definition \ref{definition_probability_of_choosing_i:mth_mean}:

\begin{align*}
\probabilityOfChoosingim(\musOrderedM) = & \mathbb{P}(\hat{\mu}_{i:m} < \hat{\mu}_{1:m}, \ldots, \hat{\mu}_{(i-1):m},  \\ &\hat{\mu}_{(i+1):m}, \ldots, \hat{\mu}_{m:m}) \\
    & \leq \mathbb{P}(\hat{\mu}_{i:m} < \hat{\mu}_{1:m})
\end{align*}
\end{proof}

\begin{lemma}[$\relError$ Upper Bound: Reduction from $m$ Means to $2$ Means]
\label{lemma_relError_upper_bound_reduction_from_m_means_to_2_means}
Given a distribution family $\mathcal{F}$. Suppose the relative error $\relError$ in the 
$2$ mean case is bounded as:

\begin{align*}
\relError(\mu_{1:2}, \mu_{2:2})  < T
\quad \quad \forall \mu_{1:2} \in [\gamma, \infty), \mu_{2:2} > \mu_{1:2} (1 + \deltaTh)
\end{align*}

For some $\deltaTh, T$. Then define a function:

\begin{eqnarray*}
\relErrorUBFour(\deltaTh, T, m) & := & \deltaTh + m T
\end{eqnarray*}

This is an upper bound on the relative error in the $m$ mean case:

\begin{align*}
\relError(\mu_{1:m}, \ldots, \mu_{m:m}) 
	 <  \relErrorUBFour(\deltaTh, T, m) \\
    \quad \quad \forall \mu_{i:m} \in [\gamma, \infty)
\end{align*}

\end{lemma}

\begin{proof}
Let $\muTh := \mu_{1:m} (1 + \deltaTh)$ . Now given $\mu_{i:m}$, let $l$ be the index such that

\begin{eqnarray*}
\mu_{1:m}, \ldots, \mu_{l:m} \leq \muTh < \mu_{(l+1):m}, \ldots, \mu_{m:m}
\end{eqnarray*}

Then we can split the relative error (using proposition \ref{proposition_formula_for_expected_relative_error})
into `head' and `tail' components.

\begin{align*}
\relError(\musOrderedM) & = 
    \sum_{i = 2}^{l} \delta_{i:m}\mathbb{P}_{i:m}(\musOrderedM) + \\&
    \sum_{i = l + 1}^{m} \delta_{i:m}\mathbb{P}_{i:m}(\musOrderedM)
\end{align*}

We can bound the head component in the following manner.
Observe \mbox{$\mu_{i:m} \leq \mu_{1:m}(1 + \deltaTh) \Rightarrow \delta_{i:m} \leq \deltaTh$}.
Hence:

\begin{align*}
\sum_{i = 2}^{l} \delta_{i:m} \mathbb{P}_{i:m}(\musOrderedM)
	& \leq  \sum_{i = 2}^{l} \deltaTh \mathbb{P}_{i:m}(\\ &\musOrderedM) \\
	& =  \deltaTh \sum_{i = 2}^{l} \mathbb{P}_{i:m}(\\ & \musOrderedM) \\
    & \leq  \deltaTh
\end{align*}

Where we have used $\sum_{i = 1}^{m} \mathbb{P}_{i:m}(\musOrderedM) = 1$.
Next we bound the tail component.

\begin{alignat*}{1}
\sum_{i = l + 1}^{m} & \delta_{i:m} \mathbb{P}_{i:m}(\musOrderedM) \\
	& \leq  \sum_{i = l + 1}^{m} \delta_{i:m} \mathbb{P}(\hat{\mu}_{i:m} < \hat{\mu}_{1:m}) 
    \quad \quad (\mathrm{by\ proposition\ } \ref{proposition_probabilityOfChoosingim_upper_bound_reduction_from_m_means_to_2_means}) \\
	& =  \sum_{i = l + 1}^{m} \relError(\eta_1, \eta_2)
    \quad \quad (\eta_1 := \mu_{1:m}, \eta_2 := \mu_{i:m}) \\
    & < \sum_{i = l + 1}^{m} T \quad (\because \eta_2 > \eta_2 (1 + \deltaTh)) \\
    & =  m T
\end{alignat*}

Combining:

\begin{eqnarray*}
\relError(\musOrderedM) < \deltaTh + m T
\end{eqnarray*}

\end{proof}

\begin{lemma}[Abstract Upper Bound on $\relError$]
\label{lemma_abstract_upper_bound_on_relError}
Given a distribution family $\mathcal{F}$, specified by parameters $\familyParametersAugumented$.
Let the correspoding $\standardizedCDFOfDifference$ be continuous.
Given $m$ distributions from $\mathcal{F}$ specified by:
\begin{eqnarray*}
\mu_{i:m} \in [\gamma, \infty) \quad \quad i = 1, \ldots, m
\end{eqnarray*}

Let the conditions of lemma \ref{lemma_abstract_upper_bound_on_relError_two_means}
(abstract upper bound on $\relError$ for two means) be satisfied. We define:

\begin{eqnarray*}
\relErrorUBFour(\deltaTh, T, m) := \deltaTh + m T
\end{eqnarray*}

for $\deltaTh, T$ as defined in lemma \ref{lemma_abstract_upper_bound_on_relError_two_means}.
This is an upper bound on $\relError$.

\begin{eqnarray*}
\relError(\musOrderedM) & < & \relErrorUBFour
    \quad \quad \forall \mu_{i:m} \in [\gamma, \infty)
\end{eqnarray*}
\end{lemma}

\begin{proof}
Since the conditions of lemma \ref{lemma_abstract_upper_bound_on_relError_two_means} 
(abstract upper bound on $\relError$ two means)
are satisfied we have:

\begin{align*}
	\relError(\mu_{1:2}, \mu_{2:2}) & <  T \\ &
    \quad \forall \mu_{1:2} \in [\gamma, \infty), \mu_{2:2} > \mu_{1:2} (1 + \deltaTh)
\end{align*}

For $\deltaTh, T$ as defined in the same lemma.
Then by lemma \ref{lemma_relError_upper_bound_reduction_from_m_means_to_2_means}
($\relError$ upper bound reduction from $m$ means to $2$ means), we have the upper bound:

\begin{align*}
\relError(\musOrderedM) & < \relErrorUBFour(\deltaTh, T, m)
	\quad \quad \\ & \forall \mu_{i:m} \in [\gamma, \infty)
\end{align*}

\end{proof}

\begin{lemma}(Upper Bound on $\relError$: Gaussian Case)
\label{lemma_upper_bound_relError_gaussian_case}
Let $\mathcal{F}$ be the Gaussian family with given values for $\familyParametersAugumented$.
Further, given $m$ distributions from $\mathcal{F}$, specified by

\begin{eqnarray*}
\mu_{i:m} \in [\gamma, \infty) \quad \quad i = 1, \ldots, m
\end{eqnarray*}

Then given a $T > 0$, suppose $n$ satisfies equations
\ref{equation_relErrorUBThreeNormal_monotonicity_requirement}
and \ref{equation_relErrorUBThreeNormal_piecewise_inverse_requirement}:
\begin{align*}
n & >  \frac{\beta}{C_3^2 (2 - \beta) \delta_c^{2 - \beta}} \\
n & \geq  \frac{1}{2 \pi C_2^2 T^2}\exp\left[-C_2^2  n \delta_c^2\right] \\
\end{align*}

And $\deltaThNormal$ is defined as in equation \ref{equation_definition_of_deltaThNormal}:

\begin{eqnarray*}
\deltaTh_{\normalDistribution} :=
    \left\{
    \begin{array}{ll}
    \left(-\frac{2}{C_2^2 n} \log(\sqrt{2 \pi n} C_2 T)\right)^{\frac{1}{2}}
    & T < \frac{1}{\sqrt{2 \pi} C_2 n^{\frac{1}{2}}} \\
    0 & \mathrm{else}
    \end{array}
    \right.
\end{eqnarray*}

Then define

\begin{eqnarray*}
\relErrorUBFourNormal(T, m) = \deltaThNormal(T) + m T
\end{eqnarray*}

Then

\begin{align*}
\relErrorNormal(\musOrderedM) & < & \relErrorUBFourNormal(T, m)
    \quad \quad \forall \mu_{i:m} \in [\gamma, \infty)
\end{align*}

\end{lemma}

\begin{proof}
Since the conditions of lemma \ref{lemma_upper_bound_on_relError_two_means_gaussian_case}
(upper bound on $\relError$ with two means: Gaussian case) are satisfied, we have:

\begin{align*}
	\relErrorNormal(\mu_{1:2}, \mu_{2:2}) & <  T \\&
    \forall \mu_{1:2} \in [\gamma, \infty), \mu_{2:2} > \mu_{1:2} (1 + \deltaThNormal)
\end{align*}

For $T > 0$ and $\deltaThNormal$ as defined in above lemma. Hence by lemma 
\ref{lemma_relError_upper_bound_reduction_from_m_means_to_2_means}
($\relError$ upper bound reduction from $m$ means to $2$ means):

\begin{align*}
\relErrorNormal(\musOrderedM) & < & \relErrorUBFourNormal(T, m)
    \quad \quad \forall \mu_{i:m} \in [\gamma, \infty)
\end{align*}

\end{proof}

\begin{lemma}(Upper Bound on $\relError$: General Case)
\label{lemma_upper_bound_on_relError_general_case}
Let $\mathcal{F}$ be some distribution family, having parameters $\familyParametersAugumented$.
For all distributions $\relError(\mu)$ in this family, let the cdf be continuous 
and the kurtosis be uniformly bounded:

\begin{eqnarray*}
\kappa(\mu_{i:m}) & \leq & \kurtosisUB
    \quad \quad \forall \mu_{i:m} \in [\gamma, \infty)
\end{eqnarray*}

Consider $m$ distributions from $\mathcal{F}$, specified by:

\begin{eqnarray*}
\mu_{i:m} \in [\gamma, \infty) \quad \quad i = 1, \ldots, m
\end{eqnarray*}

$\range(f)$ denotes the range of function $f$.
Given a T, let $n, T$ satisfy the following conditions

\begin{eqnarray*}
T & > & 0 
	\\
n & > & \frac{\beta}{C_3^2 (2 - \beta) \delta_c^{2 - \beta}}
	\\
n & \geq & \frac{2}{\pi C_2^2 T^2} \exp\left[- C_2^2 n \delta_c^2\right]
	\\
\beta & < & \frac{4}{3} = 1.333\ldots
	\\
n & \geq & \frac{C_6^{\frac{1}{2}} 2^{\frac{1}{2}}}{C_2^{\frac{3}{2}}T^{\frac{1}{2}}\delta_c}
	\\
\end{eqnarray*}

and define
	$\deltaThGen(T; \familyParametersAugumented)
	  := $

\begin{align*}
    \left\{
    \begin{array}{ll}
    \max \\ \left(
	    \left(-\frac{2}{C_2^2 n} \log(\sqrt{\frac{\pi n}{2}} C_2 T)\right)^{\frac{1}{2}},
    \frac{C_6^{\frac{1}{2}} \sqrt{2} }{C_2^{\frac{3}{2}} n T^{\frac{1}{2}}}
    \right) & 
    T < \sqrt{\frac{2}{\pi}} \frac{1}{C_2 n^{\frac{1}{2}}}
    \\
    \frac{C_6^{\frac{1}{2}} \sqrt{2} }{C_2^{\frac{3}{2}} n T^{\frac{1}{2}}} & 
    T \geq \sqrt{\frac{2}{\pi}} \frac{1}{C_2 n^{\frac{1}{2}}}
    \end{array}
    \right.
\end{align*}
\end{lemma}

Then if we define:

\begin{eqnarray*}
\relErrorUBFourGen(T, m) := \deltaThGen(T) + m T
\end{eqnarray*}

We have:

\begin{eqnarray*}
\relErrorGen(\musOrderedM) & < & \relErrorUBFourGen(T, m)
    \quad \quad \forall \mu_{i:m} \in [\gamma, \infty)
\end{eqnarray*}

\begin{proof}
Since the conditions of lemma \ref{lemma_upper_bound_on_relError_two_means_general_case}
(upper bound on $\relError$ with two means: General case) are satisfied, we have:

\begin{align*}
	\relErrorGen(\mu_{1:2}, \mu_{2:2}) & <  T \\&
    \quad \quad \forall \mu_{1:2} \in [\gamma, \infty), \mu_{2:2} > \mu_{1:2} (1 + \deltaThGen)
\end{align*}

Hence by lemma \ref{lemma_relError_upper_bound_reduction_from_m_means_to_2_means}
($\relError$ upper bound reduction from $m$ means to $2$ means):

\begin{align*}
\relErrorGen(\musOrderedM) & <  \relErrorUBFourGen(T, m)
	\\ &\quad \quad \forall \mu_{i:m} \in [\gamma, \infty)
\end{align*}
\end{proof}

%% file: md_rel_error_main_results.tex
\subsection{M-Mean Case: Main Results}
\label{subsection_m_mean_case_main_results}

In this subsection,
we derive a fundamental bound (theorem \ref{theorem_minimum_mean_estimation_error_convergence_rate})
on the relative error $\relError$ for the 
minimum mean estimation problem. 
This is a simple closed form bound on the error convergence rate:
\begin{eqnarray*}
\relError < \Theta\left(\frac{m^{\frac{1}{3}}}{n^{\frac{2}{3}}}\right)
\end{eqnarray*}
This result is the analog of the standard square root rate of convergence of Monte Carlo algorithms.
In theorem \ref{theorem_upper_bound_on_n_for_given_tolerance}, 
we derive conditions on $m,n$ to hit a given tolerance.

Our approach is to simplify lemma \ref{lemma_upper_bound_on_relError_general_case}
(upper bound on $\relError$ general case) by expressing the free variable $T$ in terms of $m, n$.
We choose a value $T = T^*$ to decrease the upper bound.

\begin{theorem}[Minimum Mean Estimation: Error Convergence Rate]
\label{theorem_minimum_mean_estimation_error_convergence_rate}
Let $\mathcal{F}$ be some distribution family, having parameters $\familyParametersAugumented$.
For all distributions $\relError(\mu)$ in this family, let the cdf be continuous
and the kurtosis be uniformly bounded:

\begin{eqnarray*}
\kappa(\mu) & \leq & \kurtosisUB
    \quad \quad \forall \mu \in [\gamma, \infty)
\end{eqnarray*}

Given $\mu_{i:m} \ \ i = 1, \ldots, m$, let the following conditions be satisfied:

\begin{eqnarray*}
n
    & > & \frac{\beta}{C_3^2 (2 - \beta) \delta_c^{2 - \beta}}
	\\
n
    & \geq & \frac{1}{C_2^2 \delta_c^2}
    \left[\log\left( \frac{2^{\frac{5}{3}}}{\pi C_6^{\frac{2}{3}}}\right)
        + \frac{4}{3} \log(m) + \frac{1}{3} \log(n)\right]
	\\
\beta & < & \frac{4}{3} = 1.333\ldots
	\\
n
    & \geq & \frac{C_6^{\frac{1}{2}} 2}{C_2^{\frac{3}{2}} \delta_c^{\frac{3}{2}}} m^{\frac{1}{2}}
	\\
\end{eqnarray*}

Further if either, the below pair of conditions are both satsified:

\begin{alignat*}{2}
m^4 n & > \frac{C_6^2 \pi^3}{2^5} & & \\
n & \leq 2 C_6^2 m^2
    \left( \frac{1}{- \log\left(\frac{\pi^{\frac{1}{2}} C_6^{\frac{1}{3}}}{2^{\frac{5}{6}} m^{\frac{2}{3}} n^{\frac{1}{6}}} \right)} \right)^3 & \\ \quad & (\text{conditions on first piecewise component of }\deltaThGen)
\end{alignat*}

or the below single condition is satisfied:

\begin{alignat*}{1}
	m^4 n & \leq \frac{C_6^2 \pi^3}{2^5} \\ \quad & (\text{condition on second piecewise component of }\deltaThGen)
\end{alignat*}

then we have a simplified upper bound on the error:

\begin{eqnarray}
\label{equation_simplified_upper_bound_on_relError}
\relErrorUBFiveGen(m; \familyParametersAugumented) & = & C_7 \frac{m^{\frac{1}{3}}}{n^{\frac{2}{3}}} \\
\relErrorGen(\musOrderedM) & < & \relErrorUBFiveGen \\
    \quad \quad \forall \mu_{i:m} \in [\gamma, \infty)
\end{eqnarray}
Where

\begin{eqnarray*}
C_7 := \frac{C_6^{\frac{1}{3}}}{C_2} \frac{3}{2^{\frac{1}{3}}}
\end{eqnarray*}

\end{theorem}

\begin{proof}
We can vary the free variable $T$ to decrease the upper bound $\relErrorUBFourGen(T, m)$.
We first provide motivation for our choice of $T$. Let us assume the upper bound is of the 
following form:

\begin{eqnarray*}
f_3(T, m) & = & f_2(T) + m T \\
\text{where} \quad \quad f_2(T) & = & \frac{2^{\frac{1}{2}}C_6^{\frac{1}{2}}}{C_2^{\frac{3}{2}} n T^{\frac{1}{2}}}
    \quad \quad T > 0
\end{eqnarray*}

$f_2(T)$ is the second term (BE term) in the $\max$ term in $\deltaThGen$.
Roughly, we are assuming that $\deltaThGen$ simplifies to $\relErrorUBThreeBEInverse$
under widely valid conditions.
Then we want to find the minimizer $T^*$ of:

\begin{eqnarray*}
\min_{T}\ f_3(T, m)
\end{eqnarray*}

\begin{eqnarray*}
\partialDerivative{f_3}{T}
    & = & - \frac{1}{2} \frac{2^{\frac{1}{2}} C_6^{\frac{1}{2}}}{C_2^{\frac{3}{2}} n T^{\frac{3}{2}}} + m
\end{eqnarray*}

equating to zero gives

\begin{alignat*}{2}
&
 & \frac{1}{2} \frac{2^{\frac{1}{2}} C_6^{\frac{1}{2}}}{C_2^{\frac{3}{2}} n (T^*)^{\frac{3}{2}}} & =  m \\
& \Leftrightarrow \quad 
 & (T^*)^{\frac{3}{2}} & = \frac{C_6^{\frac{1}{2}}}{2^{\frac{1}{2}} C_2^{\frac{3}{2}} n m} \\
& \Leftrightarrow \quad 
 & T^* & = \frac{C_6^{\frac{1}{3}}}{2^{\frac{1}{3}} C_2 n^{\frac{2}{3}} m^{\frac{2}{3}}}
\end{alignat*}

Further

\begin{alignat*}{1}
\frac{\partial^2 f_3}{\partial T^2} & = \frac{3}{4} \frac{C_6^{\frac{1}{2}} 2^{\frac{1}{2}}}{C_2^{\frac{3}{2}} n T^{\frac{5}{2}}} > 0 \quad \quad \forall T > 0
\end{alignat*}

Hence $T^*$ is the global minimum. The corresponding upper bound:

\begin{alignat*}{1}
f_3(T^*, n)
    & = \frac{2^{\frac{1}{2}} C_6^{\frac{1}{2}}}{C_2^{\frac{3}{2}} n}
    \frac{2^{\frac{1}{6}} C_2^{\frac{1}{2}} m^{\frac{1}{3}} n^{\frac{1}{3}}}{C_6^{\frac{1}{6}}}
    + \frac{m C_6^{\frac{1}{3}}}{2^{\frac{1}{3}} C_2 m^{\frac{2}{3}} n^{\frac{2}{3}}} \\
    & = \frac{m^{\frac{1}{3}}}{n^{\frac{2}{3}}} \frac{C_6^{\frac{1}{3}}}{C_2} \frac{3}{2^{\frac{1}{3}}} \\
    & = C_7 \frac{m^{\frac{1}{3}}}{n^{\frac{2}{3}}}
\end{alignat*}

Where $C_7 := \frac{C_6^{\frac{1}{3}}}{C_2} \frac{3}{2^{\frac{1}{3}}}$.
Now we ask given a $m, n$ when such an upper bound might hold.
We answer this by plugging in the value for $T^*$ into the conditions of lemma
\ref{lemma_upper_bound_on_relError_general_case} (upper bound on $\relError$: general case).
Consider the condition:

\begin{alignat*}{1}
n & \geq \frac{2}{\pi C_2^2 T^2} \exp\left[-C_2^2 n \delta_c^2 \right]
\end{alignat*}

Plugging in $T^*$:

\begin{alignat}{3}
\notag
&  & n 
    & \geq \frac{2}{\pi C_2^2}
    \left( \frac{2^{\frac{1}{3}} C_2 m^{\frac{2}{3}} n^{\frac{2}{3}}}{C_6^{\frac{1}{3}}} \right)^2
    \exp\left[-C_2^2 n \delta_c^2 \right] &
    & \\
\label{equation_first_piecewise_invertibility_relErrorUBThreeNormal_Tstar_by_2}
& \Leftrightarrow \quad & n 
	& \geq \frac{1}{C_2^2 \delta_c^2} \left[ \log\left(\frac{2^{\frac{5}{3}}}{\pi C_6^{\frac{2}{3}}}\right) + 
    \frac{4}{3}\log\left(m\right) + \frac{1}{3}\log(n) \right] &
\end{alignat}
Next consider the condition:

\begin{alignat*}{3}
&  & 
    n & \geq \frac{2^{\frac{1}{2}} C_6^{\frac{1}{2}}}{C_2^{\frac{3}{2}} T^{\frac{1}{2}} \delta_c} &
    &
\end{alignat*}

Plugging in $T^*$:

\begin{alignat}{3}
\notag
&  & 
    n & \geq \frac{2^{\frac{1}{2}} C_6^{\frac{1}{2}}}{C_2^{\frac{3}{2}} \delta_c}
    \left( \TStarExpansionInverse \right)^{\frac{1}{2}} &
    & \\
\notag
& \Leftrightarrow \quad &
    n & \geq \frac{2^{\frac{2}{3}} C_6^{\frac{1}{3}}}{C_2 \delta_c}  m^{\frac{1}{3}} n^{\frac{1}{3}} &
    & \\
\notag
& \Leftrightarrow &
    n^{\frac{2}{3}} & \geq \frac{2^{\frac{2}{3}} C_6^{\frac{1}{3}}}{C_2 \delta_c}  m^{\frac{1}{3}} &
    & \quad \left(\because \frac{1}{n^{\frac{1}{3}}} > 0 \right) \\
\label{equation_first_piecewise_invertibility_relErrorUBThreeBE_at_Tstar_by_2}
& \Leftrightarrow &
    n & \geq \frac{2 C_6^{\frac{1}{2}}}{C_2^{\frac{3}{2}} \delta_c^{\frac{3}{2}}}  m^{\frac{1}{2}} 
\end{alignat}

Since conditions \ref{equation_first_piecewise_invertibility_relErrorUBThreeNormal_Tstar_by_2},
\ref{equation_first_piecewise_invertibility_relErrorUBThreeBE_at_Tstar_by_2} are satisfied by assumption
and further:

\begin{alignat*}{1}
T^* & > 0 \\
n & > \frac{\beta}{C_3^2 (2 - \beta) \delta_c^{2 - \beta}} \\
\beta & < \frac{4}{3}
\end{alignat*}
are also satisfied by assumption, the conditions for lemma 
\ref{lemma_upper_bound_on_relError_general_case}
(upper bound on $\relError$: general case) hold at $T = T^*$.
Thus we have the upper bound on $\relErrorGen$

\begin{alignat*}{1}
\relErrorUBFourGen(T^*, m) & = \deltaThGen(T^*) + m T^*
\end{alignat*}

Where  $\deltaThGen(T)$ is as defined in equation \ref{equation_deltaThGen_definition}.
It remains to see when \mbox{$\deltaThGen(T^*) = f_2(T^*)$}.
We  rewrite $\deltaThGen$. Let

\begin{alignat*}{1}
f_1(T) & = 
    \left\{ 
    \begin{array}{ll}
    \left( - \frac{2}{C_2^2 n} \log(\frac{\pi^{\frac{1}{2}} n^{\frac{1}{2}}}{2^{\frac{1}{2}}} C_2 T) \right)^{\frac{1}{2}} & 0 < T < T_C \\
    0 & T \geq T_c
    \end{array}
    \right.
\end{alignat*}

Where 

\begin{alignat*}{2}
T_c & = \frac{1}{\sqrt{\frac{\pi n}{2}} C_2} & &\\
f_2(T) & = \frac{2^{\frac{1}{2}} C_6^{\frac{1}{2}}}{C_2^{\frac{3}{2}} n T^{\frac{1}{2}}} & & \quad \forall T > 0
\end{alignat*}

Then:

\begin{alignat*}{1}
\deltaThGen(T) & = \max(f_1(T), f_2(T)) \quad \forall T > 0
\end{alignat*}

Now, what are the conditions under which $\deltaThGen(T^*) = f_2(T^*)$.
We need conditions such that 

\begin{alignat*}{1}
f_1(T^*) \leq f_2(T^*)
\end{alignat*}

This will happen when $T^* \geq T_c$, i.e. when:

\begin{alignat}{3}
\notag
 & &
    \TStarExpansion & \geq \frac{1}{\sqrt{\frac{\pi n}{2}} C_2} &
    & \\
\notag
\Leftrightarrow & \quad &
    \frac{\sqrt{\frac{\pi}{2}} C_6^{\frac{1}{3}}}{2^{\frac{1}{3}} m^{\frac{2}{3}}} & \geq n^{\frac{1}{6}} &
    & \\
\notag
\Leftrightarrow & &
    n & \leq \frac{(\frac{\pi}{2})^3 C_6^2}{2^2 m^4} 
	\\
\label{equation_TStar_geq_Tc}
\Leftrightarrow & &
    m^4 n & \leq \frac{\pi^3 C_6^2}{2^5} &
    & 
\end{alignat}

Next, when $T^* < T_c$, $\deltaThGen(T^*)$ will equal $f_2(T^*)$ if $f_1(T^*) \leq f_2(T^*)$.
That is:

\begin{alignat*}{1}
    & \left( - \frac{2}{C_2^2 n} \log(\frac{\pi^{\frac{1}{2}} n^{\frac{1}{2}}}{2^{\frac{1}{2}}} C_2 T^*) \right)^{\frac{1}{2}}
    \leq \frac{2^{\frac{1}{2}} C_6^{\frac{1}{2}}}{C_2^{\frac{3}{2}} n (T^*)^{\frac{1}{2}}}
    \\
\Leftrightarrow 
    n & \leq 2 C_6^2 m^2
    \frac{1}{ \left( - \log\left( \frac{\pi^{\frac{1}{2}} C_6^{\frac{1}{3}}}{2^{\frac{5}{6}} m^{\frac{2}{3}} n^{\frac{1}{6}}} \right) \right)^3 }
     \left(\text{plugging in }T^*\right)
\end{alignat*}

So $\deltaThGen(T^*) = f_2(T^*)$ when the following conditions are both met:

\begin{alignat*}{2}
m^4 n & > \frac{\pi^3 C_6^2}{2^5} &
    \quad & \left(\Leftrightarrow T^* < T_c \right) \\
n & \leq 2 C_6^2 m^2
    \frac{1}{ \left( - \log\left( \frac{\pi^{\frac{1}{2}} C_6^{\frac{1}{3}}}{2^{\frac{5}{6}} m^{\frac{2}{3}} n^{\frac{1}{6}}} \right) \right)^3 } &
    &
\end{alignat*}

So $\deltaThGen(T^*) = f_2(T^*)$ if either condition \ref{equation_TStar_geq_Tc}
or the above pair are satisfied.
We note that
the above pair is easier to satisfy than condition \ref{equation_TStar_geq_Tc}.
Then:

\begin{alignat*}{1}
\relErrorUBFourGen(T^*, m) & = f_2(T^*) + m T^*
\end{alignat*}

and define:
\begin{alignat*}{1}
\relErrorUBFiveGen(m) & := \relErrorUBFourGen(T^*, m) = C_7 \frac{m^{\frac{1}{3}}}{n^{\frac{2}{3}}}
\end{alignat*}
Completing the proof.
\end{proof}

We want to understand the feasible region of $n$ in the above theorem.
In practice $n$ will mostly satisfy the pair of conditions:
\begin{alignat*}{2}
m^4 n & > \frac{C_6^2 \pi^3}{2^5} & & \\
n & \leq 2 C_6^2 m^2
	\left( \frac{1}{- \log\left(\frac{\pi^{\frac{1}{2}} C_6^{\frac{1}{3}}}{2^{\frac{5}{6}} m^{\frac{2}{3}} n^{\frac{1}{6}}} \right)} \right)^3 \\ \quad & (\text{conditions on first piecewise component of }\deltaThGen)
\end{alignat*}

and not:

\begin{alignat*}{1}
	m^4 n & \leq \frac{C_6^2 \pi^3}{2^5} \\ \quad & (\text{condition on second piecewise component of }\deltaThGen)
\end{alignat*}

Hence we will restrict our study to the feasible region when $n$ satisfies the former pair in conjunction
with the other conditions on $n$ from theorem \ref{theorem_minimum_mean_estimation_error_convergence_rate}.
We start with an useful definition

\begin{definition}[Extended Inverse]
\label{definition_extended_inverse}
Given $f: \mathcal{D} \rightarrow \range$ and a target $y^{'}$.
If $|f^{-1}(y^{'})| \leq 1$, then define the extended inverse $x^{'}$ of $y^{'}$ as:

\begin{alignat*}{1}
x^{'} & =
    \left\{
    \begin{array}{ll}
    -\infty   & y^{'} < y\ \forall y \in \range \\
    f^{-1}(y^{'}) & y^{'} \in \range \\
    \infty    & y < y^{'}\ \forall y \in \range
    \end{array}
    \right.
\end{alignat*}
Note that the extended inverse coincides with the inverse when $y^{'} \in \range$.
Hence, hereafter we use the same notation $f^{-1}(y^{'})$ for both.
\end{definition}

\begin{proposition}[Feasible Region for Theorem \ref{theorem_minimum_mean_estimation_error_convergence_rate}]
Consider the conditions on $n$ in theorem \ref{theorem_minimum_mean_estimation_error_convergence_rate}
(minimum mean estimation: error convergence rate).
When:

\begin{alignat*}{1}
m \geq \frac{e^{\frac{1}{4}} \pi^{\frac{3}{4}} C_6^{\frac{1}{2}}}{2^{\frac{5}{4}}}
\end{alignat*}

a feasible region for these conditions is:

\begin{alignat*}{1}
n & >    g_1^{-1}(0) \\
n & \geq g_2^{-1}(0) \\
n & \geq g_3^{-1}(0) \\
n & >    g_4^{-1}(0) \\
n & \leq g_5^{-1}(0)
\end{alignat*}

Where $g_i$ are the conditions of theorem \ref{theorem_minimum_mean_estimation_error_convergence_rate}
expressed in functional form, equations 
\ref{equation_definiton_of_g_1},
\ref{equation_definiton_of_g_2},
\ref{equation_definiton_of_g_3},
\ref{equation_definiton_of_g_4},
\ref{equation_definiton_of_g_5}
respectively. And the inverses are the extended inverses as per definition \ref{definition_extended_inverse}.
\end{proposition}

\begin{proof}
Consider the conditions on $n$ in theorem \ref{theorem_minimum_mean_estimation_error_convergence_rate},
specifically where 'conditions on first piecewise component of $\deltaThGen$' are being satisfied 
in conjunction with the rest.
These $5$ conditions may be split into four lower bounds and one upper bound.
We rewrite these conditions in functional form.
The domain for all these functions will be $[1, \infty)$.
We start by rewriting the lower bounds. Define:

\begin{eqnarray}
\label{equation_definiton_of_g_1}
g_1(n) & := & n - \frac{\beta}{C_3^2 (2 - \beta) \delta_c^{2 - \beta}} \\
\label{equation_definiton_of_g_2}
g_2(n) & := & n - \frac{1}{C_2^2 \delta_c^2}
    [\log\left( \frac{2^{\frac{5}{3}}}{\pi C_6^{\frac{2}{3}}}\right) \\
        & & + \frac{4}{3} \log(m) + \frac{1}{3} \log(n)] \\
\label{equation_definiton_of_g_3}
g_3(n) & := & n - \frac{C_6^{\frac{1}{2}} 2}{C_2^{\frac{3}{2}} \delta_c^{\frac{3}{2}}} m^{\frac{1}{2}} \\
\label{equation_definiton_of_g_4}
g_4(n) & := & n - \frac{C_6^2 \pi^3}{2^5 m^{4}}
\end{eqnarray}

Then the lower bound conditions are:

\begin{alignat*}{1}
g_1(n) & > 0 \\
g_2(n) & \geq 0 \\
g_3(n) & \geq 0 \\
g_4(n) & > 0 \\
\end{alignat*}

If for all $n \geq 1$ such that $g_i(n) \geq 0$:

\begin{alignat*}{1}
\frac{dg_i}{dn} & > 0 
\end{alignat*}

then $g_i$ is strictly increasing when it is non-negative valued.
Hence the $g_i$ have an unique pre-image for $0$ (if it exists):

\begin{alignat*}{1}
|g_i^{-1}(0)| & \leq 1
\end{alignat*}

Again by the above monotonicity property and 
using definition \ref{definition_extended_inverse} (extended inverse),
we can write the feasible regions as

\begin{alignat*}{1}
n & >    g_1^{-1}(0) \\
n & \geq g_2^{-1}(0) \\
n & \geq g_3^{-1}(0) \\
n & >    g_4^{-1}(0)
\end{alignat*}

We now establish the monotonicity of the $g_i$, starting with $g_2$.

\begin{alignat*}{1}
\frac{dg_2}{dn} & = 1 - \frac{1}{3 C_2^2 \delta_c^2} \frac{1}{n}
\end{alignat*}

If $n \geq 1$ is such that $g_2(n) \geq 0$, then:

\begin{alignat*}{1}
n \geq \frac{1}{C_2^2 \delta^2} \log(\frac{2^{\frac{5}{3}} m^{\frac{4}{3}}}{\pi C_6^{\frac{2}{3}}})
\end{alignat*}

But:
\begin{alignat*}{2}
& &
    m & \geq \frac{e^{\frac{1}{4}} \pi^{\frac{3}{4}} C_6^{\frac{1}{2}}}{2^{\frac{5}{4}}} \\
\Leftrightarrow & \quad &
    3 \log\left( \frac{2^{\frac{5}{3}} m^{\frac{4}{3}}}{\pi C_6^{\frac{2}{3}}} \right) & > 1
\end{alignat*}

and so:

\begin{alignat*}{1}
n > \frac{1}{3 C_2^2 \delta_c^2}
\end{alignat*}

then:

\begin{alignat*}{1}
\frac{dg_2}{dn} & > 1 - 1 = 0
\end{alignat*}

The other three lower bound $g_i$ have derivative $1$ everywhere.
Next consider the upper bound condition, we  rewrite this initially as:

\begin{alignat*}{1}
g_{5\text{orig}}(n) & := n - 2 C_6^2 m^2
    \left( \frac{1}{- \log\left(\frac{\pi^{\frac{1}{2}} C_6^{\frac{1}{3}}}{2^{\frac{5}{6}} m^{\frac{2}{3}} n^{\frac{1}{6}}} \right)} \right)^3 \\
g_{5\text{orig}}(n) & \leq 0
\end{alignat*}

We will now rewrite this in a more tractable form and establish the same monotonicity property
for it as well. Consider the upper bound condition:

\begin{alignat*}{1}
n & \leq 2 C_6^2 m^2
    \left( \frac{1}{- \log\left(\frac{\pi^{\frac{1}{2}} C_6^{\frac{1}{3}}}{2^{\frac{5}{6}} m^{\frac{2}{3}} n^{\frac{1}{6}}} \right)} \right)^3
\end{alignat*}

We have
$\log\left(\frac{\pi^{\frac{1}{2}} C_6^{\frac{1}{3}}}{2^{\frac{5}{6}} m^{\frac{2}{3}} n^{\frac{1}{6}}} \right) > 0$ when:
$\frac{\pi^{\frac{1}{2}} C_6^{\frac{1}{3}}}{2^{\frac{5}{6}} m^{\frac{2}{3}} n^{\frac{1}{6}}} < 1$.
But:

\begin{alignat*}{2}
& &
    \frac{\pi^{\frac{1}{2}} C_6^{\frac{1}{3}}}{2^{\frac{5}{6}} m^{\frac{2}{3}} n^{\frac{1}{6}}} & < 1 \\
\Leftrightarrow & & 
    m^4 n & > \frac{C_6^2 \pi^3}{2^5}
\end{alignat*}
This is satisfied when $g_4(n) > 0$ i.e. when $n > g_4^{-1}(0)$.
Then for such a $n$ we can rewrite the upper bound condition as:

\begin{alignat*}{2}
& 
    n \left(- \log\left(\frac{\pi^{\frac{1}{2}} C_6^{\frac{1}{3}}}{2^{\frac{5}{6}} m^{\frac{2}{3}} n^{\frac{1}{6}}} \right) \right)^3
     \leq 2 C_6^2 m^2 \\
\Leftrightarrow  \quad &
    n^{\frac{1}{3}} \left(- \log\left(\frac{\pi^{\frac{1}{2}} C_6^{\frac{1}{3}}}{2^{\frac{5}{6}} m^{\frac{2}{3}} n^{\frac{1}{6}}} \right) \right)
     \leq 2^{\frac{1}{3}} C_6^{\frac{2}{3}} m^{\frac{2}{3}} \\
\Leftrightarrow \quad & 
    n^{\frac{1}{3}} \left( \log\left(
    \frac{2^{\frac{5}{6}} m^{\frac{2}{3}} n^{\frac{1}{6}}}{\pi^{\frac{1}{2}} C_6^{\frac{1}{3}}}
    \right) \right)
     \leq 2^{\frac{1}{3}} C_6^{\frac{2}{3}} m^{\frac{2}{3}} \\
\Leftrightarrow \quad & 
    n^{\frac{1}{3}} \left[
    \log\left(\frac{2^{\frac{5}{6}} m^{\frac{2}{3}}}{\pi^{\frac{1}{2}} C_6^{\frac{1}{3}}}\right) 
    + \frac{1}{6} \log(n) \right] - 2^{\frac{1}{3}} C_6^{\frac{2}{3}} m^{\frac{2}{3}}
     \leq 0
\end{alignat*}

then set:

\begin{alignat}{1}
\label{equation_definiton_of_g_5}
g_5(n) & :=
    \log\left(\frac{2^{\frac{5}{6}} m^{\frac{2}{3}}}{\pi^{\frac{1}{2}} C_6^{\frac{1}{3}}}\right) n^{\frac{1}{3}}
    + \frac{1}{6} \log(n) n^{\frac{1}{3}}
    - 2^{\frac{1}{3}} C_6^{\frac{2}{3}} m^{\frac{2}{3}}
\end{alignat}
then the upper bound condition is the same as:

\begin{alignat*}{1}
g_5(n) \leq 0
\end{alignat*}

Now we want to show the strict increase of $g_5$.
It suffices to show the strict increase of

\begin{alignat*}{1}
h(n) := n^{\frac{1}{3}} \log(n)
\end{alignat*}

But:

\begin{alignat*}{1}
\sgn\left(\frac{dh}{dn}\right) = \sgn\left(\frac{d}{dn} \log(h)\right)
\end{alignat*}

And so it suffices to show the strict increase of $\log(h)$.
This in turn is equivalent to:

\begin{alignat*}{1}
\frac{1}{n} \left(\frac{1}{3} + \frac{1}{\log(n)} \right) > 0
\end{alignat*}

When $n > 0$, suffices to have:

\begin{alignat*}{3}
& &
    \left(\frac{1}{3} + \frac{1}{\log(n)} \right) & > 0
\end{alignat*}
Clearly this is satisfied for all $n \geq 1$.
Now we have:
\begin{alignat*}{1}
n \geq 1 & \quad \Rightarrow \quad \frac{d g_5}{dn} > 0 \\
\end{alignat*}
So:
\begin{alignat*}{1}
n \leq g_5^{-1}(0) & \quad \Rightarrow \quad g_5(n) \leq 0
\end{alignat*}
But:
\begin{alignat*}{1}
n > g_4^{-1}(0) & \quad \Rightarrow \quad  \left( g_{5\text{orig}}(n) \leq 0 \Leftrightarrow g_5(n) \leq 0 \right)
\end{alignat*}
So a feasible region defined by the upper bound condition is:

\begin{alignat*}{1}
g_4^{-1}(0) < n \leq g_5^{-1}(0)
\end{alignat*}

\end{proof}

This is a qualitative result showing that the feasible region has a simple form: an interval.
It is easy to see that this interval is non-empty for a very wide range of $n$ and $m$.
Since the $g_{5r}^{-1}(0)$ term will grow almost like $\Theta(m^2)$
whereas the other terms will grow almost like $\Theta(m^{\frac{1}{2}})$.
In the following theorem we will formally prove the non-emptiness of the interval by
picking the smallest $n$ in the feasible interval (given a $m$).
We now give more context on this theorem.

In the following theorem,
we use the above result (theorem \ref{theorem_minimum_mean_estimation_error_convergence_rate})
to answer a practical question that has strong implications
for algorithmic design.
Given a $m$, we want to find a $n$ (as small as possible) such that $\relError$ is 
below a tolerance:

\begin{alignat*}{1}
\relError(\musOrderedM; \familyParameters)
    & < \frac{p}{100} \quad \quad \forall \mu_{i:m} \in [\gamma, \infty)
\end{alignat*}

Where $p > 0$ is a percentage that we want to upper bound the error with.
A good answer to this question will resolve some crucial algorithmic design questions.
Our strategy is simple. 
Given a target $\frac{p}{100}$, we will solve our upper bound from
equation \ref{equation_simplified_upper_bound_on_relError} for $n$, keeping everything else fixed.

\begin{definition}[$\pMax$]
\label{definition_pMax}
Let $\mathcal{F}$ be some distribution family having parameters $\familyParametersAugumented$.
For all distributions $\relError(\mu)$ in this family, let the cdf be continuous and the kurtosis
be uniformly bounded:

\begin{alignat*}{1}
\kappa(\mu) & \leq \kurtosisUB \quad \quad \forall \mu \in [\gamma, \infty)
\end{alignat*}

Given $\mu_{i:m} \in [\gamma, \infty)$ $i = 1,\ldots,m$.
Then define:

\begin{alignat*}{1}
\pMax & := 150 \delta_c \ \ (\ \geq 150)
\end{alignat*}
\end{definition}

\begin{definition}[$\mMin$]
\label{definition_mMin}
Let $\mathcal{F}$ be some distribution family having parameters $\familyParametersAugumented$.
For all distributions $\relError(\mu)$ in this family, let the cdf be continuous and the kurtosis
be uniformly bounded:

\begin{alignat*}{1}
\kappa(\mu) & \leq \kurtosisUB \quad \quad \forall \mu \in [\gamma, \infty)
\end{alignat*}

Given $\mu_{i:m} \in [\gamma, \infty)$ $i = 1,\ldots,m$ and $0 < p \leq \pMax$ 
(required for first piecewise invertibility of $\relErrorUBThreeBE$ at $\frac{T^*}{2}$),
define

\begin{alignat*}{1}
\mMinOne(p; \familyParameters)
	& := \\ &\frac{1}{4 \times 150^3} \frac{C_2^3}{C_6 C_3^4} \left(\frac{\beta}{2 - \beta}\right)^2
    \frac{p^3}{\delta_c^{4 - 2 \beta}} \\
\mMinTwo(p, m; \familyParameters) 
	& :=  \frac{1}{6 \times 10^6} \frac{p^3}{C_2 C_6 \delta_c^4} \times \\ &
	[\frac{1}{3} \log(\frac{24 \times 10^2}{\pi^2 C_2 C_6}) \\
    & - \frac{1}{3} \log(p) + \log(m) ]^2 \\
\mMinThree(p; \familyParameters)
    & := \frac{\pi^{\frac{2}{3}}}{2 \times 10^{\frac{2}{3}} \times 3^{\frac{1}{3}}}
    C_6^{\frac{1}{3}} C_2^{\frac{1}{3}} p^{\frac{1}{3}} \\
\mMinFour(p, m; \familyParameters)
	& := \\& \frac{3^3 \times 10^2}{2^5} \frac{1}{C_2 C_6 p} \times\\&
    [\frac{1}{3}\log\left(\frac{24 \times 10^2}{\pi^2 C_2 C_6}\right) \\
    & - \frac{1}{3} \log(p) + \log(m) ]^2
\end{alignat*}

and define $\mMin(p; \familyParameters)$ as the infinimum $m$ that satisfies the equations

\begin{alignat*}{2}
m & > \mMinOne &
    \quad & (\text{required for monotonicity of } \relErrorUBThreeNormal) \\
m & \geq \mMinTwo &
    & (\text{required for first 
	invertibility of }\relErrorUBThreeNormal \text{at } \frac{T^*}{2}) \\
m & > \mMinThree &
    & (\text{required for simplified } \deltaThGen) \\
m & \geq \mMinFour &
    & (\text{required for simplified } \deltaThGen)
\end{alignat*}
\end{definition}

\begin{theorem}[Upper Bound on $n$ for Given Tolerance]
\label{theorem_upper_bound_on_n_for_given_tolerance}
Let $\mathcal{F}$ be some distribution family having parameters $\familyParametersAugumented$.
For all distributions $\relError(\mu)$ in this family, let the cdf be continuous and the kurtosis
be uniformly bounded:

\begin{alignat*}{1}
\kappa(\mu) & \leq \kurtosisUB \quad \quad \forall \mu \in [\gamma, \infty)
\end{alignat*}

Given $\mu_{i:m} \in [\gamma, \infty)$ $i = 1,\ldots,m$ and $0 < p \leq \pMax(; \familyParameters)$
and let $m \in \mathbb{Z}^+$ be such that:

\begin{alignat*}{1}
m & > \mMin(p; \familyParameters)
\end{alignat*}
And let:

\begin{alignat*}{1}
\beta < \frac{4}{3} = 1.333\ldots 
\end{alignat*}

And let:

\begin{alignat*}{1}
n
    & = \left\lceil
    \CInTermsOfp
    m^{\frac{1}{2}}
    \right\rceil
\end{alignat*}

We have

\begin{alignat*}{1}
\relError(\musOrderedM) < \frac{p}{100} \quad \quad \forall \mu_{i:m}  \in [\gamma, \infty)
\end{alignat*}

\end{theorem}

\begin{proof}
We consider $n$ of the form:

\begin{alignat*}{1}
n & = C m^{\frac{1}{2}}
\end{alignat*}

Where $C > 0$. Then we derive equivalent conditions for the conditions of theorem
\ref{theorem_minimum_mean_estimation_error_convergence_rate} (minimum mean estimation: error convergence rate).
Consider the condition:

\begin{alignat}{3}
\notag
& &
    n & >  \frac{\beta}{C_3^2 (2 - \beta) \delta_c^{2 - \beta}} &
    \quad & \\
\notag
\Leftrightarrow & \quad &
    C m^{\frac{1}{2}} & >  \frac{\beta}{C_3^2 (2 - \beta) \delta_c^{2 - \beta}} &
    \quad & (\text{plugging in for }n) \\
\notag
\Leftrightarrow & \quad &
    m^{\frac{1}{2}} & >  \frac{\beta}{C C_3^2 (2 - \beta) \delta_c^{2 - \beta}} &
    \quad & (\because \frac{1}{C} > 0) \\
\label{equation_monotonicity_relErrorUBThreeNormal_n_equal_Csqrtm}
\Leftrightarrow & \quad &
    m & >  \frac{\beta^2}{C^2 C_3^4 (2 - \beta)^2 \delta_c^{4 - 2\beta}} &
    \quad & 
\end{alignat}
Next consider the condition

\begin{alignat*}{3}
\notag
& &
    n & \geq \frac{1}{C_2^2 \delta_c^2}
    \left[\log\left( \frac{2^{\frac{5}{3}}}{\pi C_6^{\frac{2}{3}}}\right)
        + \frac{4}{3} \log(m) + \frac{1}{3} \log(n)\right] &
    & \\
\notag
\Leftrightarrow & \quad &
    C m^{\frac{1}{2}} & \geq \frac{1}{C_2^2 \delta_c^2}
    \left[\log\left( \frac{2^{\frac{5}{3}}}{\pi C_6^{\frac{2}{3}}}\right)
        + \frac{4}{3} \log(m) + \frac{1}{3} \log(Cm^{\frac{1}{2}})\right] 
    \quad & 
	\\
\notag
\Leftrightarrow & \quad &
    m^{\frac{1}{2}} & \geq \frac{1}{C C_2^2 \delta_c^2}
    \log\left( \frac{2^{\frac{5}{3}} C^{\frac{1}{3}}}{\pi C_6^{\frac{2}{3}}} m^{\frac{3}{2}} \right) &
    \quad & 
	\\
\end{alignat*}
If $\frac{2^{\frac{5}{3}} C^{\frac{1}{3}}}{\pi C_6^{\frac{2}{3}}} m^{\frac{3}{2}} \leq 1$
the inequality is trivially satisfied.
If $\frac{2^{\frac{5}{3}} C^{\frac{1}{3}}}{\pi C_6^{\frac{2}{3}}} m^{\frac{3}{2}} > 1$
then the inequality is equivalent to:

\begin{alignat}{3}
\notag
\Leftrightarrow & \quad &
    m & \geq \frac{1}{C^2 C_2^4 \delta_c^4}
    \left( \log\left( \frac{2^{\frac{5}{3}} C^{\frac{1}{3}}}{\pi C_6^{\frac{2}{3}}} m^{\frac{3}{2}} \right) \right)^2 &
    \quad & 
	\\
\label{equation_first_piecewise_invertibility_of_relErrorUBThreeNormal_at_Tstarby2_n_equal_Csqrtm}
\Leftrightarrow & \quad &
    m & \geq \frac{1}{C^2 C_2^4 \delta_c^4}
    \left(\frac{9}{4}\right)
    \left( \log\left( \frac{2^{\frac{10}{9}} C^{\frac{2}{9}}}{\pi^{\frac{2}{3}} C_6^{\frac{4}{9}}} m \right) \right)^2 &
    \quad &
\end{alignat}

So we replace the original condition with the above stronger condition.
Next, consider the condition

\begin{alignat}{3}
\notag
& &
    n & \geq \frac{C_6^{\frac{1}{2}} 2}{C_2^{\frac{3}{2}} \delta_c^{\frac{3}{2}}} m^{\frac{1}{2}} &
    & \\
\notag
\Leftrightarrow & \quad &
    Cm^{\frac{1}{2}} & \geq \frac{C_6^{\frac{1}{2}} 2}{C_2^{\frac{3}{2}} \delta_c^{\frac{3}{2}}} m^{\frac{1}{2}} &
    \quad & (\text{plugging in for }n) \\
\label{equation_first_piecewise_invertibility_of_relErrorUBThreeBE_at_Tstarby2_n_equal_Csqrtm}
\Leftrightarrow & \quad &
    C & \geq \frac{C_6^{\frac{1}{2}} 2}{C_2^{\frac{3}{2}} \delta_c^{\frac{3}{2}}} &
    \quad &
\end{alignat}

Next consider the condition

\begin{alignat}{3}
\notag
& &
    m^4 n & > \frac{C_6^2 \pi^3}{2^5} &
    & \\
\notag
\Leftrightarrow & \quad &
    m^4 C m^{\frac{1}{2}} & > \frac{C_6^2 \pi^3}{2^5} &
    \quad & (\text{plugging in for }n) \\
\notag
\Leftrightarrow & &
    m^{\frac{9}{2}} & > \frac{C_6^2 \pi^3}{C 2^5} &
    \quad & (\because \frac{1}{C} > 0) \\
\label{equation_simplifying_deltaThGen_eq1_n_equal_Csqrtm}
\Leftrightarrow & &
    m & > \frac{C_6^{\frac{4}{9}} \pi^{\frac{2}{3}}}{C^{\frac{2}{9}} 2^{\frac{10}{9}}} &
    \quad & (\because \text{both sides were positive})
\end{alignat}

Finally consider

\begin{alignat*}{3}
& &
    n & \leq 2 C_6^2 m^2
    \left( \frac{1}{- \log\left(\frac{\pi^{\frac{1}{2}} C_6^{\frac{1}{3}}}{2^{\frac{5}{6}} m^{\frac{2}{3}} n^{\frac{1}{6}}} \right)} \right)^3 &
    & \\
\Leftrightarrow & \quad &
    Cm^{\frac{1}{2}} & \leq 2 C_6^2 m^2
    \left( \frac{1}{- \log\left(\frac{\pi^{\frac{1}{2}} C_6^{\frac{1}{3}}}{2^{\frac{5}{6}} C^{\frac{1}{6}} m^{\frac{3}{4}}} \right)} \right)^3 &
    \quad & 
	\\
\Leftrightarrow & &
    2 C_6^2 m^\frac{3}{2}
    & \geq 
    C
    \left( - \log\left(\frac{\pi^{\frac{1}{2}} C_6^{\frac{1}{3}}}{2^{\frac{5}{6}} C^{\frac{1}{6}} m^{\frac{3}{4}}} \right) \right)^3 &
    & 
\end{alignat*}
Because $m^{\frac{1}{2}} > 0$. And 
$-\log\left(\frac{\pi^{\frac{1}{2}} C_6^{\frac{1}{3}}}{2^{\frac{5}{6}} C^{\frac{1}{6}} m^{\frac{3}{4}}} \right) > 0$ 
when 
\mbox{$m^{\frac{3}{4}} > \frac{\pi^{\frac{1}{2}} C_6^{\frac{1}{3}}}{2^{\frac{5}{6}} C^{\frac{1}{6}}}$}
(which is equivalent to condition \ref{equation_simplifying_deltaThGen_eq1_n_equal_Csqrtm})

\begin{alignat}{3}
\notag
\Leftrightarrow & \quad &
    m^\frac{3}{2}
    & \geq 
    \frac{C}{2 C_6^2}
    \left( - \log\left(\frac{\pi^{\frac{1}{2}} C_6^{\frac{1}{3}}}{2^{\frac{5}{6}} C^{\frac{1}{6}} m^{\frac{3}{4}}} \right) \right)^3 &
    & 
	\\
\notag
\Leftrightarrow & \quad &
    m
    & \geq 
    \frac{C^{\frac{2}{3}}}{2^{\frac{2}{3}} C_6^{\frac{4}{3}}}
    \left( - \log\left(\frac{\pi^{\frac{1}{2}} C_6^{\frac{1}{3}}}{2^{\frac{5}{6}} C^{\frac{1}{6}} m^{\frac{3}{4}}} \right) \right)^2 &
    & 
	\\
\notag
\Leftrightarrow & \quad &
    m
    & \geq 
    \frac{C^{\frac{2}{3}}}{2^{\frac{2}{3}} C_6^{\frac{4}{3}}}
    \left( \frac{3}{4} \right)^2
    \left( \log\left(\frac{2^{\frac{10}{9}} C^{\frac{2}{9}}}{\pi^{\frac{2}{3}} C_6^{\frac{4}{9}}} m \right) \right)^2 &
    & \\
\label{equation_simplifying_deltaThGen_eq2_n_equal_Csqrtm}
\Leftrightarrow & \quad &
    m
    & \geq 
    \frac{3^2}{2^{\frac{14}{3}}} \frac{C^{\frac{2}{3}}}{C_6^{\frac{4}{3}}}
    \left[ \log\left(\frac{2^{\frac{10}{9}} C^{\frac{2}{9}}}{\pi^{\frac{2}{3}} C_6^{\frac{4}{9}}}\right) + \log(m) \right]^2 &
    &
\end{alignat}

Hence if conditions \ref{equation_monotonicity_relErrorUBThreeNormal_n_equal_Csqrtm},
\ref{equation_first_piecewise_invertibility_of_relErrorUBThreeNormal_at_Tstarby2_n_equal_Csqrtm},
\ref{equation_first_piecewise_invertibility_of_relErrorUBThreeBE_at_Tstarby2_n_equal_Csqrtm},
\ref{equation_simplifying_deltaThGen_eq1_n_equal_Csqrtm} and
\ref{equation_simplifying_deltaThGen_eq2_n_equal_Csqrtm}
are satisfied and also:

\begin{alignat*}{1}
\beta < \frac{4}{3}
\end{alignat*}

then all the requirements for theorem \ref{theorem_minimum_mean_estimation_error_convergence_rate}
(minimum mean estimation: error convergence rate)  are satisfied with $n = C m^{\frac{1}{2}}$.
Thus:

\begin{alignat*}{1}
\relErrorUBFiveGen(m; \familyParametersAugumented)
    & = C_7 \frac{m^{\frac{1}{3}}}{C^{\frac{2}{3}} m^{\frac{1}{3}}} = \frac{C_7}{C^{\frac{2}{3}}}
\end{alignat*}
is a constant for all $m, n$.
We want to express this constant as a percentage $p > 0$:

\begin{alignat}{3}
\notag
& &
    \relErrorUBFiveGen & = \frac{p}{100} &
    & \\
\notag
\Leftrightarrow & &
    \frac{C_7}{C^{\frac{2}{3}}} & = \frac{p}{100} &
    \quad & (\text{plugging in for }n) \\
\label{equation_C_for_tolerance}
\Leftrightarrow & &
    C & = \CInTermsOfp &
    & (\because p, C > 0)
\end{alignat}

We will now plug this requirement into the previously derived set of conditions.
Consider condition \ref{equation_first_piecewise_invertibility_of_relErrorUBThreeBE_at_Tstarby2_n_equal_Csqrtm}:

\begin{alignat}{3}
\notag
& &
    C & \geq \frac{2 C_6^{\frac{1}{2}}}{C_2^{\frac{3}{2}} \delta_c^{\frac{3}{2}}} &
    & \\
\notag
\Leftrightarrow & \quad &
    \CInTermsOfp & \geq \frac{2 C_6^{\frac{1}{2}}}{C_2^{\frac{3}{2}} \delta_c^{\frac{3}{2}}} &
    \quad & (\text{plugging in for }C) \\
\notag
\Leftrightarrow & \quad &
    {\frac{2 \times 150^{\frac{3}{2}} C_6^{\frac{1}{2}}}{C_2^{\frac{3}{2}}}}  
    \frac{C_2^{\frac{3}{2}} \delta_c^{\frac{3}{2}}}{2 C_6^{\frac{1}{2}}}
    & \geq p^{\frac{3}{2}} &
    & (\because p^{\frac{3}{2}}, \frac{C_2^{\frac{3}{2}} \delta_c^{\frac{3}{2}}}{2 C_6^{\frac{1}{2}}} > 0 ) \\
\notag
\Leftrightarrow & \quad &
    p & \leq 150 \delta_c &
    & 
	\\
\label{equation_first_piecewise_invertibility_of_relErrorUBThreeBE_at_Tstarby2_n_equal_Csqrtm_C_for_p}
\Leftrightarrow & \quad &
    p & \leq \pMax
    & 
\end{alignat}

Consider condition \ref{equation_monotonicity_relErrorUBThreeNormal_n_equal_Csqrtm}

\begin{alignat}{3}
\notag 
& &
    m & >  \frac{\beta^2}{C^2 C_3^4 (2 - \beta)^2 \delta_c^{4 - 2\beta}} &
    & \\
\notag 
\Leftrightarrow & \quad &
    m & > \left(\CInTermsOfpInverse\right)^2 \frac{\beta^2}{C_3^4 (2 - \beta)^2 \delta_c^{4 - 2\beta}} &
    \quad & 
	\\
\label{equation_monotonicity_relErrorUBThreeNormal_n_equal_Csqrtm_C_for_p}
\Leftrightarrow & \quad &
    m & > \frac{1}{4 \times 150^3} \frac{C_2^3}{C_6 C_3^4} \left( \frac{\beta}{2 - \beta} \right)^2
    \frac{p^3}{\delta_c^{4 - 2 \beta}}
    &
\end{alignat}

Next consider condition
\ref{equation_first_piecewise_invertibility_of_relErrorUBThreeNormal_at_Tstarby2_n_equal_Csqrtm}.

\begin{alignat}{1}
\notag
    m & \geq \frac{1}{C^2 C_2^4 \delta_c^4}
    \left(\frac{9}{4}\right)
    \left( \log\left( \frac{2^{\frac{10}{9}} C^{\frac{2}{9}}}{\pi^{\frac{2}{3}} C_6^{\frac{4}{9}}} m \right) \right)^2 \\
\label{equation_first_piecewise_invertibility_of_relErrorUBThreeNormal_at_Tstarby2_n_equal_Csqrtm_C_for_p}.
\Leftrightarrow
    m & \geq \frac{1}{6 \times 10^6} \frac{1}{C_2 C_6} \frac{p^3}{\delta_c^4}
    [\frac{1}{3} \log\left(\frac{24 \times 10^2}{\pi^2 C_2 C_6}\right) \\
    & \quad \quad -\frac{1}{3}\log(p) + \log(m)]^2
\end{alignat}

Next consider equation \ref{equation_simplifying_deltaThGen_eq1_n_equal_Csqrtm}:

\begin{alignat}{3}
\notag
& &
    m & > \frac{C_6^{\frac{4}{9}} \pi^{\frac{2}{3}}}{C^{\frac{2}{9}} 2^{\frac{10}{9}}} &
    & \\
\notag
& &
    m & > \frac{C_6^{\frac{4}{9}} \pi^{\frac{2}{3}}}{2^{\frac{10}{9}}} \left( \CInTermsOfpInverse \right)^{\frac{2}{9}} &
    \quad & (\text{plugging in for }C) \\
\label{equation_simplifying_deltaThGen_eq1_n_equal_Csqrtm_C_for_p}
& &
    m & > \frac{\pi^{\frac{2}{3}}}{2 \times 10^{\frac{2}{3}} \times 3^{\frac{1}{3}}} C_6^{\frac{1}{3}} C_2^{\frac{1}{3}} p^{\frac{1}{3}}
    &
\end{alignat}

Finally consider condition
\ref{equation_simplifying_deltaThGen_eq2_n_equal_Csqrtm}.

\begin{alignat}{1}
\notag
    m
    & \geq 
    \frac{3^2}{2^{\frac{14}{3}}} \frac{C^{\frac{2}{3}}}{C_6^{\frac{4}{3}}}
    \left[ \log\left(\frac{2^{\frac{10}{9}} C^{\frac{2}{9}}}{\pi^{\frac{2}{3}} C_6^{\frac{4}{9}}}\right) + \log(m) \right]^2 \\
\label{equation_simplifying_deltaThGen_eq2_n_equal_Csqrtm_C_for_p}
\Leftrightarrow 
    m & \geq \frac{3^3 \times 10^2}{2^5} \frac{1}{C_2 C_6 p}
    [\frac{1}{3} \log\left( \frac{2^3 \times 3 \times 10^2}{\pi^2 C_2 C_6} \right)\\
    & -\frac{1}{3} \log(p) + \log(m) ]^2
\end{alignat}

Since conditions
\ref{equation_first_piecewise_invertibility_of_relErrorUBThreeBE_at_Tstarby2_n_equal_Csqrtm_C_for_p},
\ref{equation_monotonicity_relErrorUBThreeNormal_n_equal_Csqrtm_C_for_p},
\ref{equation_first_piecewise_invertibility_of_relErrorUBThreeNormal_at_Tstarby2_n_equal_Csqrtm_C_for_p},
\ref{equation_simplifying_deltaThGen_eq1_n_equal_Csqrtm_C_for_p} and
\ref{equation_simplifying_deltaThGen_eq2_n_equal_Csqrtm_C_for_p}
are satisfied by assumption and also:
\begin{alignat*}{1}
\beta & < \frac{4}{3} \\
n & \geq \CInTermsOfp m^{\frac{1}{2}}
\end{alignat*}

We have:

\begin{alignat*}{1}
\relErrorUBFiveGen & \leq \frac{p}{100}
\end{alignat*}

\end{proof}

%% file: mcpam_theory.tex
\section{MCPAM Proofs and Experimental Details}

\subsection{Theory}
\label{section_MCPAM_theory}

\begin{theorem}[MCPAM Guarantees $k \geq 1$]
\label{theorem_mcpam_guarantees_k_geq_1}
We analyze MCPAM inner loop (line 5), with $\tau = 0, n_{\max} = \infty$.
Let $\Delta = \min |\Ecc(\kTuple_{i_1}) - \Ecc(\kTuple_{i_2})|$, with 
$\kTuple_{i_1}, \kTuple_{i_2}$ taken over all possible k-points generatable from $\{x_i\}$.
Let $\Delta > 0$ and be independent of $m$, let the distributions $\Ecc(\kTuple)$ be supported on $\mathbb{R}$.
Let $C$ be the runtime cost of loop 5.
Then loop  $5$ is guaranteed to terminate and (expectation over the $\{\underlyingSpaceSample_i\}_{i = 1}^m$)

\begin{alignat*}{1}
\expectation C  = \mathcal{O}(k m)
\end{alignat*}

Upon exit, with high probability, either one of the following holds:
\begin{itemize}
\item we have found a smaller $\Ecc$ than $\Ecc(\candidateMedoid)$
\item there are no smaller $\Ecc$ than $\Ecc(\candidateMedoid)$ in the 
    swap set constructed from $\candidateMedoid$
\end{itemize}
And the confidence interval 
$[\EccHat_{\text{lo}}(\candidateMedoidUpdated), \EccHat_{\text{hi}}(\candidateMedoidUpdated)]$
has true coverage of $\geq \alpha$.

For the distributed version, computational cost is $\mathcal{O}(\frac{km}{c})$
and communication cost is $\Theta(1)$.
\end{theorem}

\begin{proof}

It is easy to show via SLLN that sample mean and sample variance converge almost
surely to the mean and variance respectively, for r.v with finite variance.
This gives that the width of the confidence intervals of the $\EccHat$ 
go to zero almost surely as $n$ increases.
Since $\Delta > 0$, the intervals become non-overlapping and either line 9 or 11 of MCPAM
must be satisfied a.s . Hence we will exit from the loop of line 5.

Let $p(n)$ be the probability of exiting from loop 5. We have shown
$\lim_{n \rightarrow \infty} p(n) = 1$. 
We are increasing  $n$ in steps of 10.
Let $F$ be number of iterations of loop 5 before exit.
This is a non-homogenous geometric variable.
Let $C(F)$ denote the runtime cost. 
We have per iteration cost $k m n = k m 10^{2 + f}$.
 By summing the geometric progression
$k m 10^3, k m 10^4, \ldots$ we get:

\begin{alignat*}{1}
C(f) & = \frac{10^3}{9} k m (10^f - 1)\\
\expectation C(F) & \geq \frac{10^3}{9} k m (10^{\expectation(F)} - 1) \quad (\text{by Jensens}) \\
\expectation(F) & = \sum_{i = 1}^{\infty} i p(10^{i + 2}) \prod_{j = 1}^{i - 1} (1 - p(10^{j + 2}))  
\end{alignat*}
The $\expectation(F)$ sum is dependent on the data distribution and independent of $m$.
It is easy to see the convergence of this sum, by applying the ratio test.
Let $T_i$ be the $i^{\text{th}}$ term.

\begin{alignat*}{1}
\lim_{i \rightarrow \infty} |\frac{T_{i+1}}{T_i}| 
    & = \lim_{i \rightarrow \infty} \frac{i + 1}{i} \frac{p(10^{i + 3})}{p(10^{i + 2})} (1 - p(10^{i + 2})) \\
    & = \lim_{i \rightarrow \infty} \frac{i + 1}{i} \lim_{i \rightarrow \infty} \frac{p(10^{i + 3})}{p(10^{i + 2})} \lim_{i \rightarrow \infty} (1 - p(10^{i + 2})) \\
    & = 0
\end{alignat*}

Hence the average runtime is $\expectation(C(F)) = \mathcal{O}(k m)$.
We believe the sum $\expectation(F)$ will be upper bounded by $\frac{1}{p(10^3)}$.

The procedure we follow in loop 5 is termed sequential Monte Carlo.
It is well studied and has a convergence result \cite{glynn1992SequentialStopping} similar to the CLT.
The confidence interval $[\EccHat_{\text{lo}}(\kTuple), \EccHat_{\text{hi}}(\kTuple)]$ contains 
$\Ecc(\kTuple)$ with $1 - \alpha$ probability as $m$ increases.
Finally, note that condition 9 is checking (across $il$)
$\Ecc(\candidateMedoid) < \Ecc(\candidateMedoidSwap)$ for the box confidence region 
of 
$
[\EccHat_{\text{lo}}(\candidateMedoid), \EccHat_{\text{hi}}(\candidateMedoid)]
\times 
[\EccHat_{\text{lo}}(\candidateMedoidSwap), \EccHat_{\text{hi}}(\candidateMedoidSwap)]
$.
The box confidence region is a superset of the actual  $1 - \alpha$ ellipsoidal confidence
region, even in the correlated means case. The result on the distributed version is immediate
from the above.
\end{proof}

\begin{theorem}[MCPAM Guarantees $k = 1$]
\label{theorem_mcpam_guarantees_k_1}
We analyze 1-medoid MCPAM, with $\tau = 0, n_{\max} = \infty$.
If the conditions of \ref{theorem_mcpam_guarantees_k_geq_1} hold with $k = 1$,
Then MCPAM is guaranteed to terminate and
$\expectation C  = \mathcal{O}(m)$
Upon exit, we have found the sample medoid with high probability
and the confidence interval around the estimate has true coverage of $\geq \alpha$.
\end{theorem}

\begin{proof}
In the 1-medoid case, we finally exit loop 5 only via condition 9.
The arguments of theorem \ref{theorem_mcpam_guarantees_k_geq_1}
apply essentially unchanged.
When we satisfy line 9, $\EccHat(\candidateMedoid)$ now has the lowest
$\Ecc$ with probability $1 - \alpha$.
\end{proof}

\subsection{Experiments}

\subsubsection{Datasets}
\label{subsubsection_mcpam_datasets}
We use the collection of datasets provided in \cite{ClusteringDatasets} for most of our evaluation.
{\textbf{S1, S2, S3, S4}} are datasets of size $5000$ in two dimensions with increasing overlap among a cluster.
For ex- S4 will have significantly higher overlap among the clusters compared to S1, S2 and S3.
They all have 15 clusters.
{\textbf{Leaves}} is taken from \cite{mallah2013leaves} it contains 1600 rows with 64 attributes.
There are 100 clusters.
{\textbf{letter1} to \textbf{letter4}} is borrowed from \cite{frey1991letter} each of them have around 4600 rows with 16 attributes with increasing overlap among the classes.
They all have 26 clusters.

For large scale run we used \textbf{Foursquare checkin} dataset \cite{yang2016foursquare} which contains around 32 million rows.

In addition we have generated two non-Euclidean datasets.

{\textbf{M1,M2,M3,M4:}} These are synthetic datasets with mixed (numeric and categorical) attributes.
Each dataset consists of 2 numeric attributes 
and 1 categorical attribute with 2 levels. There are  3200 points in total 
and  32 clusters. Each cluster is a hybrid distribution, the numeric  attributes are drawn from uncorrelated multivariate gaussian.
The categorical attribute  follows a Bernoulli distributon.
The datasets M1,M2,M3,M4 are in increasing order of overlap.
We use Gowers distance \cite{gowerDist} with equal weights.

{\textbf{BillionOne:}} This is a synthetic, mixed (numeric and categorical attributes) dataset 
with a billion points.
We have 5 real valued attributes, 1 categorical attribute with 24 levels.
Each cluster is a hybrid distribution, the numeric attributes are
drawn from an uncorrelated multivariate gaussian. The categorical attribute is drawn
from a categorical distribution, aka generalized binomial distribution.
There are 24 such clusters.
We use Gowers distance \cite{gowerDist} with equal weights.

\subsubsection{Experimental Setup}
\label{subsubsection_mcpam_experimental_setup}

\textbf{Software Setup:} For PAM we used R's 
\href{https://cran.r-project.org/package=cluster}{cluster}
package:
cluster\_2.0.3, R version 3.2.3 .
For dbscan we used R's \href{https://cran.r-project.org/package=dbscan}{dbscan} package:
dbscan\_1.1-2 . The OS was Ubuntu 16.04 .

\textbf{Hardware Setup:} For small scale time and memory
comparison we used commodity 16 GB RAM laptop, with a 6th Generation Intel i7 processor. 
For Foursquare dataset we used a local cluster of 4
commodity machines.
Each with 32 GB RAM and Intel Core i7 CPUs, connected by a 1 Gigabit network. 
Each machine was running multiple workers, but the workers were isolated in different userspaces, 
so no two workers were affecting each other despite running on the same machine. 
For 1 Billion data point run we used 4 C5.4x large each with four workers and a C5.2x large as master.

\textbf{Distribution Setup:} For distribution, we implemented a master worker topology. 
For which we use Flask to create REST API endpoints and Redis as Message Broker 
for making asynchronous requests.
Given $c$ workers, the data is partitioned into $c$ chunks of $m/c$ points.
In our implementation the Master does out of core random sampling.
The data is on the hard disk, but is never loaded into memory.
Alternatively, it is also possible for master to not have access to any data just knowledge 
of how many data points each worker has is sufficient.

\subsubsection{Experimental Results (Contd.)}
\label{subsubsection_experimental_results_contd}

Table \ref{figure_cc_comparison} compares clustering cost between PAM and MCPAM.

\input{pictures/cc_table.tex}

Table \ref{billion_scaling} details the scaling of MCPAM on BillionOne data set.

\input{pictures/summarized_billion_paper_table.tex}

%% file: pictures/cc_table.tex
\begin{figure}
\label{figure_cc_comparison}
\centering
\begin{tabular}{@{}lll@{}}
\toprule
\textbf{Dataset} & \textbf{PAM-CC} & \multicolumn{1}{c}{\textbf{\begin{tabular}[c]{@{}c@{}}MCPAM-CC\\ (mean)\end{tabular}}} \\ \midrule
\multicolumn{1}{|l|}{S1} & \multicolumn{1}{l|}{42767.52} & \multicolumn{1}{l|}{46250.96} \\ \midrule
\multicolumn{1}{|l|}{S2} & \multicolumn{1}{l|}{52284.17} & \multicolumn{1}{l|}{59683.26} \\ \midrule
\multicolumn{1}{|l|}{S3} & \multicolumn{1}{l|}{60695.48} & \multicolumn{1}{l|}{67777.88} \\ \midrule
\multicolumn{1}{|l|}{S4} & \multicolumn{1}{l|}{57481.78} & \multicolumn{1}{l|}{66843.88} \\ \midrule
\multicolumn{1}{|l|}{Leaves} & \multicolumn{1}{l|}{0.00278} & \multicolumn{1}{l|}{0.0034} \\ \midrule
\multicolumn{1}{|l|}{letter1} & \multicolumn{1}{l|}{13496.55} & \multicolumn{1}{l|}{15265.26} \\ \midrule
\multicolumn{1}{|l|}{letter2} & \multicolumn{1}{l|}{13081.18} & \multicolumn{1}{l|}{15222.28} \\ \midrule
\multicolumn{1}{|l|}{letter3} & \multicolumn{1}{l|}{11666.38} & \multicolumn{1}{l|}{13750.14} \\ \midrule
\multicolumn{1}{|l|}{letter4} & \multicolumn{1}{l|}{12879.32} & \multicolumn{1}{l|}{14947.95} \\ \midrule
\multicolumn{1}{|l|}{M1} & \multicolumn{1}{l|}{0.315} & \multicolumn{1}{l|}{0.411} \\ \midrule
\multicolumn{1}{|l|}{M2} & \multicolumn{1}{l|}{0.454} & \multicolumn{1}{l|}{0.546} \\ \midrule
\multicolumn{1}{|l|}{M3} & \multicolumn{1}{l|}{0.511} & \multicolumn{1}{l|}{0.603} \\ \midrule
\multicolumn{1}{|l|}{M4} & \multicolumn{1}{l|}{0.554} & \multicolumn{1}{l|}{0.644} \\ \bottomrule
\end{tabular}
\caption{
\label{figure_cc_comparison}
Clustering cost(CC) comparison between PAM and MCPAM on various datasets. 
Less CC is better.
MCPAM closely follows PAM in general.}
\end{figure}

%% file: pictures/summarized_billion_paper_table.tex
\begin{table}[]
\centering
\scalebox{0.7}{
\begin{tabular}{@{}|l|l|l|l|l|l|@{}}
\toprule
\textbf{$K$}
    & \textbf{Num Workers} 
    & \multicolumn{1}{c|}{\textbf{\begin{tabular}[c]{@{}c@{}}Num MCPAM\\ Iters\end{tabular}}}
    & \textbf{Runtime (s)} 
    & \textbf{Memory (MB)} 
    \\ \midrule
1 & 12 & 2 & 5800 & 4030 \\ \midrule 
2 & 12 & 2 & 5870 & 4030 \\ \midrule 
5 & 12 & 3 & 6310 & 4030 \\ \midrule 
10 & 12 & 3 & 7300 & 4030 \\ \midrule 
15 & 12 & 6 & 8510 & 4030 \\ \midrule 
20 & 12 & 8 & 9830 & 4030 \\ \midrule 
30 & 12 & 10 & 11800 & 4030 \\ \midrule 
50 & 12 & 5 & 15700 & 4030 \\ \midrule 
100 & 12 & 1 & 25400 & 4030 \\ \bottomrule 
\end{tabular}}
\caption{
\label{billion_scaling}
Scaling of MCPAM on BillionOne dataset.
Due to the large size of the dataset, we only run with 12 workers.}
\end{table}

%% file: main.bbl
\begin{thebibliography}{10}

\bibitem{abbasi2011improved}
Yasin Abbasi-Yadkori, D{\'a}vid P{\'a}l, and Csaba Szepesv{\'a}ri.
\newblock Improved algorithms for linear stochastic bandits.
\newblock In {\em Advances in Neural Information Processing Systems}, pages
  2312--2320, 2011.

\bibitem{agrawal2012analysis}
Shipra Agrawal and Navin Goyal.
\newblock Analysis of thompson sampling for the multi-armed bandit problem.
\newblock In {\em Conference on Learning Theory}, pages 39--1, 2012.

\bibitem{Arthur07K++}
David Arthur and Sergei Vassilvitskii.
\newblock K-means++: the advantages of careful seeding.
\newblock In {\em In Proceedings of the 18th Annual ACM-SIAM Symposium on
  Discrete Algorithms}, 2007.

\bibitem{bagaria2017medoids}
Vivek Bagaria, Govinda~M Kamath, Vasilis Ntranos, Martin~J Zhang, and David
  Tse.
\newblock Medoids in almost linear time via multi-armed bandits.
\newblock {\em arXiv preprint arXiv:1711.00817}, 2017.

\bibitem{balkema1990EVT}
A.~A. Balkema and L.~De Haan.
\newblock A convergence rate in extreme-value theory.
\newblock {\em Journal of Applied Probability}, 27(3):577--585, 1990.

\bibitem{basford1997Standard}
KE~Basford, DR~Greenway, GJ~McLachlan, and D~Peel.
\newblock Standard errors of fitted component means of normal mixtures.
\newblock {\em Computational Statistics}, 12(1):1--18, 1997.

\bibitem{chu1955SampleMedian}
John~T. Chu.
\newblock On the distribution of the sample median.
\newblock {\em Ann. Math. Statist.}, 26(1):112--116, 03 1955.

\bibitem{dani2008stochastic}
Varsha Dani, Thomas~P Hayes, and Sham~M Kakade.
\newblock Stochastic linear optimization under bandit feedback.
\newblock 2008.

\bibitem{eppstein2001fast}
David Eppstein and Joseph Wang.
\newblock Fast approximation of centrality.
\newblock In {\em Proceedings of the twelfth annual ACM-SIAM symposium on
  Discrete algorithms}, pages 228--229. Society for Industrial and Applied
  Mathematics, 2001.

\bibitem{esterKDD}
Martin Ester, Hans-Peter Kriegel, Jiirg Sander, and Xiaowei Xu.
\newblock A density-based algorithm for discovering clustersin large spatial
  databases with noise.
\newblock {\em KDD}, 1996.

\bibitem{fellerProbability}
W.~Feller.
\newblock {\em An Introduction to Probability Theory and Its Applications},
  volume~1.
\newblock Springer Series in Statistics, 3 edition, 1968.

\bibitem{fraignaudRTT}
P.~{Fraigniaud}, E.~{Lebhar}, and L.~{Viennot}.
\newblock The inframetric model for the internet.
\newblock In {\em IEEE INFOCOM 2008 - The 27th Conference on Computer
  Communications}, pages 1085--1093, April 2008.

\bibitem{Ssets}
P.~Fr\"anti and O.~Virmajoki.
\newblock Iterative shrinking method for clustering problems.
\newblock {\em Pattern Recognition}, 39(5):761--765, 2006.

\bibitem{ClusteringDatasets}
Pasi Fr\"anti and Sami Sieranoja.
\newblock K-means properties on six clustering benchmark datasets, 2018.

\bibitem{frey1991letter}
Peter~W Frey and David~J Slate.
\newblock Letter recognition using holland-style adaptive classifiers.
\newblock {\em Machine learning}, 6(2):161--182, 1991.

\bibitem{gabillon2011BAI}
Victor Gabillon, Mohammad Ghavamzadeh, Alessandro Lazaric, and S\'{e}bastien
  Bubeck.
\newblock Multi-bandit best arm identification.
\newblock In J.~Shawe-Taylor, R.~S. Zemel, P.~L. Bartlett, F.~Pereira, and
  K.~Q. Weinberger, editors, {\em Advances in Neural Information Processing
  Systems 24}, pages 2222--2230. Curran Associates, Inc., 2011.

\bibitem{glynn1992SequentialStopping}
Peter Glynn and Ward Whitt.
\newblock The asymptotic validity of sequential stopping rules for stochastic
  simulations.
\newblock {\em Annals of Applied Probability}, 2(1):180--198, 1992.

\bibitem{gotzHPDBSCAN}
Markus G\"{o}tz, Christian Bodenstein, and Morris Riedel.
\newblock Hpdbscan: Highly parallel dbscan.
\newblock In {\em Proceedings of the Workshop on Machine Learning in
  High-Performance Computing Environments}, MLHPC '15, pages 2:1--2:10, New
  York, NY, USA, 2015. ACM.

\bibitem{gowerDist}
J.~C. Gower.
\newblock A general coefficient of similarity and some of its properties.
\newblock {\em Biometrics}, 1971.

\bibitem{hansen1996Generalization}
L.K. Hansen and J.~Larsen.
\newblock Unsupervised learning and generalization.
\newblock In {\em In Proceedings of IEEE International Conference on Neural
  Networks. https://doi.org/10.1109/ICNN.1996.548861}, pages 25--30, 1996.

\bibitem{he2011parallel}
Qing He, Qun Wang, Fuzhen Zhuang, Qing Tan, and Zhongzhi Shi.
\newblock Parallel clarans clustering based on mapreduce.
\newblock {\em Energy Procedia}, (13):3269--3279, 2011.

\bibitem{heMRDBSCAN}
Yaobin He, Haoyu Tan, Wuman Luo, Huajian Mao, Di~Ma, Shengzhong Feng, and
  Jianping Fan.
\newblock Mr-dbscan: An efficient parallel density-based clustering algorithm
  using mapreduce.
\newblock {\em 2011 IEEE 17th International Conference on Parallel and
  Distributed Systems}, pages 473--480, 2011.

\bibitem{jianwenEVT}
Jianwen Huang, Jianjun Wang, and Guowang Luo.
\newblock On the rate of convergence of maxima for the generalized maxwell
  distribution.
\newblock {\em Statistics}, 51(5):1105--1117, 2017.

\bibitem{jiang2014parallel}
Yaobin Jiang and Jiongmin Zhang.
\newblock Parallel k-medoids clustering algorithm based on hadoop.
\newblock In {\em 2014 IEEE 5th International Conference on Software
  Engineering and Service Science}, pages 649--652. IEEE, 2014.

\bibitem{kaufman2008clustering}
L~Kaufman and Pr~J Rousseeuw.
\newblock Clustering large applications (program clara).
\newblock {\em Finding groups in data: an introduction to cluster analysis},
  pages 126--163, 2008.

\bibitem{kaufman1990partitioning}
Leonard Kaufman and Peter~J Rousseeuw.
\newblock Partitioning around medoids (program pam).
\newblock {\em Finding groups in data: an introduction to cluster analysis},
  pages 68--125, 1990.

\bibitem{leadbetterEVT}
M.R. Leadbetter, G.~Lindgren, and H.~Rootzen.
\newblock {\em Extremes and Related Properties of Random Sequences and
  Processes}.
\newblock Springer Series in Statistics, 1983.

\bibitem{leskovec2014mining}
Jure Leskovec, Anand Rajaraman, and Jeffrey~David Ullman.
\newblock {\em Mining of massive datasets}.
\newblock Cambridge university press, 2014.

\bibitem{Louis1982EMInformationMatrix}
Thomas~A. Louis.
\newblock Finding the observed information matrix when using the em algorithm.
\newblock {\em Journal of the Royal Statistical Society. Series B
  (Methodological)}, 44(2):226--233, 1982.

\bibitem{mallah2013leaves}
Charles Mallah, James Cope, and James Orwell.
\newblock Plant leaf classification using probabilistic integration of shape,
  texture and margin features.
\newblock {\em Signal Processing, Pattern Recognition and Applications}, 5(1),
  2013.

\bibitem{martino2017efficient}
Alessio Martino, Antonello Rizzi, and Fabio Massimo~Frattale Mascioli.
\newblock Efficient approaches for solving the large-scale k-medoids problem.
\newblock In {\em IJCCI}, pages 338--347, 2017.

\bibitem{nair1940Median}
K.~R. Nair.
\newblock Table of confidence interval for the median in samples from any
  continuous population.
\newblock {\em Sankhya: The Indian Journal of Statistics (1933-1960)},
  4(4):551--558, 1940.

\bibitem{newling2016sub}
James Newling and Fran{\c{c}}ois Fleuret.
\newblock A sub-quadratic exact medoid algorithm.
\newblock {\em arXiv preprint arXiv:1605.06950}, 2016.

\bibitem{ng2002clarans}
Raymond~T Ng and Jiawei Han.
\newblock Clarans: A method for clustering objects for spatial data mining.
\newblock {\em IEEE Transactions on Knowledge \& Data Engineering},
  (5):1003--1016, 2002.

\bibitem{ntranos2016fast}
Vasilis Ntranos, Govinda~M Kamath, Jesse~M Zhang, Lior Pachter, and N~Tse
  David.
\newblock Fast and accurate single-cell rna-seq analysis by clustering of
  transcript-compatibility counts.
\newblock {\em Genome biology}, 17(1):112, 2016.

\bibitem{OakesEMInformationMatrix}
David Oakes.
\newblock Direct calculation of the information matrix via the em algorithm.
\newblock {\em Journal of the Royal Statistical Society. Series B (Statistical
  Methodology)}, 61(2):479--482, 1999.

\bibitem{okamoto2008ranking}
Kazuya Okamoto, Wei Chen, and Xiang-Yang Li.
\newblock Ranking of closeness centrality for large-scale social networks.
\newblock In {\em International Workshop on Frontiers in Algorithmics}, pages
  186--195. Springer, 2008.

\bibitem{PaditzBE}
Ludwig Paditz.
\newblock On the analytical structure of the constant in the nonuniform version
  of the esseen inequality.
\newblock {\em Statistics}, 20(3):453--464, 1989.

\bibitem{rand1971objective}
William~M Rand.
\newblock Objective criteria for the evaluation of clustering methods.
\newblock {\em Journal of the American Statistical association},
  66(336):846--850, 1971.

\bibitem{regev2017learning}
Oded Regev and Aravindan Vijayaraghavan.
\newblock On learning mixtures of well-separated gaussians.
\newblock In {\em 2017 IEEE 58th Annual Symposium on Foundations of Computer
  Science (FOCS)}, pages 85--96. IEEE, 2017.

\bibitem{smith1982EVT}
Richard~L. Smith.
\newblock Uniform rates of convergence in extreme-value theory.
\newblock {\em Advances in Applied Probability}, 14(3):600--622, 1982.

\bibitem{song2017pamae}
Hwanjun Song, Jae-Gil Lee, and Wook-Shin Han.
\newblock Pamae: Parallel k-medoids clustering with high accuracy and
  efficiency.
\newblock In {\em Proceedings of the 23rd ACM SIGKDD International Conference
  on Knowledge Discovery and Data Mining}, pages 1087--1096. ACM, 2017.

\bibitem{sergei2019clusteringbook_1}
Sergei Vassilvitskii and Suresh Venkatasubramanian.
\newblock {\em Clustering}, chapter~2, pages 7--15.
\newblock Book Draft.

\bibitem{vempala2004spectral}
Santosh Vempala and Grant Wang.
\newblock A spectral algorithm for learning mixture models.
\newblock {\em Journal of Computer and System Sciences}, 68(4):841--860, 2004.

\bibitem{yang2016foursquare}
Dingqi Yang, Daqing Zhang, and Bingqing Qu.
\newblock Participatory cultural mapping based on collective behavior data in
  location-based social networks.
\newblock {\em ACM Transactions on Intelligent Systems and Technology (TIST)},
  7(3):30, 2016.

\bibitem{yue2016parallel}
Xia Yue, Wang Man, Jun Yue, and Guangcao Liu.
\newblock Parallel k-medoids++ spatial clustering algorithm based on mapreduce.
\newblock {\em arXiv preprint arXiv:1608.06861}, 2016.

\bibitem{zhang2004parallel}
Ya-Ping Zhang, Ji-Zhao Sun, Yi~Zhang, and Xu~Zhang.
\newblock Parallel implementation of clarans using pvm.
\newblock In {\em Proceedings of 2004 International Conference on Machine
  Learning and Cybernetics (IEEE Cat. No. 04EX826)}, volume~3, pages
  1646--1649. IEEE, 2004.

\end{thebibliography}
